\documentclass[twoside,11pt]{article}
\pdfoutput=1
     \PassOptionsToPackage{numbers, compress}{natbib}

\usepackage{jmlr2e}
\usepackage{bm}
\usepackage[inline]{enumitem}
\usepackage[autostyle=true]{csquotes} 
\usepackage{mathpazo}
\usepackage{enumitem}
\usepackage{amsmath,amssymb,amsfonts,mathtools}
\usepackage{color}
\usepackage{bbm}
\usepackage{accents}
\usepackage{hyperref}

\DeclareMathOperator*{\argmax}{\operatorname{argmax}}
\newcommand{\KL}[0]{\operatorname{KL}}
\newcommand{\KLinf}[0]{\operatorname{KL_{inf}}}
\newcommand{\KLinfL}[0]{\operatorname{KL_{inf}^L}}
\newcommand{\KLinfU}[0]{\operatorname{KL_{inf}^U}}
\newcommand{\KLinfUC}[1]{\operatorname{KL}^{\operatorname{U},{#1}}_{\operatorname{inf}}}
\newcommand{\KLinfLC}[1]{\operatorname{KL}^{\operatorname{L},{#1}}_{\operatorname{inf}}}

\newcommand{\Supp}[0]{\operatorname{Supp}}

\newcommand{\lrset}[1]{\left\{{#1}\right\}}
\newcommand{\lrp}[1]{\left({#1}\right)}

\newcommand{\floor}[1]{\left\lfloor{#1}\right\rfloor}
\newcommand{\ubar}[1]{\underbar{${#1}$}}
\newcommand{\abs}[1]{\left\lvert{#1}\right\rvert}

\newcommand\numberthis{\addtocounter{equation}{1}\tag{\theequation}}
\newcommand{\Exp}[1]{\mathbb{E}\lrp{#1}}
\newcommand{\E}[2]{\mathbb{E}_{#1}\lrp{#2}}

\newcommand{\inv}{^{\text{-}1}}

\newcommand\blfootnote[1]{%
  \begingroup
  \renewcommand\thefootnote{}\footnote{#1}%
  \addtocounter{footnote}{-1}%
  \endgroup
}

\ShortHeadings{}{}
\firstpageno{1}

\begin{document}

\title{Optimal $\delta$-Correct Best-Arm Selection for Heavy-Tailed Distributions}

\author{\name Shubhada Agrawal {\email sagrawal362@gatech.edu} \\
		\addr Georgia Institute of Technology\\
		\name Sandeep Juneja \email juneja@tifr.res.in \\
       \addr  TIFR, Mumbai\\
   		\name Peter Glynn \email glynn@stanford.edu\\
   		\addr Stanford University}

\editor{}



\maketitle\blfootnote{This is an updated version of \cite{agrawal2020optimal}. Some subtle technical errors were present in the previous version. These stemmed from working with the Wasserstein metric on the space of probability measures and due to improper handling of constants in the concentration result developed in that work. To address these, in this updated version, we use the L\'evy metric throughout and develop a tighter concentration result.}

\vspace{-3em}
\begin{abstract}
		Given a finite set of unknown distributions $\textit{or arms}$ that can be sampled, we consider the problem of identifying the one with the maximum mean using a $\delta$-correct algorithm (an adaptive, sequential algorithm that restricts the probability of error to a specified $\delta$) that has minimum sample complexity.   Lower bounds for $\delta$-correct algorithms are well known. $\delta$-correct algorithms that match the lower bound asymptotically as $\delta$ reduces to zero have been previously developed when arm distributions are restricted to a single parameter exponential family. In this paper, we first observe a negative result that some restrictions are essential, as otherwise, under a $\delta$-correct algorithm, distributions with unbounded support would require an infinite number of samples in expectation. We then propose a $\delta$-correct algorithm that matches the lower bound as $\delta$ reduces to zero under the mild restriction that a known bound on the expectation of $(1+\epsilon)^{th}$ moment of the underlying random variables exists, for $\epsilon > 0$. We also propose batch processing and identify near-optimal batch sizes to speed up the proposed algorithm substantially. The best-arm problem has many learning applications, including recommendation systems and product selection. It is also a well-studied classic problem in the simulation community. 
\end{abstract}
\section{Introduction}
Given a vector of unknown arms or probability distributions that can be sampled, we consider algorithms that sequentially sample from or {\em pull} these arms and, at termination, identify the best arm, i.e., the arm with the largest mean. The algorithms considered provide  $\delta$-correct probabilistic guarantees, i.e., the probability of identifying an incorrect arm is bounded from above by a pre-specified $\delta>0$. Further, the $\delta$-correct algorithms aim to minimize the sample complexity or the expected total number of arms pulled before they terminate. 

This best-arm problem is well-studied in the literature (see, e.g.,  in \emph{learning} - ~\cite{wang2021fast,degenne2019non,chen2017nearly,chen17b, garivier2016optimal,kaufmann2016complexity,russo2016simple,jamieson2014lil,karnin13,kaufmann2013information,kalyanakrishnan2012pac,gabillon2012best,bubeck2011pure,audibert2010best,even2006action,mannor2004sample}; in earlier \emph{statistics literature} - \cite{jennison1982asymptotically,bechhofer1968sequential,paulson1964sequential}; in \emph{simulation}  - \cite{glynn2004large,kim2001fully,chen2000simulation,dai1996converge,ho1992ordinal}). 

The  $\delta$-correct guarantee of the algorithm imposes constraints on the expected number of times the algorithm must pull each arm. The lower bounding techniques for these problems rely on {\em change of measure} based analysis that goes back at least to  \cite{lai1985asymptotically}. Also see \cite{mannor2004sample} and \cite{BURNETAS1996}. Recently, several works proposed inequalities capturing these changes of measures, relating the log-likelihood ratio under the two measures to the probability of events under these measures (see, for example, \cite{garivier2019explore}).

\cite{garivier2016optimal} consider the best arm problem assuming that each arm distribution belongs to a single parameter exponential family (SPEF). Under this restriction, they arrive at an {\em asymptotically optimal} algorithm having a sample complexity matching the derived lower bound asymptotically as $\delta \rightarrow 0$. SPEF  distributions include Bernoulli, Poisson, and Gaussian distributions with known variance.

In practice, it is rarely the case (other than in the Bernoulli setting) that the distributions are from SPEF. For example, \cite{dubey2019thompson} argue that the rewards from the interactions in many real-world systems often follow heavy-tailed distributions. It is well known that the distribution of online behavior on websites \citep{kumar2010characterization}, distributions of packet-delays in congested networks \citep{yu2018pure,liebeherr2012delay}, etc., do not enjoy the sub-Gaussian properties. Best arm problems, in particular, arise in many settings. For instance, one can view the selection of the best product version to roll out for production and sale after a set of expensive pilot trials among many competing versions to be the best arm problem. In simulation theory, selecting the best design amongst many (based on output from a simulation model) is a classic problem with applications to manufacturing, road and communications network design, etc. In these and many other settings, the underlying distributions can be very general and may not be modeled well by a SPEF distribution. Hence, there is a need for a general theory and efficient algorithms with broader applicability. We substantially address this issue. 

\paragraph{Contributions.} Our first contribution is an impossibility result illustrating why some distributional restrictions on arms are necessary for $\delta$-correct algorithms to be practical. Consider an algorithm that provides  $\delta$-correct guarantees when acting on a finite set of distributions, each belonging to a collection $\mathcal{H}$, where  $\mathcal{H}$ comprises distributions with unbounded support that are {\em $\KL$ right dense} (defined in Section~\ref{sec:problem.baimean}). In this setup, we show that the expected number of samples required by any algorithm on every instance must be infinite. Examples of such $\mathcal{H}$ include all light-tailed distributions with unbounded support (a distribution is said to be light-tailed if its moment-generating function is finite in a neighborhood of zero). Another example is a collection of unbounded distributions supported on $\Re$, which have a finite absolute $p^{th}$ moment, for some $p \geq 1$.

To arrive at an effective $\delta$-correct algorithm, we restrict arm distributions to the collection
\begin{equation}
\label{eq:LClass}
\mathcal{L} \triangleq \lrset{\eta \in \mathcal{P}(\Re) : \E{X\sim \eta}{\abs{X}^{1+\epsilon}} \leq B} ,
\end{equation}
where \(\mathcal{P}(\Re)\) denotes the collection of all probability measures with support in \(\Re \), and $B > 0$ and $\epsilon > 0$ are known fixed constants. Such bounds can often be found in simulation models using Lyapunov function-based techniques (see, e.g., \cite{glynn2008bounding}). 

With this mild restriction, we greatly simplify the lower bound on sample complexity for $\delta$-correct algorithms to understand its structure. Our {\bf main contribution} is an {\em asymptotically optimal} $\delta$-correct algorithm whose sample complexity matches the lower bound, asymptotically as $\delta \rightarrow 0$, on every bandit instance with distributions from $\cal L$. The proposed algorithm can be specialized to bandits with distributions in $[0,1]$, giving optimal algorithms for this setting. 

Crucial to our analysis of the lower bound and the $\delta$-correct algorithm are the $\KL$ projection functionals $\KLinfL: \mathcal P(\Re) \times \Re \rightarrow \Re^+$ and $\KLinfU: \mathcal P(\Re) \times \Re \rightarrow \Re^+,$ defined below. For probability measures $\kappa_1$ and $\kappa_2$, let $\KL(\kappa_1,\kappa_2)$ denote the Kullback-Leibler divergence between them (defined in Section~\ref{sec:problem.baimean}), and let $m(\kappa)$ denote the mean of the probability distribution $\kappa$. Then, for $\eta\in \mathcal{P}(\Re)$ and $x \in \Re$, 
\begin{equation} \label{eq:KLinf}
\KLinfU(\eta,x) := \min\limits_{\kappa \in \mathcal L, m(\kappa) \ge x}~ \KL(\eta,\kappa), \quad  \KLinfL(\eta,x) := \min\limits_{ \kappa \in \mathcal L, m(\kappa) \le x} ~\KL(\eta,\kappa).
\end{equation}
Heuristically, $\KLinfU(\eta,x)$ measures the difficulty of separating the distribution $\eta$ from all the other distributions in $\mathcal{L}$ whose mean is at least $x$. It equals
zero when $x=m(\eta)$ and \(\eta\in\mathcal{L}\). Symmetric statements hold for $\KLinfL$. 

Unlike in the large-deviations theory (\cite{dembo2010large}), the optimization is in the second argument of $\KL$ in the definitions above. While the theory with optimization in the first argument of $\KL$ is well-developed, in this work, we study the structure and properties of these $\KL$-projection functions defined in~\eqref{eq:KLinf} in great detail and characterize their optimal solutions. We develop a concentration inequality for scaled sums of empirical versions of these functions, which plays a key role in the proof of \(\delta \)-correctness of the proposed algorithm. Our concentration inequality relies on arriving at simpler (dual) representations for the projection functionals. To this end, we substantially extend the similar representations developed by  \cite{HondaBounded10} for bounded random variables to random variables belonging to ${\cal L}$. We point out that \citet{HondaBounded10} solve the regret minimization problem for stochastic bandits.

The proposed algorithm is a plug-and-play strategy that solves a max-min optimization problem for the empirical distributions at each step. In bounding its sample complexity, we exploit the joint continuity of \(\KLinfL\) and $\KLinfU$ in their arguments, as well as the continuity of the optimizers for the max-min problem, with respect to the underlying arm distributions. A key challenge in proving these properties is arriving at the topology in the space of probability measures, ensuring the continuity of these functions, along with fast convergence of the empirical estimates to the true ones. We consider the L\'evy metric on the space of probability distributions (Section~\ref{sec:WC}). The L\'evy metric is relatively tractable to work with. Moreover, ${\cal L}$ is a compact space under this distance (Lemma~\ref{lem:propL}). This, in particular, allows us to use the well-known Berge's Maximum Theorem and related results (Section~\ref{sec:Berge}) to derive the requisite continuity properties.

In, e.g.,  \cite{garivier2016optimal,kalyanakrishnan2012pac}, the proposed algorithms solve the lower bound problem at every iteration. However, this can be computationally demanding, particularly in the generality considered in the current work. We instead solve this problem in batches and arrive at near-optimal batch sizes that minimize the overall cost of the algorithm. This includes the computational effort and the cost of sampling from an arm. The latter may be significant when each sample is costly, for example, when the samples correspond to a clinical trial result. To the best of our knowledge, this idea of updating the statistics only after a batch of samples to balance the computational and statistical costs was not previously considered in the literature of MAB. 

In this work, we keep the batch size constant and give it as an input to the algorithm. For the optimal batch size, the overall cost of the algorithm reduces from being quadratic in the number of samples to linear. The sample complexity of the resulting algorithm is a constant factor away from the lower bound. This cost is still quadratic in the total number of samples for sub-optimal batch sizes but with a smaller constant. However, the resulting algorithm exactly matches the lower bound.

As previously mentioned in a footnote, this work is an updated version of \cite{agrawal2020optimal}. Some subtle technical errors were present in that version that stemmed from working with the Wasserstein metric on the space of probability measures. In addition, in the concentration result developed, the constant terms were dependent on the unknown bandit instance, which resulted in a stopping threshold (an input to the proposed algorithm) dependent on this unknown constant. We fixed these errors in an erratum posted online \citep{agrawal2020perratum}. In this updated version, we incorporate these corrections and improve the analysis. We use the L\'evy metric instead of the Wasserstein metric throughout. To fix the stopping threshold, we present a much tighter martingale-based concentration result. In addition, towards the end, we discuss the special setting of bandits with bounded-support distributions.

\paragraph{Roadmap.} In Section~\ref{sec:problem.baimean}, we review some background material and present an impossibility result illustrating the need for distributional restrictions on arms. In Section~\ref{sec:lb.mean}, an efficient lower bound for $\delta$-correct algorithms for the fixed-confidence BAI problem is provided when the arm distributions are restricted to ${\cal L}$. The algorithm that matches this lower bound asymptotically as $\delta \rightarrow 0$ is developed in Section~\ref{sec:trackstop.mean}. We present the theoretical guarantees of the proposed algorithm in Section~\ref{sec:tg.meanBAI}. Discussion on optimal batch size and a numerical experiment are shown in Sections~\ref{sec:optbatch.mean} and~\ref{numerics}, respectively. We present the optimal algorithm for the MAB with bounded-support distributions in Section~\ref{sec:bdd.meanBAI}. While the critical ideas in some of the proofs are outlined in the main body, details of the proofs are all given in the appendices.

\section{Background and the Impossibility Result}
\label{sec:problem.baimean}

Let $\mathcal P(\Re)$ denote the collection of all probability measures on $\Re$. For $\eta$ and $\kappa$ in $\mathcal P(\Re)$, let 
$$\KL(\eta,\;\kappa)= \int\log\lrp{\frac{d \eta}{d \kappa}(x)}d \eta(x)$$ 
denote the Kullback-Leibler divergence between them. For $p,q\!\in\!(0,\!1)$, let $d(p,q)$ denote the KL-divergence between Bernoulli distributions with mean $p$ and $q$, respectively, i.e., 
$$d(p, q) = p \log \left ( \frac{p}{q} \right ) + (1-p)  \log \left ( \frac{1-p}{1-q} \right ).$$ 
Furthermore, let $[K]:= \{1,\dots, K\}$, and for $\eta\in\mathcal P(\Re)$, let $m(\eta)$ denote its mean. 

Let $\mathcal{H} \subset \mathcal P(\Re)$ denote the universe of probability distributions for which we aim to devise $\delta$-correct algorithms. We assume each distribution in $\mathcal{H}$ has a finite mean. Consider a vector of distributions $\mu=(\mu_1, \ldots, \mu_K)$ such that $\mu_i \in \mathcal H$, for all $i \in [K]$.  Let $a^*(\mu)$ denote the optimal arm in $\mu$ and let $\operatorname{Alt}(\mu)$ denote the collection of all the bandit instances with optimal arm different from that in $\mu$, i.e.,
\[ \operatorname{Alt}(\mu)=\lrset{ \nu \in \mathcal H^K: ~  m(\nu_{a^*(\mu)}) < \max\limits_{b\in[K]} ~ m(\nu_b) }.  \]

Under a  $\delta$-correct algorithm acting on $\mu$, for $\delta \in (0,1)$ and $\nu\in\operatorname{Alt}(\mu)$, recall the following consequence of data-processing inequality (see \citet[Section 2]{garivier2019explore}):
\begin{equation}
\label{eq:lb}
\sum_{i=1}^K \mathbb{E}_{\mu} (N_i(\tau_\delta)) \, \KL(\mu_i , \nu_i) \geq d(\delta, 1-\delta) \geq \log \left (\frac{1}{2.4 \delta} \right ),
\end{equation}
where $N_i(t)$ denotes the number of times arm $i$ is pulled by the algorithm in $t$ trials, and $\tau_\delta= \sum_{i=1}^K N_i(\tau_\delta)$ denotes the algorithm's termination time. It relates the expected log-likelihood ratio of observing samples under $\mu$ and $\nu$ (l.h.s. in~\eqref{eq:lb}) and the Bernoulli KL-divergence between the probabilities of events under the two instances (r.h.s. in~\eqref{eq:lb}). 

The following lemma helps in proving our negative result in Theorem~\ref{theorem:thm_main}. 
\begin{lemma}  \label{lem:lem201}
     For $\mathcal H = \mathcal P(\Re)$, given  $\eta \in \mathcal H$, for any finite $a >0$ and $b> m(\eta)$,
     there exists a distribution $\kappa \in \cal H$ such that 
     \begin{equation} \label{eqn:neg_result_odd_space1}
     \KL({\eta}, \kappa)   \leq a \;\;\;\;\;\text{ and }\;\;\;\;\; m(\kappa) \geq b.
     \end{equation}
\end{lemma}
It shows that given a distribution $\eta$, there exists another distribution $\kappa$ with the given mean $b$, such that $\kappa$ is close in $\KL$ divergence to $\eta$. It is easy to see that we can always construct $\kappa$ from $\eta$ by scaling down $\eta$ by a constant close to $1$ and putting the remaining mass at a point in the extreme right, say $y$. The scaling factor and the point $y$ can be chosen as a function of $a$ and $b$. This argument is made rigorous in Section~\ref{sec:klrightdenselemma}. 

\begin{definition}
     A collection of probability distributions $\mathcal{H}$  is referred to as $\KL$ right dense,   if for every $\eta \in \mathcal{H}$,
     and every $a >0$, $b > m(\eta)$,  there exists a distribution  $\kappa \in \mathcal{H}$
     such that (\ref{eqn:neg_result_odd_space1}) holds.
\end{definition}

Observe that a necessary condition for $\mathcal{H}$ to be $\KL$ right dense is that a uniform upper bound on the supports of each member of $\cal H$ does not exist. 

\begin{theorem} \label{theorem:thm_main} 
     For a $\KL$ right dense collection ${\mathcal{H}}$, $\mu\in \mathcal H^K$ such that $a^*(\mu) = 1$, and for a $\delta$-correct algorithm  operating on $\mu$, $\mathbb{E}_\mu\lrp{N_k(\tau_\delta)}=  \infty$ for all $k\in\lrset{2, \dots, K}$.
\end{theorem}

From Lemma~\ref{lem:lem201}, for any $k \geq 2$, one can easily find a bandit instance $\nu \in {\operatorname{Alt}(\mu)}$ such that for some $k\ne 1$ and $\forall i \ne k$, $\nu_i= \mu_i$, $m(\nu_k) > m(\mu_1)$, and $\KL(\mu_k, \nu_k)$ is arbitrarily small. Theorem~\ref{theorem:thm_main} now follows from (\ref{eq:lb}). 

When the only information available about a distribution is that its mean exists, \cite{bahadur1956nonexistence} prove a related impossibility result that there does not exist an effective test of hypothesis for testing  whether the mean of the distribution is zero (also see \cite{Lehmann2006testing}). However, to the best of our knowledge, Theorem~\ref{theorem:thm_main} is the first impossibility result in the best arm (or, equivalently, ranking and selection) setting. 

Theorem~\ref{theorem:thm_main} suggests that further restrictions are needed on  \(\mathcal{H} \) for $\delta$-correct algorithms to provide reasonable performance guarantees. To this end, we limit our analysis to the class \(\mathcal{L}\) defined in (\ref{eq:LClass}) earlier, i.e., we set $\mathcal H = \mathcal L$. Notice that this class includes several heavy-tailed distributions. The techniques developed in the current work to handle this general class of distributions can be specialized for simpler classes to deduce similar results. To illustrate this, in Section~\ref{sec:bdd.meanBAI}, we present optimal bounds for distributions supported in $[0,1]$, highlighting the modifications in the proposed framework.

\section{Lower Bound}\label{sec:lb.mean}
In this section, we present a lower bound on the expected number of samples that a $\delta$-correct algorithm needs to generate when acting on a given bandit instance. This lower bound involves a max-min problem, where the minimization is in the space of probability measures. We characterize its optimizers and bring out the structural properties of the max-min problem, which forms a crucial component of the proposed algorithm.

For $x \in \Re$, let $\delta_x$ denote the probability measure that assigns unit mass to $\lrset{x}$, and $0$ everywhere else. Define $$M := \left[-B^{\frac{1}{1+\epsilon}}, B^{\frac{1}{1+\epsilon}}\right].$$ 
Let $M^o$ denote its interior. The following lemma characterizes means of distributions in $\cal L$. 

\begin{lemma}\label{lem:MeanBounds}
For $\eta\in \mathcal L$, $m(\eta) \in M$. Moreover, $\delta_{B^{{1}/\lrp{1+\epsilon}}}$ and $\delta_{-B^{{1}/\lrp{1+\epsilon}}}$ are the unique probability measures in $\cal L$ with means $B^{\frac{1}{1+\epsilon}}$ and $-B^{\frac{1}{1+\epsilon}}$, respectively.
\end{lemma}
The results in the above lemma follow from applications of Jensen's inequality. We refer the reader to Section~\ref{app:lb:MeanBounds} for its proof. 

Next, let $\mathcal M := \mathcal L^K$, $\mu\in\mathcal M$, and let $a^*(\mu)$ denote the unique best-arm in $\mu$. For simplicity of presentation, we assume that $a^*(\mu) = 1$.  Henceforth, we set $\operatorname{Alt}(\mu)$ to be the collection of bandit instances from $\cal M$ in which arm $1$ is sub-optimal, i.e., 
\[ \operatorname{Alt}(\mu) =  \bigcup\limits_{j\ne 1} \lrset{ \nu \in \mathcal M: ~ m(\nu_j) > m(\nu_1)}. \numberthis\label{eq:altmu.mean} \]
From (\ref{eq:lb}) it follows that for any $\delta$-correct algorithm acting on \(\mu\), for every $\nu \in \operatorname{Alt}(\mu)$, 
\begin{equation*}
\mathbb{E}_{\mu}(\tau_\delta) \inf\limits_{\nu \operatorname{Alt}(\mu)} ~ \sum_{i=1}^K \frac{\mathbb{E}_{\mu} (N_i(\tau_\delta))}{\mathbb{E}_{\mu}(\tau_\delta)} \KL(\mu_i , \nu_i) \geq \log \left(\frac{1}{2.4 \delta} \right ). 
\end{equation*}

Let \(\Sigma_K \) denote probability simplex in \(\Re^K\). It then follows that
\[\mathbb{E}_{\mu}(\tau_\delta) \ge \frac{\log \frac{1}{2.4\delta}}{V(\mu)},\quad \text{ where } \quad V(\mu) = \sup\limits_{t\in\Sigma^K}\inf\limits_{\nu\in \operatorname{Alt}(\mu)} ~ \sum_{i=1}^K t_i\KL(\mu_i , \nu_i). \numberthis \label{eq:lb3}
\]

We remark that the problem of computing the lower bound on $\mathbb{E}_{\mu}(\tau_\delta)$ reduces to solving the above max-min problem. While in the case of SPEF, the above inner minimization problem in (\ref{eq:lb3}) is essentially an optimization problem in the Euclidean space, in our setting, it is in the space of probability measures. We now exploit the structure of this max-min problem to arrive at an equivalent but more straightforward representation, which will be essential for implementing our algorithm and its analysis. 

\subsection{Simplification of the Lower Bound and its Properties}\label{sec:simpLB}

Recall that for \(\eta\in\mathcal{P}(\Re)\) and \(x \in \Re\),   \(\KLinfU(\eta,x)\) and $\KLinfL(\eta,x)$ are defined in (\ref{eq:KLinf}). The following lemma gives an alternative representation for the max-min problem associated with the lower bound in~\eqref{eq:lb3}.

\begin{lemma} \label{lem:maxminsimple.mean}
	For $\mu \in \mathcal M$ with $a^*(\mu) = 1$, the inner minimization problem in $V(\mu)$ equals 
	\[ \min\limits_{j\ne 1} ~ \inf \limits_{x\le y} \lrset{t_1 \KLinfL(\mu_1, x) + t_j \KLinfU(\mu_j, y)},  \]
	and hence, 
	\[V(\mu) = \sup\limits_{t\in\Sigma_K} ~ \min\limits_{j\ne 1}~ \inf\limits_{x\le y} ~ \lrset{t_1 \KLinfL(\mu_1, x) + t_j \KLinfU(\mu_j, y)} .\numberthis \label{eq:Vmu.mean} \]
\end{lemma} 

The proof of the above lemma uses the observation that the bandit instances from $\operatorname{Alt}(\mu)$ attaining the infimum in $V(\mu)$ differ from $\mu$ in distributions corresponding to at most $2$ arms. See \cite{garivier2016optimal} for a similar observation in the SPEF setting. We refer the reader to Section~\ref{sec:proofmaxminmean} for proof of the above result.

To characterize the optimizers in (\ref{eq:lb3}), or equivalently in~\eqref{eq:Vmu.mean}, and to study their topological properties that will be useful in the analysis of the proposed algorithm, we endow $\mathcal P(\Re)$, and hence $\mathcal L$, with the topology of weak convergence (see Section~\ref{sec:WC} for a background). It is metrized by the \emph{L\'evy metric} on \( \mathcal{P}(\Re) \) (denoted by \(d_L\)), which is defined below. Thus, convergence of sequences in the weak topology is equivalent to their convergence in the L\'evy metric (see, for example, \citet[Theorem 6.8]{billingsley2013convergence}, \citet[Theorem D.8]{dembo2010large}). 
\begin{definition}[L\'evy metric]
{For \(\eta , \kappa \in \mathcal{P}(\Re) \), 
$$d_L(\eta,\kappa):= \inf\lrset{\gamma > 0 : F_\eta(x-\gamma)-\gamma \leq F_\kappa(x) \leq F_\eta(x+\gamma) + \gamma, ~ \forall x\in\Re },$$
where for $x\in\Re$, $F_\eta(x) := \eta((-\infty, x])$ and $F_\kappa(x) := \kappa((-\infty, x])$.}
\end{definition}


The following lemma establishes properties of the distributions in $\cal L$ endowed with L\'evy metric, which will be used to prove the convergence of the empirical estimates used by the algorithm. We refer the reader to Section~\ref{sec:proof_propL} for proof of the Lemma~\ref{lem:propL}.
\begin{lemma}\label{lem:propL}
	\(\cal L\) is convex and compact in the topology of weak convergence. Random variables associated with its members form a uniformly integrable collection. Furthermore, for a sequence of distributions \(\eta_n \in \cal L\) such that $\eta_n$ converges to  \(\eta \in \cal L \) in L\'evy metric, we have \(m(\eta_n) \longrightarrow m(\eta) \).
\end{lemma}

Next, observe that the max-min problem in~\eqref{eq:Vmu.mean} involves the two functions $\KLinfL$ and $\KLinfU$, defined in~\eqref{eq:KLinf}. We now present structural and topological properties of these functions that will bring out the structure of the max-min problem in~\eqref{eq:Vmu.mean}. Joint continuity of these functions, in particular, will play a crucial role in proving sample complexity bound for the proposed algorithm.

\begin{lemma}\label{lem:propklinf}
 Consider L\'evy metric (or topology of weak convergence) on $\mathcal P(\Re)$. 
	\begin{enumerate}[label=(\alph*)]
		\item \label{lem:ConvexandstconvexklinfU} $\KLinfU$ is a jointly convex function. For $\eta \in \mathcal L$, and $x \in M^o$ such that $x > m(\eta)$, $\KLinfU(\eta,x)$ is a strictly convex function of $x$. Moreover, it equals $0$ for $x \le m(\eta)$.
		\item \label{lem:ConvexandstconvexklinfL} $\KLinfL$ is a jointly convex function. For $\eta \in \mathcal L$, and $x\in M^o$ such that $x < m(\eta)$, $\KLinfL(\eta,x)$ is a strictly convex function of $x$. Moreover, it equals $0$ for $x \ge m(\eta)$.
		\item \label{lem:cont_klinf_mean} When restricted to $\mathcal L \times M^o$, $\KLinfU$ and $\KLinfL$ are  jointly continuous in their arguments. 
	\end{enumerate}
\end{lemma}

From~\eqref{eq:KLinf}, we have that for $\eta \in \mathcal L$ and $x\le m(\eta)$, $\KLinfU(\eta,x) = 0$. This follows from feasibility of $\eta$ for $\KLinfU$ and non-negativity of $\KL$. Similarly, for $x\ge m(\eta)$, $\KLinfL(\eta,x) = 0$. Convexity for these functions also follows from their definitions, along with convexity of $\KL$ divergence. The proof for joint continuity is technically challenging. We separately prove the joint lower-semicontinuity and upper-semicontinuity. The proofs use the properties of probability measures in \( \cal L \) from Lemma~\ref{lem:propL} and the dual representations for \( \KLinfL\) and  \(\KLinfU\) from  Theorem~\ref{th:klinfDual_mean}, presented later. Lower semicontinuity is argued from the primal formulation. To prove the upper-semicontinuity, we use the dual representations along with Berge's Theorem and associated results mentioned in Section~\ref{sec:Berge}. We refer the reader to Section~\ref{sec:prop.klinf.mean} for a detailed proof of the lemma.

Using the properties from the above lemma, we further simplify the representation for max-min problem in ~\eqref{eq:Vmu.mean}. To this end, we introduce some notation.

\paragraph{Notation. }Given $\mu \in \mathcal M$, recall that $a^*(\mu) = 1$. For \(j \in \lrset{2,\dots,K}\), consider functions $g_j: \mathcal M \times \Re^3\rightarrow \Re$  and  $G_j: \mathcal M\times\Re^2\rightarrow\Re. $ 
Let $\Re^+$ denote the set of non-negative reals. Then, for $x\in\Re$, $y\in\Re^+$ and $z\in\Re^+$, 
\[g_j(\mu,y,z,x)  :=  y \KLinfL(\mu_1,x)+z \KLinfU(\mu_j,x)\]
and
\[G_j(\mu,y,z) := \inf_x~\lrset{ g_j(y,z,x) : x\in  [m(\mu_j), m(\mu_1)]}. \]
\begin{lemma}\label{lem:positiveG}
	For $\mu\in\mathcal M^o$ with $a^*(\mu) = 1$, and for $y > 0$ and $z > 0$, $G_j(\mu,y,z) > 0~$ for all $j\ne 1$.
\end{lemma}

The lemma above shows that $G_j$ is strictly positive if $y$ and $z$ are non-zero. The proof follows immediately from continuity of the two $\KL$ projection functions (Lemma~\ref{lem:propklinf}\ref{lem:cont_klinf_mean}) along with the observation that for $x$ in the specified range, at least one of $\KLinfL$ and $\KLinfU$ is strictly positive. We refer the reader to Section~\ref{app:positiveG} for a proof of the lemma.

Next, let $x^*_j: \mathcal M^o\times [0,1]^2 \rightarrow \Re$, where for $\nu\in\mathcal M^o$, $y\in [0,1]$, $z\in [0,1]$, and $j\ne 1$,  $x^*_j(\nu, y,z)$ denotes the set of minimizers in $G_j(\nu,y,z)$. Further, let $t^*:\mathcal M \rightarrow \Sigma_K$, where for $\nu\in\mathcal M$,  $t^*(\nu)$ denotes the set of maximizers in $V(\nu)$. Overloading notation, we use $x^*_j(\nu)$ to denote the set of minimizers in $G_j(\nu, t^*_1(\nu), t^*_j(\nu))$. 

We now simplify the minimization problem in $V(\cdot)$ further by restricting the feasible region of the inner infimum to a compact set. This representation will enable us to use the classical Berge's Theorem to prove the continuity of the set of maximizers in $V(\cdot)$ as a function of the bandit instance. We refer the reader to Section~\ref{sec:Berge} for a definition of continuity for set-valued maps and details of Berge's Theorem and related results.

\begin{lemma}\label{lem:AltVmu.mean}
For $\mu \in \mathcal M^o$ with $a^*(\mu) = 1$, $V(\mu) > 0$, and it simplifies as 
\[ V(\mu) = \max\limits_{t\in\Sigma_K}~\min\limits_{j\ne 1 }~ G_j(\mu,t_1,t_j).\numberthis\label{eq:altVmurep.mean} \]
$x^{*}_j(\cdot,\cdot,\cdot)$ is singleton and a jointly-continuous function of its arguments. Moreover, for each $j\ne 1$, functions $G_j$ are jointly continuous. 
\end{lemma}

We will use (\ref{eq:altVmurep.mean}) and the formulation of $V(\mu)$ in (\ref{eq:Vmu.mean}) interchangeably. To arrive at the alternative representation for $V(\mu)$, we use strict convexity of the two $\KL$-projection functions (Lemma~\ref{lem:propklinf}\ref{lem:ConvexandstconvexklinfU} and \ref{lem:ConvexandstconvexklinfL}) for appropriate ranges of the second argument. To prove the continuity of various functions, we use Berge's Theorem at several steps, along with joint continuity of $\KLinfU$ and $\KLinfL$ in their respective arguments (Lemma~\ref{lem:propklinf}\ref{lem:cont_klinf_mean}). Strict positivity of $V(\mu)$, then, follows from Lemma~\ref{lem:positiveG}. We refer the reader to Section~\ref{sec:prooflemAltVmu.mean} for proof of the above result.

In the following theorem, we show that if there is a unique best arm in $\mu$, then there is a unique maximizer for the lower bound problem. Furthermore, the maximizer, $t^*(\cdot)$, is a continuous function of the bandit instance (in the topology of weak convergence on $\cal L$). As mentioned earlier, the proposed algorithm is a plug-and-play strategy that solves the max-min problem for the empirical distributions and tracks its maximizers. The continuity of the maximizers will be crucial for the proof of convergence of algorithmic estimates and allocation to the corresponding true values.

\begin{theorem}\label{PropOpt}
     For $\mu\in \mathcal M^o$ with $a^*(\mu) = 1$, the set $t^*(\mu)$ is singleton. Moreover,
     \begin{enumerate}[label=(\alph*)]
          \item \(\forall i,\; t^*_i(\mu) > 0~\) and \(~\sum_{i}t^*_i(\mu)=1 \) \label{PropOpt_gre0}, 
          \item \(G_2(\mu, t^*_1(\mu),t^*_2(\mu)) = G_j(\mu, t^*_1(\mu),t^*_j(\mu)), \text{ for } j \in \lrset{3,\ldots,K}\). \label{PropOpt_EqualVals}
     \end{enumerate}
     Furthermore, \(V(\mu)\) equals \(G_2(\mu, t^{*}_1(\mu),t^{*}_2(\mu))\),  and \(t^*(\cdot)\) is a continuous function.
\end{theorem}

\noindent\emph{Proof sketch: } We show \ref{PropOpt_gre0} and \ref{PropOpt_EqualVals} by contradiction. For \ref{PropOpt_gre0}, we argue that if for $t\in t^*(\mu)$ there is  an index $i$ such that $t_i = 0$, then $V(\mu) = 0$ contradicting the strict positivity of $V(\mu)$ (Lemma~\ref{lem:AltVmu.mean}). For \ref{PropOpt_EqualVals}, we first establish the monotonicity of $G_j(\mu,1,\cdot)$ for all $j\ne 1$. We, then, argue that if for $t\in t^*(\mu)$, \ref{PropOpt_EqualVals} doesn't hold, then by continuity of $G_j(\mu, \cdot, \cdot)$ (Lemma~\ref{lem:AltVmu.mean}), we can perturb $t$ slightly to arrive at $t'$ for which the value of $\min_j G_j(\mu, t'_1, t'_j)$ increases, contradicting the optimality of $t$.

Next, we use \ref{PropOpt_EqualVals} and strict monotonicity of $G_j(\mu,1,\cdot)$ to prove the uniqueness of maximizers, $t^*$. To prove the continuity of $t^*$, we first show that it is an upper-hemicontinuous map. Its upper-hemicontinuity follows from Berge's Theorem applied to the max-min problem. Then, continuity follows since $t^*(\mu)$ is a singleton. We refer the reader to Section~\ref{sec:Berge} for the definition of upper-hemicontinuity and Berge's and related theorems and to Section~\ref{sec:proofPropOpt} for a detailed proof of the theorem.

\subsection{Understanding the Lower Bound: Dual Formulations}\label{sec:dual}
The proposed algorithm, discussed in Section~\ref{sec:trackstop.mean}, relies on solving the max-min problem in (\ref{eq:Vmu.mean}) repeatedly with $\mu$ replaced by its running estimator, which may not belong to $\cal M$. Thus, efficiently solving for the optimal proportions $t^*(\nu)$ for any $\nu \in \mathcal P(\Re)$ is crucial for the algorithm's implementation. Towards this, in this section, we develop alternative representations (dual) for the two $\KL$-projection functionals. Unlike the primal problems for these, which are optimization problems in the space of probability measures, the dual formulations are convex optimization problems in two real-valued variables.

\paragraph{Notation. }Let $\Re^+$ and $\Re^-$ denote the collection of non-negative and non-positive real numbers. For \(\eta \in\mathcal{P}(\Re)\), let \(\Supp(\eta)\) denote the collection of points in the support of measure \(\eta\). For \(x\in M^o\), \(\bm{\lambda} \in \Re^2\), $\bm{\gamma}\in \Re^2$, and \(X \in \Re \), define \[
g^U( X,\bm{\lambda}, x) = 1 - \lambda_1 (X-x)  - \lambda_2 (B-\abs{X}^{1+\epsilon}).  \] 
Similarly, consider a symmetric quantity, 
\[g^L(X,\bm{\gamma},x) =  1 + \gamma_1(X-x) -\gamma_2(B-\abs{X}^{1+\epsilon}) .\] 
Furthermore, define the sets 
\[{S}^U(x) = \lrset{\lambda_1 \geq 0, \lambda_2\geq 0: ~~ 1 + \lambda_1 x - \lambda_2 B -  \frac{\epsilon\lambda^{1+\frac{1}{\epsilon}}_1}{(1+\epsilon)^{1+\frac{1}{\epsilon}} \lambda^{\frac{1}{\epsilon}}_2 } \ge 0 }, \numberthis\label{eq:SU} \]
and 
\[{S}^L(x) = \lrset{\gamma_1\geq 0, \gamma_2\geq 0: ~~ 1 - \gamma_1 x -\gamma_2 B - \frac{\epsilon \gamma^{1+\frac{1}{\epsilon}}_1 }{ \gamma^{\frac{1}{\epsilon}}_2 (1+\epsilon)^{1+ \frac{1}{\epsilon}} } }. \numberthis\label{eq:SL} \]  
The set $S^U(x)$ corresponds to values of the dual parameters for which $g^U(y,\bm{\lambda}, x) \ge 0$, for all $y\in\Re$. Similarly, $S^L(x)$ corresponds to those $\bm \gamma$ for which, $g^L(y, \bm \gamma, x) \ge 0$, for all $y\in\Re$.  Notice that these are convex sets. 

\begin{theorem}\label{th:klinfDual_mean}
     \begin{enumerate}[label=(\alph*)] 
     \item \label{th:klinfUDual_mean}For \(\eta\in\mathcal{P}(\Re) \) and \(x \in M^o \),  
          $$\KLinfU(\eta,x) = \max_{\bm{\lambda} \in S^U(x)}~ \E{\eta}{\log\lrp{g^U(X,\bm{\lambda},x)}} .$$ 
          The maximum in this expression is attained at a unique \({\bm \lambda^*} \in S^U(x)\). The unique probability measure \(\kappa^* \in \cal L\) that achieves infimum in the primal problem satisfies $$ \frac{d\kappa^*}{d\eta}(y) = \frac{1}{g^U(y,\bm{\lambda^*},x)}, \quad \text{ for } y\in \Supp(\eta).$$ 
          $\kappa^*$ has at most $1$ point in its support outside $\Supp(\eta)$, given by 
          \[ y^* = \lrp{\frac{\lambda^*_1}{\lambda^*_2 (1+\epsilon)}}^{\frac{1}{\epsilon}}.\]
          Moreover, this extra point, $y^*$, satisfies: $g^U(y^*,\bm{\lambda^*},x) = 0$. 
          Additionally, as long as $\eta$ is itself not feasible for $\KLinfU(\eta,x)$ problem, $\lambda^*_2 > 0$ and the moment constraint is tight, i.e., $\E{\kappa^*}{|X|^{1+\epsilon}} = B.$ Furthermore, for $\eta\in\mathcal L$ and $x\in M^o$ and $x > m(\eta)$, the mean constraint is also tight, i.e., $m(\kappa^*) = x$. 

     \item For \(\eta\in\mathcal{P}(\Re) \) and \(x \in M^o \),  
          $$\KLinfL(\eta,x) = \max_{\bm{\gamma} \in S^L(x)}~ \E{\eta}{\log\lrp{g^L(X,\bm{\gamma},x)}} .$$ 
          The maximum in this expression is attained at a unique \({\bm \gamma^*} \in S^L(x)\). The unique probability measure \(\kappa^* \in \cal L\)  that achieves infimum in the primal problem satisfies $$ \frac{d\kappa^*}{d\eta}(y) = \frac{1}{g^L(y,\bm{\gamma^*},x)}, \quad \text{ for } y\in \Supp(\eta).$$ 
          $\kappa^*$ has at most $1$ point in its support outside $\Supp(\eta)$, given by 
          \[ y^* = - \lrp{\frac{\gamma^*_1}{\gamma^*_2 (1+\epsilon)}}^{\frac{1}{\epsilon}}.\]          
          Moreover, this extra point, $y^*$, satisfies: $g^L(y^*,\bm{\gamma^*},x) = 0$. 
          Additionally, as long as $\eta$ is itself not feasible for $\KLinfL(\eta,x)$ problem, $\gamma^*_2 > 0 $ and the moment constraint is tight, i.e., $\E{\kappa^*}{|X|^{1+\epsilon}} = B.$ Furthermore, for $\eta\in\mathcal L$ and $x\in M^o$ and $x < m(\eta)$, the mean constraint is also tight, i.e., $m(\kappa^*) = x$. 
     \end{enumerate}
\end{theorem}
Since the primal optimization problem is in the space of probability measures, the alternative representations from the theorem above help in computing these $\KL$ projection functions. In the proof of the above theorem, we first show that the alternate representations for $\KLinfU$ and $\KLinfL$  are their dual formulations, and the parameters $\bm{\lambda}$ and $\bm{\gamma}$ are the corresponding dual variables. The sets ${S}^U$ and $S^L$ correspond to the feasible values of these dual variables. We argue that strong duality holds in our setting, implying that the primal-optimal value equals the dual-optimal value. The quantities, $\KLinfU$ and $\KLinfL$ are symmetric. Hence, we only present the proof for Theorem~\ref{th:klinfDual_mean}\ref{th:klinfUDual_mean} in Section~\ref{Proof_th:klinfDual_mean}. Proof for the other part follows similarly. 

\begin{remark}\emph{When the optimization is in the first argument of $\KL$ with a fixed second argument, it is well known that the optimizer is an exponentially twisted version of the second argument, which is also absolutely continuous with respect to the second argument. Notably, in the current setup with a fixed first argument but with optimization in the second argument of $\KL$, the optimizer is a (non-uniform) scaling of the first argument. Moreover, it is no longer restricted to be absolutely continuous with respect to the fixed first argument and can put mass at an additional point. The theorem above characterizes this unique point in terms of the optimal dual variables.}
\end{remark}

In our algorithm, the $V(\mu)$ statistic appears with the empirical estimates, instead of $\mu$, at two places - first for computing the maximizers of the max-min problem to decide which arm to sample next, second, our stopping statistic is related to the inner minimization problem in $V(\cdot)$. From~\eqref{eq:altVmurep.mean}, it is easily seen that efficiently computing the two $\KL$-projection functionals is crucial for implementing our sampling and stopping rules. Lemma~\ref{lem:compactdualspace} below shows that the feasible regions in the dual representations of these $\KL$ projection functions are compact. Thus, for computing $\KLinfU$ and $\KLinfL$, we can use versions of ellipsoid methods on the dual formulations. We refer the reader to Section~\ref{app:compactDual} for proof of Lemma~\ref{lem:compactdualspace}. 
Define  
\[ \bar{M}:= \left[ -B^{\frac{1}{1+\epsilon}}, B^{\frac{1}{1+\epsilon}} \right) \quad \text{ and }\quad \ubar{M}:= \left( -B^{\frac{1}{1+\epsilon}}, B^{\frac{1}{1+\epsilon}}\right]. \]
\begin{lemma}\label{lem:compactdualspace}
For $x \in \bar{M}$ and $y\in \ubar{M}$, $S^U(x)$ and $S^L(y)$ are convex and compact subsets of $\Re^2$. 
\end{lemma}

Theorem~\ref{th:klinfDual_mean} and Lemma~\ref{lem:compactdualspace} also play a crucial role in establishing the $\delta$-correctness of the proposed algorithm. Consider $\eta\in\cal L$ and the empirical distribution corresponding to $n$ samples from it, denoted by $\hat{\eta}_n$. Observe from the dual representations that $\KLinfU(\hat{\eta}_n, \cdot)$ and $\KLinfL(\hat{\eta}_n, \cdot)$ are maximum of empirical averages. As we will see later, for the algorithm to be $\delta$-correct, we require a concentration of these $\KL$-projection functions evaluated for the empirical distributions. For this, we use the dual formulations from Theorem~\ref{th:klinfDual_mean} to arrive at appropriate mixtures of super-martingales. At a very high level, for each fixed value of the dual variables, we construct a martingale parameterized by the dual variables. We mix these different martingales with a uniform prior over the dual regions. This is possible as the dual-feasible sets are compact (Lemma~\ref{lem:compactdualspace}). 

\section{The Optimal Algorithm and its Guarantees}\label{sec:trackstop.mean}
In this section, we propose a $\delta$-correct algorithm and show that its sample complexity matches the lower bound exactly as $\delta \rightarrow 0$. Recall that a $\delta$-correct algorithm has a \emph{sampling rule} that, at any stage, based on the information available, decides which arm to sample next. Further,  it has a \emph{stopping rule} and, at the stopping time, it announces an estimate for the arm with the highest mean reward while ensuring that the probability of incorrect assessment is less than or equal to a  pre-specified $\delta \in (0,1)$. We describe the proposed sampling and stopping rules in the following sections. 

\subsection{Randomized Tracking Rule}
Our sampling algorithm relies on solving the max-min lower bound problem with the vector of empirical distributions used as a proxy for the unknown true bandit instance $\mu$. Recall from Theorem~\ref{PropOpt} that $t^*(\cdot)$ is a continuous function on $\cal M$. Since at time $n$, the empirical distributions $\hat{\mu}(n)$ may not belong to $\mathcal M$, we first project $\hat{\mu}(n)$ to $\mathcal M$ in the Kolmogorov metric $d_K$ using the projection map $\Pi$ defined in (\ref{eq:KolProj.mean}) below. We then compute the maximizers for these projected empirical distributions, which guide the sampling strategy. The proposed algorithm also explores each arm to ensure that no arm is starved with insufficient samples. 

For $\eta\in\mathcal P(\Re)$, $x\in\Re$, recall that $F_\eta(x) = \eta((-\infty, x])$ denotes the distribution function of $\eta$.  Define the map $\Pi = (\tilde{\Pi}, \dots, \tilde{\Pi})$, where for $\eta\in\mathcal P(\Re)$, $\tilde{\Pi}: \mathcal P(\Re) \rightarrow \mathcal L$ is given by 
\[ \tilde{\Pi}(\eta) \in \arg\!\min\limits_{\kappa\in\mathcal L} d_K(\eta,\kappa), \quad \text{ where } \quad d_K(\eta,\kappa) = \sup\limits_{x\in\Re} \abs{F_\eta(x) - F_\kappa(x)}. \numberthis\label{eq:KolProj.mean}\]

At a high level, the exploration ensures that $\hat{\mu}(n)$ converges to $\mu$ and the sampling rule and continuity of $t^*$ together guarantee that the fraction of samples allocated to each arm converges to $t^*(\mu)$. \cite{garivier2016optimal} and \cite{juneja2019partition} follow a similar plug-in strategy for SPEF distributions, where their algorithms use the empirical means as proxies for the true means. We point out that this projection step discussed above is required for the theoretical guarantees of the proposed algorithm. However, in the numerical experiments, we do not project the empirical distributions to $\cal L$, even though it is not supported by the current theory.

Since solving the max-min lower bound problem can be computationally demanding, we solve it periodically after fixed, well-chosen $m > 1$ samples (which is allowed to be a function of \(\delta \)), where $m$ may then be optimized to minimize the overall computation effort. The specifics of the proposed algorithm $\bf AL_1$, are as follows:  
\begin{enumerate}
     \item
     Initialize by allocating \(m\) samples in a round-robin way to generate at least  $\floor{\frac{m}{K}}$ samples from each arm. Set  $l=1$ and let $l m$ denote the total number of samples generated.
     \item
     Check if the stopping condition (discussed in (\ref{eq:beta.mean})) is met. If not, compute the optimal proportions for the projected empirical distributions, i.e., compute \(t^*(\Pi(\hat{\mu}(lm)))\). 
     \item
     Compute starvation \(s_a\) for each arm as \[s_a:= \max\lrset{0, {((l+1)m)}^{1/2} -N_{a}(lm)}.\]
     \item if \(m \geq \sum_a s_a\), generate \(s_a\) samples from each arm $a$.
     Specifically, first, generate $s_1$ samples from arm 1,
     then $s_2$ samples from arm 2 and so on. Additionally, generate \( \max\lrset{m-\sum_a s_a,0}\) independent samples from \(t^*(\Pi(\hat{\mu}(lm)))\in\Sigma_K\) to 
     arrive at the number of remaining samples to be allocated to each arm, and sample each arm that many times. 
     
     \item Else, if $\sum_a s_a>m$,  generate \(\hat{s}_a\) samples from each  arm \(a\), where \(\lrset{\hat{s}_a}_{a=1}^K\) are a solution to the load balancing problem: \(\min \lrp{\max_a \lrset{ s_a-\hat{s}_a}} \text{ s.t. } s_a\geq \hat{s}_a\geq 0 \;\forall a\in[K],\) and \(\sum_a\hat{s}_a=m  . \)
     Again, first generate $\hat{s}_1$ samples from arm 1,
     then $\hat{s}_2$ samples from arm 2 and so on.
     \item  Increment $l$ by 1 and return to step 2.
\end{enumerate}

The following lemma proves that at the end $l^{th}$ interval of length \(m\), the sampling algorithm ensures at least (\(\sqrt{lm}-1 \)) samples to each arm, guaranteeing that no arm is starved. 

\begin{lemma}
     \label{lemma:MinNoOfSamples}
     Set \(m\geq (K+1)^2\).  Algorithm ${\bf AL_1}$  ensures that $N_a(lm) \geq {(l m)}^{1/2}-1$ for all $l\geq 1$.
\end{lemma}

Next, we show that our sampling algorithm also ensures that the fraction of times each arm is pulled is close to its optimal proportion \(t^*_a(\mu)\), i.e.,

\begin{lemma}
     \label{Lem:ASconvergence}
     For all \(a\in\lrset{1,\dots,K}\), $$ \frac{N_a(lm)}{lm} \xrightarrow{a.s.} t^*_a(\mu), \text{ as } l\rightarrow \infty.$$ 
\end{lemma}
We refer the reader to Sections~\ref{app:lem:MinNoOfSamples} and~\ref{app:lem:ASconvergence} for proofs of the above results. We now discuss the stopping rule for the Algorithm ${\bf AL_1}$. 

\subsection{Modified Empirical Likelihood Ratio Test: The Stopping Rule}\label{sec:ELT.mean}
At each step, generated data suggest a unique arm with the maximum mean (arbitrarily breaking ties, if any), say arm \(j\). Call this the null hypothesis, and its complement (arm \(j\) does not have maximum mean), the alternate hypothesis. For the stopping rule, we consider a modification of the generalized (empirical) likelihood ratio test (see \citet{chernoff1959sequential,owen2001empirical}). The algorithm stops if an appropriately defined likelihood ratio exceeds a specified threshold. Otherwise, it continues to generate samples according to the randomized tracking rule from the previous section. We refer the reader to Section~\ref{ch:EL} for a brief discussion on Empirical Likelihood for non-parametric distributions. 

Numerator in the likelihood ratio that we use is the likelihood under the most likely \(K\)-vector of distributions (not necessarily from $\cal M$) with arm \(j\) having the maximum mean. At time \(n\), since among all \(K\)-vectors of distributions in \(\lrp{\mathcal{P}(\Re)}^K\) with distribution \(j\) having a maximum mean, empirical distribution vector \(\hat{\mu}(n)=\lrp{\hat{\mu}_a(n): a\in [K]}\) maximizes the likelihood of the observed data (see Theorem~\ref{th:ELikelihood}), we take numerator to be the likelihood of observing data under the empirical distribution \(\hat{\mu}(n)\). 

The denominator equals the likelihood of the observed data under the most-likely bandit model (distributions of arms from $\mathcal L$) satisfying the alternate hypothesis, i.e., the likelihood under $\nu\in\mathcal{M}$ that maximizes the likelihood of given data under the alternative hypothesis. It follows from Remark~\ref{rem:discreteEL} that the most-likely model for the denominator likelihood will also be a discrete measure supported at least on all the data points. Thus, the likelihood ratio under consideration is well-defined. 

Recall that the empirical distributions may not belong to $\mathcal L$. Our likelihood ratio stopping statistic described above is also related to the lower bound max-min problem, as discussed next. If at stage $n$, arm $j$ is optimal in $\hat{\mu}(n)$, then consider the set $\operatorname{Alt}(\hat{\mu}(n))$ from (\ref{eq:altmu.mean}), i.e., 
\[\operatorname{Alt}(\hat{\mu}(n)) = \bigcup\limits_{i\ne j}\lrset{ \nu \in \mathcal M: ~ m(\nu_i) > m(\nu_j)  }.\]
The \(\log \lrp{\text{generalized empirical likelihood ratio}}\) described in the previous paragraph can be shown to equal (see Section~\ref{simpSR} for a proof)
\[ S_j(n) = \inf_{\mu' \in \operatorname{Alt}(\hat{\mu}(n))}~ \sum_{a=1}^{K}N_a(n)\KL(\hat{\mu}_a(n),\mu'_a).\] 
In addition, if $\hat{\mu}(n) \in \mathcal M$, then the above statistic simplifies as in Lemma~\ref{lem:maxminsimple.mean}. However, this need not be true always. Our stopping statistic at time $n$, when arm $j$ is optimal in $\hat{\mu}(n) $, is a modification of $S_j(n)$ and is defined to be 
\[ Z_j(n) :=  \min\limits_{a\ne j} ~ \inf\limits_{x\le y} ~ \lrset{N_j(n) \KLinfL(\hat{\mu}_j(n), x) + N_a(n) \KLinfU(\hat{\mu}_a(n), y) }. \numberthis \label{eq:zj}\]
It equals $S_j(n)$ when $\hat{\mu}(n)\in\mathcal M$ (Lemma~\ref{lem:maxminsimple.mean}) and arm $j$ is optimal in $\hat{\mu}(n)$.

\paragraph{Stopping and Recommendation Rules. } If at stage \(n\), arm $j$ is optimal in $\hat{\mu}(n)$, check 

\begin{equation} \label{eq:beta.mean}
Z_{j}(n) \ge \beta(n,\delta), \quad\text{ where }\quad \beta(n,\delta) = \log \frac{K-1}{\delta} + 4 \log (n + 1) + 2.
\end{equation}
The algorithm stops if \(Z_{j}(n)\geq \beta(n,\delta)\), announcing arm $j$ as the one with the maximum mean. If this condition is not met, the algorithm continues generating samples. 

A mild nuance in our analysis is that the stopping condition is checked only after intervals of \(m\), i.e., every time after $m$ samples are generated, and not at each step. Our analysis shows that as long as the size of each batch, denoted by $m$, is $o(\log\frac{1}{\delta})$, the additional samples required due to the delay in checking for the stopping condition are not significant compared to the total number of samples generated by the stopping time. Hence, this does not affect the asymptotic optimality of the proposed algorithm. We now present the theoretical guarantees of the proposed algorithm.

\subsection{Theoretical Guarantees}\label{sec:tg.meanBAI}
Recall that for simplicity of notation, we assume that $a^*(\mu) = 1$. For $\delta > 0$, let \(\tau_{\delta}\) denote the stopping time of the algorithm. It makes an error if at time \(\tau_{\delta}\), $m(\hat{\mu}_j(\tau_{\delta})) > \max_{i\ne j} m(\hat{\mu}_i(\tau_{\delta}))$, for some \(j\ne 1\). Let  \(\mathcal{E}\) denote the error event.  

\begin{theorem}
     \label{bigTh2}
     The algorithm ${\bf AL_1}$, with  \(\beta(n,\delta)\) as in (\ref{eq:beta.mean}), and {\(m=o(\log(1/\delta))\)}, is \(\delta\)-correct, i.e.,
     \begin{equation} \mathbb{P}(\mathcal{E}) \leq \delta. \label{DeltaCorrectub}
     \end{equation}
     Further, 
     \begin{equation}
     \label{SampleComplexityub}
     \limsup\limits_{\delta\rightarrow 0}\frac{\mathbb{E}_{\mu}(\tau_{\delta})}{\log\lrp{1/\delta}}\leq \frac{1}{V(\mu)}. 
     \end{equation}
\end{theorem}
We first analyze \(\delta\)-correctness of algorithm \(\bf{AL_1}\) and present the analysis of the sample complexity towards the end of the current section. 

\subsubsection*{Proof of $\delta$-correctness in Theorem~\ref{bigTh2}}
	The proof of \(\delta \)-correctness relies on the concentration of scaled sums of the empirical versions of the two $\KL$ functionals, i.e., \(\KLinfL(\hat{\kappa}(n),m(\kappa))\) and $\KLinfU(\hat{\kappa}(n), m(\kappa))$, where for \(\kappa\in\mathcal{L}\), \(\hat{\kappa}(n)\) denotes the empirical distribution corresponding to \(n\) samples from \(\kappa\). The concentration result, in turn, relies on the dual representations from Theorem~\ref{th:klinfDual_mean}.

     Recall that the algorithm makes an error if, at the stopping time \(\tau_{\delta}\), the empirically best arm is not the arm $1$. Moreover, we only check for the stopping condition at times $lm$, for $l\in\mathbb{N}$ and a fixed $m\ge 1$. Let $\mathcal E_j(n)$ be the event that at time $n$, the empirically-best arm is arm $j$, i.e.,
     \[ \mathcal E_j(n) = \lrset{ \hat{\mu}(n) : m(\hat{\mu}_j(n)) > \max\limits_{a\ne j} ~ m(\hat{\mu}_a(n)) }.\]

    Then, if an error occurs, there exists $l$ such that at time $lm$, an arm $j\ne 1$ has the maximum empirical mean, and the stopping condition is met at this time. Thus, 
    \[ \mathcal E \subset \lrset{\exists l\in\mathbb{N} :  \bigcup\limits_{j\ne 1}\lrset{ Z_j(lm) \ge \beta(lm,\delta),~ \mathcal E_j(lm)}}. \]

    Recall from~\eqref{eq:zj} that $Z_j(\cdot)$ is a minimum of certain arm-specific statistics over arms $a\ne j$. Thus, it is upper bounded by the statistic for arm $1$, i.e., 
    $$Z_j(lm) \le \inf\limits_{x\le y} \lrset{N_j(lm)\KLinfL(\hat{\mu}_j(lm), x) + N_1(lm)\KLinfU(\hat{\mu}_1(lm), y)}. $$
    Next, observe that $x=m(\mu_j)$ and $y=m(\mu_1)$ are feasible. Hence, 
    \[ Z_j(lm) \le {N_j(lm)\KLinfL(\hat{\mu}_j(lm), m(\mu_j)) + N_1(lm)\KLinfU(\hat{\mu}_1(lm), m(\mu_1))}.  \]
    Thus, the error event is further contained in 
    \[ \lrset{\!\exists n\!\in\! \mathbb{N}: ~ \bigcup\limits_{j\ne 1}\! \lrset{ N_j(n)\KLinfL(\hat{\mu}_j(n), m(\mu_j)) \! + \! N_1(n)\KLinfU(\hat{\mu}_1(n), m(\mu_1)) \ge \beta(n,\delta), ~ \mathcal E_j(n) }\!}. \]
	Using union bound,
	\begin{equation} \label{eq:deltasummand}
		\mathbb{P}(\mathcal E) \le \sum\limits_{j\ne 1}\mathbb{P}\lrp{\exists n\in \mathbb{N}: { N_j(n)\KLinfL(\hat{\mu}_j(n), m(\mu_j))\!+\! N_1(n)\KLinfU(\hat{\mu}_1(n), m(\mu_1)) \ge \beta(n,\delta) }}. 
	\end{equation}

	Using the concentration result from Proposition~\ref{prop:DeviationsMean} (stated below) and $\beta(n,\delta)$ from (\ref{eq:beta.mean}), we get that each summand in the above expression is at most $\frac{\delta}{K-1}$, proving the $\delta$-correctness of the algorithm. \BlackBox

	We remark that in the earlier version of this work \citep{agrawal2020optimal}, we developed an exponential tail bound for the summands in~\eqref{eq:deltasummand} using an $\epsilon$-net-based argument. Building on the concentration results in our later works \citep{agrawal2021regret,agrawal2021optimal}, here we provide a much simpler and elegant martingale-proof for bounding these terms.

	\begin{proposition} \label{prop:DeviationsMean}
    For \( i\in [K], j\in[K] \), \(i\ne j\), \(h(n) = 4\log(n+1)+2\), and \( x \ge 0 \),
    \[ \mathbb{P}\lrp{\exists n:  N_i(n) \KLinfU(\hat{\mu}_i(n), m(\mu_i)) + N_j(n)\KLinfL(\hat{\mu}_j(n), m(\mu_j)) -h(n) \ge x  } \le e^{-x}.\]
	\end{proposition}
\noindent\emph{Proof sketch: } To get the said concentration result, we construct a non-negative super-martingale dominating the exponential of the l.h.s. above. Ville's inequality \citep{ville1939etude} then gives the desired result. Below, we sketch the steps for arriving at the appropriate super-martingale. 

\paragraph{Constructing the mixture martingale.} From the dual formulation for $\KLinfU$ (Theorem~\ref{th:klinfDual_mean}), 
\[  N_i(n)\!\KLinfU(\hat{\mu}_i(n),m(\mu_i))\! =\! \max\limits_{(\lambda_1, \lambda_2)\! \in\! S^U(x,B)} \sum\limits_{l=1}^{N_i(n)} \!  \log (1\!-\!\lambda_1(X_l-m(\mu_i)) \!-\! \lambda_2(B\!-\!\abs{X_l}^{1+\epsilon})),\! \numberthis\label{eq:d1} \]
where $\lambda_1$ and $\lambda_2$ are dual variables, and the set $S^U$ is defined in~\eqref{eq:SU}. 

Next, recall that $\mu_i$ belongs to $\mathcal L$ and $X_l$ are i.i.d. Thus, for fixed $ \lambda_1 $ and $\lambda_2$, the summands in the r.h.s. of~\eqref{eq:d1} are i.i.d. and their exponentials are non-negative with mean bounded from above by $1$. Thus, for fixed dual variables, the exponential of r.h.s. in~\eqref{eq:d1} is a non-negative super-martingale. However, none of these super-martingales dominate the exponential of $N_i(n)\KLinfL(\hat{\mu}_i(n), m(\mu_i))$. Lemma~\ref{lem:exp-concave.BAIm} below shows that a uniform mixture of these super-martingales mixed over the possible dual variables dominates the l.h.s. of~\eqref{eq:d1} after adjusting it by a mild cost of \( \frac{h(n)}{2}\). 

Following symmetric arguments, we construct another (mixture) super-martingale that dominates the exponential of $N_j(n)\KLinfL(\hat{\mu}_j(n), m(\mu_j)) - \frac{h(n)}{2}$. We then show that the product of these two (mixture) super-martingales that we constructed, is a non-negative super-martingale dominating the exponential of the l.h.s. in the original expression. 

We refer the reader to Section~\ref{App_DeviationsMean.BAI} for complete proof of this result. 

\begin{lemma}[{\citet[Lemma F.1]{agrawal2021optimal}}]\label{lem:exp-concave.BAIm}
    Let $\Lambda \subseteq \mathbb R^d$ be a compact and convex subset, and let $q$ be the uniform distribution on $\Lambda$. Let $g_i: \Lambda \to \mathbb R$ be any series of exp-concave functions. Then
    \[
    \max_{\bm{\lambda} \in \Lambda}~
    \sum_{i=1}^T g_i(\bm{\lambda})
    ~\le~
    \log \E{\bm{\lambda}\sim q}{
        e^{\sum_{i=1}^T g_i(\bm{\lambda})}
    }
    + d\log(T+1)+1.
    \]
\end{lemma}

     \subsubsection*{Proof Sketch for Sample Complexity in Theorem~\ref{bigTh2}} 
     To see that the expected number of samples required by \(\bf{AL_1}\) matches lower bound as \(\delta\rightarrow 0\), i.e., (\ref{SampleComplexityub}) holds, we first give a high-level  proof idea.
     Recall from (\ref{eq:beta.mean}) that $\tau_\delta$ equals 
     \begin{align*}
     \inf\lrset{lm : \max_{j\in[K]}\min_{b\ne j} \inf\limits_{x\le y} \lrset {\frac{N_j(lm)}{lm}\KLinfL(\hat{\mu}_j(lm),x)+\frac{N_b(lm)}{lm}\KLinfU(\hat{\mu}_b(lm),y)} \geq \frac{\beta(lm,\delta)}{lm}},
     \end{align*}
     and satisfies
     \begin{equation*}
     \max_{j\in[K]}\min_{b\ne j} \inf\limits_{x\le y} \lrset {\frac{N_j(\tau_{\delta})}{\tau_{\delta}}\KLinfL(\hat{\mu}_j(\tau_{\delta}),x)+\frac{N_b(\tau_{\delta})}{\tau_{\delta}}\KLinfU(\hat{\mu}_b(\tau_{\delta}),y)} \approx \frac{\beta\lrp{\tau_{\delta},\delta}}{\tau_{\delta}} .     
     \end{equation*}

     Furthermore, for sufficiently large \(l\), consider the event that the empirical distributions at time $lm$ are close to the true distributions in the Kolmogorov metric (defined in~\eqref{eq:KolProj.mean}). With high probability, \(\forall a, \;\; \hat{\mu}_{a}(lm)  \approx \mu_a \), and from Lemma~\ref{Lem:ASconvergence}, \(N_a(lm)/{lm} \approx t^*_a(\mu) \). When this is true, arm \(1\) is the best arm, and \(\tau_{\delta}\) satisfies     
     \begin{equation*}
     \tau_{\delta}\approx \beta\lrp{\tau_{\delta},\delta} \lrp{\min\limits_{b\ne 1} \inf_{x\le y} \lrset{t^*_1(\mu) \KLinfL(\mu_1,x)+t^*_{b}(\mu)\KLinfU(\mu_b,y)}}^{-1} = \frac{\beta\lrp{\tau_{\delta},\delta}}{V(\mu)}.
     \end{equation*}

     With $\beta(\cdot,\cdot)$ from (\ref{eq:beta.mean}), \(\tau_{\delta} \) that satisfies the above equality is given by 
     \begin{equation}\label{approxTauDelta}
     \tau_{\delta} = \frac{1}{V(\mu)}\log\frac{K-1}{\delta}  + \frac{4}{V(\mu)} \log \lrp{1+\log\frac{K-1}{\delta} } + o\lrp{\log\frac{1}{\delta}}. \end{equation}
     Dividing both sides of (\ref{approxTauDelta}) by \(\log(1/\delta) \), we get \(\frac{\tau_{\delta}}{\log\lrp{1/\delta}}\approx \frac{1}{V(\mu)}\), for sufficiently small \(\delta \). 

    Complement of this high-probability event (i.e., the set where $\hat{\mu}(lm)$ is far from $\mu$) contributes only lower order terms (with respect to \(\log(1/\delta) \)) to \(\mathbb{E}_{\mu}(\tau_{\delta})\) since the probability of this event is exponentially small in the number of samples $lm$. Combining these, we get an upper bound on \(\mathbb{E}_{\mu}(\tau_{\delta})\) that asymptotically (as \(\delta\rightarrow 0\)) matches the lower bound in (\ref{eq:lb3}). 

    Rigorous proof of the sample complexity result in Theorem~\ref{bigTh2}, i.e., proof for (\ref{SampleComplexityub}), is given in Section~\ref{App_SampleComplexity}. Our proof builds upon that in \cite{garivier2016optimal}, where the authors consider a restricted SPEF, while we allow arm distributions to belong to a more general class \(\mathcal{L} \). Our proof differs in that we work in the space of probability measures instead of in Euclidean space, which leads to additional nuances. To work in the space of probability measures, we use the L\'evy metric to define the continuity of functions and convergence of sequences in this space. Moreover, we work with the projections of the empirical distributions and check for the stopping condition only once in \(m\) samples instead of doing so in every sample. Our proof allows for these flexibilities. 

	In the next section, we show that for the optimal batch size in the current framework, the computational cost reduces from quadratic in the number of samples to linear. However, the sample complexity of the resulting algorithm is off from the lower bound by a multiplicative constant. For sub-optimal batch sizes, the computational effort of the algorithm is still quadratic in the total number of samples but with smaller constants. This results in an improvement in the computation time in practice. Moreover, the resulting algorithm exactly matches the lower bound. 

     \subsection{Discussion on Optimizing Batch Sizes}\label{sec:optbatch.mean}
     We now discuss the choice of the batch size $m$ to minimize the overall cost of the experiment. Suppose that the average cost of generating any sample is given by $c_1$. This cost may be significant when sampling is costly, for instance, if samples correspond to the output of a massive simulation model or a result of clinical trials. It may be small, e.g.,  in an online recommendation system. 

     The cost of solving the max-min problem may be measured by the computational effort involved, which involves determining the \(\KL\) projection functions for the empirical distributions of each arm. Empirically, we see that the computational cost of estimating the \(\KLinf\)s increases with the number of samples in the empirical distribution in the setting considered in this work. It can also be seen from the dual formulations (Theorem~\ref{th:klinfDual_mean}), where the number of terms in the objective function increases by $1$ with each sample. See also \citep{cappe2013kullback,honda2015SemiBounded} for a discussion on the computational cost of $\KLinf$s. Hence, the cost for computing the optimal weights at time $n$ is also linear in $n$. Hence, the overall computational cost of solving the max-min problem after $n$ samples is modeled well as $c_{21}+ c_{22}n$, where $c_{21}$ and $c_{22}$ are constants. Observe that when arm distributions belong to a SPEF, the cost of solving the max-min problem is constant, independent of number of samples. Hence, it can be modelled by setting $c_{22} = 0$ in the current discussion.
      
     To approximate the optimal batch size, we need to approximate the sample complexity. To this end, let 
     $$\tilde{\beta}(\delta) :=   \log \frac{K-1}{\delta} + 4 \log \lrp{1 + \log \frac{K-1}{\delta}}.$$ 

     For small values of \(\delta\),  the sample complexity of \(\bf AL_1\) is bounded as below (see (\ref{EqForM*}) in Section~\ref{App_SampleComplexity}),
     \begin{equation}\label{eq:SampleComplexDependenceOnM}
     \mathbb{E}_{\mu}\lrp{\tau_{\delta}} \leq \frac{\tilde{\beta}(\delta)}{V(\mu)} + m 
     + \mbox{ lower order terms},\end{equation}
     where $m$ denotes the batch size. Equation~\eqref{eq:SampleComplexDependenceOnM} remains valid if we use $\log\frac{1}{\delta}$ in place of $\tilde{\beta}(\delta)$. However, for small values of $\delta$, these may differ significantly. Empirically we find  that $\frac{\log\frac{1}{\delta}}{V(\mu)}+m$ substantially underestimates $\mathbb{E}_{\mu}\lrp{\tau_{\delta}}$, while  $\frac{\tilde{\beta}(\delta)}{V(\mu)}+m$ is much closer. 

     Using $\frac{\tilde{\beta}(\delta)}{V(\mu)} + m$ as a reasonable proxy for $\mathbb{E}_{\mu}\lrp{\tau_{\delta}}$, the total cost $\mathcal{C}$ of \(\bf AL_1 \) is approximated as 
     \[\mathcal{C} = \lrp{c_1 + \frac{c_{21}}{m}}\frac{\tilde{\beta}(\delta)}{V(\mu)} + \frac{ c_{22}\tilde{\beta}^2(\delta)}{2mV^2(\mu)} + \lrp{c_1 + \frac{c_{22}}{2}}m.\]
     
     For a constant $m$ (independent of $\delta$), $\tilde{\beta}(\delta) = O(\log\frac{1}{\delta})$, implying $\mathcal{C} = O(\log^2{\frac{1}{\delta}}).$ For \(m= O(\log\frac{1}{\delta})\), it is \(O(\log(\frac{1}{\delta}))\). Optimizing over \(m\) to minimize \(\mathcal{C}\), we get     
     \begin{equation}\label{optBatch}
     m^* = \lrp{c_{21} \frac{\tilde{\beta}(\delta)}{V(\mu)} + 0.5c_{22}{\frac{\tilde{\beta}^2(\delta)}{V^2(\mu)}}}^{0.5}\frac{1}{\lrp{c_1+0.5\;c_{22}}^{0.5}},
     \end{equation}
     i.e., optimizer \(m^* = O(\log\frac{1}{\delta})\). Notice that even though \(m=m^*\) minimizes \(\mathcal{C}\), (\ref{eq:SampleComplexDependenceOnM}) suggests that with this choice of \(m\), the ratio of sample complexity of \(\bf AL_1\) to the max-min lower bound no longer converges to 1  as $\delta \rightarrow 0$. It can, however, be seen that the $\delta$-correct property still holds for \(\bf AL_1\) even for $m = O(\log\frac{1}{\delta})$. 

     If, however, the \(\KL \) functionals could be estimated using computational effort that is independent of the size of the empirical distribution, that is, if $c_{22}=0$,  then  
     \[m^* =O\left(\sqrt{\log\frac{1}{\delta}}\right), \] 
     and \(\bf AL_1 \) is asymptotically optimal. Notably, this suggests that this is the optimal batch-size for SPEF setting, and the resulting batched algorithm remains asymptotically optimal. One way to achieve such a guarantee in the current setting is to approximate the empirical distribution by a fixed size distribution (e.g., by bucketing the generated samples into finitely many bins). Doing this may substantially reduce the computation time and is worth exploring further. 

     \section{Numerical Results}\label{numerics}
     In this section, we present the experimental results for the algorithm \(\bf AL_1\) on a small example. We consider a \(4\)-armed bandit, with each arm having Pareto distribution with two parameters (\(\alpha,\beta\)), supported on \([\beta,\infty)\) and has PDF  
     $$f_{\alpha,\beta}(x) = \frac{\alpha\beta^{\alpha}}{x^{\alpha+1}}.$$ 
     The four arms have parameters set to \((4,1.875),(4,1.5),(4,1.25),(4,0.75)\). We choose \(\mathcal{L} \) with \(\epsilon = 1\) and \(B=9\), i.e.,
     \[\mathcal L  = \lrset{\eta\in\mathcal P(\Re): ~ \E{\eta}{\abs{X}^2} \le 9}. \]

     \begin{figure}
     \centering
     \includegraphics[width=0.8\linewidth]{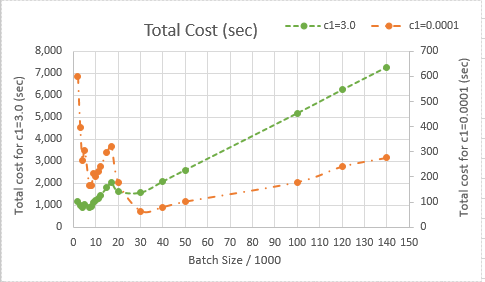}
     \caption{Total computational cost as a function of \((a)\) batch size and \((b)\) \(c_1\), the cost of generating a sample.}
     \label{fig:VBS}
     \end{figure}

     We compare the total cost, that is, the cost of generating samples (\(c_1\) per sample) added to the total cost of computing the max-min solutions after every batch, as a function of the batch size. The horizontal axis in Figure~\ref{fig:VBS} denotes the batch size in thousands, and both the vertical axes correspond to total cost (normalized and measured in seconds).

     It shows that the computational cost, and hence, the total cost, initially comes down with an increase in the  batch size since the number of batches in the algorithm reduces. However, a further increase in batch size increases the total cost due to the delay in stopping by the last batch (the algorithm would have stopped earlier if the batch size was smaller). 
	
	Interestingly, as the batch size further increases, we see that the cost decreases and then increases again. To understand this, consider the case where the algorithm typically stops after two batches are complete. Now, at some stage, if the batch size increases sufficiently, the algorithm typically stops in a single batch, leading to substantial cost reduction. If the batch size is further increased, the delay in stopping increases, again leading to an increase in cost. In this case, the fact that the support size of the empirical distributions increases also increases the cost with an increase in the batch size.

     \begin{figure}
          \centering
          \includegraphics[width=0.8\linewidth]{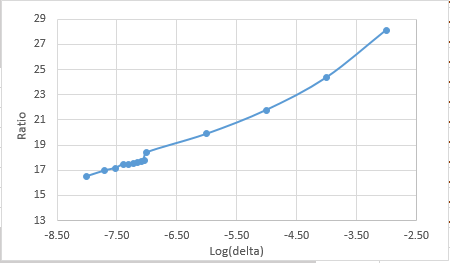}
          \caption{Ratio of average number of samples needed by the algorithm to stop and the lower bound as a function of \(\delta \).}
          \label{fig:LBUB}
     \end{figure}

    Figure~\ref{fig:LBUB} plots the ratio of an average number of samples needed by \(\bf AL_1 \) to stop, and the lower bound on this quantity, as a function of \(\delta\). As can be seen from the figure, as \(\delta \) reduces from \(0.001\) to \(10^{-8}\), this ratio decreases from 30 to 16. 

\section{Discussion on MAB with Bounded-Support Distributions}\label{sec:bdd.meanBAI}
     In this section, we discuss an optimal algorithm for MAB with distributions having bounded support, say in $[0,1]$. We show that algorithm \(\bf  AL_1\) when specialized to this class, is an asymptotically-optimal \(\delta\)-correct algorithm for this class of MAB problems. Since all the proofs follow as in the more general setting discussed in detail earlier, we only give the statements of the results. 

     Let 
    \[{\mathcal{H}}=\lrset{\eta\in\mathcal{P}(\Re): \Supp(\eta)\subset [0,1]}.\]
	Observe that $\cal H$ is a uniformly integrable family of measures \citep{williams1991probability}. Let \(\mathcal{M}_{\mathcal H} = \mathcal H^K \) denote the collection of all vectors of \(K\) distributions, each belonging to \(\cal H \). As earlier, we endow \(\cal H \) with the L\'evy metric (or equivalently, the topology of weak convergence). Furthermore, \({\mathcal{H}} \) can be shown to be a compact set with respect to this metric (see \citet[Section 5]{billingsley2013convergence}). 

     Recall that the $\KL$-projection functions defined in~\eqref{eq:KLinf} implicitly depend on the class of distribution under consideration. For \(\eta\in \mathcal H\) and \(x \in [0,1] \), define analogous $\KL$ projection functions, denoted by \( \KLinfUC{\mathcal H}(\eta,x) \) and $\KLinfLC{\mathcal H}(\eta,x)$, 
     \[ \KLinfUC{\mathcal H}(\eta,x) := \min\limits_{\substack{\kappa\in\mathcal H, m(\kappa)\ge x }} \KL(\eta,\kappa)\quad \text{and}\quad \KLinfLC{\mathcal H}(\eta,x) = \min\limits_{\substack{\kappa\in\mathcal H, m(\kappa)\le x}} \KL(\eta,\kappa). \numberthis\label{eq:bddklinf} \]
     \citet{HondaBounded10} develop a dual representation for $\KLinfUC{\mathcal H}$ and establish various properties of this function. Since the two functions in~\eqref{eq:bddklinf} are symmetric, similar representations and properties can be proven for $\KLinfLC{\mathcal H}$.   
     
    Given $\mu \in \cal M_H$, for simplicity of notation, we assume that arm $1$ is the unique arm with the maximum mean in $\mu$. Let $\tau_\delta$ denote the stopping time of a $\delta$-correct algorithm. The lower bound on the average number of samples required by any $\delta$-correct algorithm can be shown to be 
     \[ \E{\mu}{\tau_\delta} \ge \frac{\log \frac{1}{\delta}}{V_{\mathcal H}(\mu)}, ~~~\text{where} ~~~ V_{\mathcal H}(\mu) = \sup\limits_{w\in \Sigma_K} ~ \min\limits_{j\ne 1}\inf\limits_{x\le y} \lrset{w_1 \KLinfLC{\mathcal H}(\mu_1, x) + w_j \KLinfUC{\mathcal H}(\mu_j, y)}. \]
    This follows from the proof of~\eqref{eq:Vmu.mean}, replacing $\cal L$ by $\cal H$.

    Next, consider the algorithm that replaces $\cal L$ by $\cal H$, and $\KLinfU$ and $\KLinfL$ in $\bf AL_1$ with $\KLinfUC{\mathcal H}$ and $\KLinfLC{\mathcal H}$, respectively. Call it $\bf AL_B$. Unlike in the more general setting considered earlier, observe that empirical distribution always belongs to $\cal M_H$. Hence, the projection of empirical distributions on $\cal H$ are the empirical distributions themselves, i.e., at time $t$, 
     \[ {\Pi}\lrp{\hat{\mu}(t)} = \hat{\mu}(t),\]
     where $\Pi$ is defined in (\ref{eq:KolProj.mean}) with $\cal L$ replaced by $\cal H$. Set 
     \[\beta(n,\delta) = \log\frac{K-1}{\delta} + 2\log(n+1) + 2. \numberthis\label{eq:BetaBdd}\]

     \begin{theorem}
     Algorithm  $\bf AL_B$ with $\beta$ in (\ref{eq:BetaBdd}) and $m = o(\log\frac{1}{\delta})$, is $\delta$-correct and satisfies 
     \[ \limsup\limits_{\delta \rightarrow 0} \frac{\E{\mu}{\tau_\delta}}{\log \frac{1}{\delta}} \le \frac{1}{V_{\mathcal H}(\mu)}. \]
     \end{theorem}

     Proof of the above theorem is similar to that of Theorem~\ref{bigTh2}. The following proposition will be useful in proving the $\delta$-correctness.

     \begin{proposition} \label{prop:DeviationsMean.bdd}
     For \( i\in [K], j\in[K] \), \(i\ne j\), \(h(n) = 2\log(n+1)+2\), and \( x \ge 0 \),
     \[ \mathbb{P}\lrp{\exists n:  N_i(n) \KLinfUC{\mathcal H}(\hat{\mu}_i(n), m(\mu_i)) + N_j(n)\KLinfLC{\mathcal H}(\hat{\mu}_j(n), m(\mu_j)) -h(n) \ge x  } \le e^{-x}.\]
     \end{proposition}

     The proof of this proposition follows along the lines of that in Proposition~\ref{prop:DeviationsMean}. The dual formulations for $\KLinfUC{\mathcal H}$ and $\KLinfLC{\mathcal H}$ suggest martingales for fixed values of the dual variables. As earlier, we construct a mixture of these martingales that dominate the exponentials of the l.h.s. to give the desired result. We point out that \citet{honda2015SemiBounded} also develop a similar concentration result for the empirical $\KLinfUC{\mathcal H}$ using an $\epsilon$-net argument. However, their concentration inequality is only valid for a fixed $n$. A direct application of their concentration inequality may result in a slightly higher threshold $\beta$, affecting the performance of the algorithm for reasonable values of $\delta$.

     \section{Conclusions} 
     In this work, we developed an asymptotically optimal algorithm for the fixed-confidence BAI problem in two settings: (1) arms with heavy-tailed distributions and (2) arms with bounded-support distributions. We saw that the two $\KL$ projection functions ($\KLinfU$ and $\KLinfL$) play a crucial role in the implementation and analysis of the algorithm. We developed alternative representations for these and proved several properties that are of independent interest. We also discussed the computational cost of the proposed algorithms and developed a batched version to reduce this cost while ensuring its good statistical performance. 

     A natural extension is the $\epsilon$-BAI problem in the PAC framework. An asymptotic lower bound (as $\delta\rightarrow 0$) for this problem is known (see, for example, \cite{NIPS2019_MultipleCorrectAns,garivier2021nonasymptotic}). Optimal algorithms for this problem have been proposed for SPEF arms in the MAB setting (\cite{NIPS2019_MultipleCorrectAns}) and recently for structured MAB (linear) in \cite{jourdan2022choosing}. While extending these to a heavy-tailed setting may be possible using the techniques developed in the current work, the algorithms relying on computing the solutions of the max-max-min problem for this general class of distributions may not be tractable. Computationally tractable algorithms for $\epsilon$-BAI might also give efficient algorithms for the BAI problem by choosing $\epsilon$ as a function of the number of samples ($\epsilon_t$) such that $\epsilon_t\rightarrow 0$ as $t\rightarrow\infty$. We believe that this is an exciting direction for future research. 

     \acks{We acknowledge the support of Department of Atomic Energy, Government of India, under project no. 12-R\&D-TFR-5.01-0500. We also thank ICTS, TIFR to allow the authors to collaborate on this work during the Applied Probability Program at ICTS. We are grateful to Wouter M. Koolen for providing his insightful feedback on our initial results.}

\bibliography{BibTex.bib}

\newpage
\appendix
\section{Mathematical Background}\label{app:bg}
	In this appendix, we present some of the classical mathematical results from the literature we use in our analysis. In Section~\ref{sec:WC}, we state the definitions of the relevant topologies on the set of probability measures. We also review some concepts from the theory of weak convergence of probability measures in this section. To prove the continuity of the optimal values of optimization problems, as well as of the optimizers, we use the classical Berge's Theorem. We give the necessary background on the continuity-like properties of the set-valued maps and state Berge's Theorem in Section~\ref{sec:Berge}. The stopping rule of our algorithm relies on the likelihood ratio test. However, since we do not make any parametric assumptions in our work, it is crucial to define likelihood ratios in the non-parametric setting. In Section~\ref{ch:EL}, we introduce the empirical (non-parametric) likelihood and a few results from this area.

    \subsection{Topologies in the Space of Probability Measures} \label{sec:WC}   
    We endow the space of probability measures with the topology of weak convergence. In this section, we define this notion of convergence of a sequence of probability measures and introduce different topologies in the space of probability measures that become handy in our analysis. We refer the reader to \cite{billingsley2013convergence} for alternative characterizations of weak convergence and its properties.

    Recall that $\Re$ denotes the set of real numbers, and  $\mathcal P(\Re)$ denotes the collection of all probability measures with support in $\Re$. For $\eta\in\mathcal P(\Re)$, let $F_\eta(y) := \eta((-\infty, y])$ denote the CDF function for $\eta$. Let $\phi:\Re\rightarrow \Re$ be a bounded and continuous function, ${\gamma} > 0 $, and $x\in\Re$. \emph{Weak topology} on $\mathcal P(\Re)$ is the topology generated by the base sets of the form 
    $$ \mathcal U(\phi, x, {\gamma}) := \lrset{ \eta\in \mathcal P(\Re): \abs{\int\limits_{\Re} \phi(y) d\eta(y) - x }  \le {\gamma} }.  $$
    Consider a sequence of probability measures $\{\kappa_n\}_{n\in\mathbb{N}}$ and $\kappa \in \mathcal P(\Re)$. Weak convergence of $\kappa_n$ to $\kappa$, denoted as $\kappa_n \xRightarrow{D} \kappa$, is convergence in this topology \cite[see][Section D.2]{dembo2010large}. 

    The convergence in the weak topology is equivalent to that in {\bf L\'evy metric} on \( \mathcal{P}(\Re) \), (denoted by \(d_L\)). We define the L\'evy metric below (see, \citet[Theorem 6.8]{billingsley2013convergence}, \citet[Theorem D.8]{dembo2010large}). 

    \begin{definition}[L\'evy metric]
    \emph{For \(\eta , \kappa \in \mathcal{P}(\Re) \), 
        $$d_L(\eta,\kappa):= \inf\lrset{\gamma > 0 : F_\eta(x-\gamma)-\gamma \leq F_\kappa(x) \leq F_\eta(x+\gamma) + \gamma, ~ \forall x\in\Re }.$$}
    \end{definition}
    It can be shown that the metric space \( \lrp{\mathcal{P}(\Re), d_L} \) is complete and separable. Another popular metric on $\mathcal P(\Re)$ is the {\bf Kolmogorov metric}, defined next. 
    \begin{definition}[Kolmogorov metric ]
    \emph{For probability measures $\eta$ and $\kappa$,   
        $$ d_K(\eta,\kappa) := \sup\limits_{x\in\Re} \abs{F_\eta(x) - F_\kappa(x)}. $$}
    \end{definition}
    From the definitions above, it follows that $d_L(\eta,\kappa) \le d_K(\eta,\kappa)$ \cite[see][]{gibbs2002choosing}. We refer the reader to \cite{billingsley2013convergence} for a detailed exposition on weak convergence of probability measures.     

    \subsection{Berge's Theorem}\label{sec:Berge}
    Proving the continuity of the optimal values and optimizers of optimization problems is a crucial component of our analysis. We now review the notion of continuity for set-valued maps and then present the classical Berge's Theorem (see, \citet[Chapter 6]{berge1997topological} and \cite[Chapter 9]{SundaramOpt1996}).

    Let $l\in\mathbb{N}$, $d \in \mathbb{N}$, $S \subset \Re^l$, and $T\subset \Re^d$. Let $\Gamma: S\rightarrow T$ be a \emph{set-valued function} which, for each $s\in S$, associates a set $\Gamma(s) \subset T$.

    \begin{definition}[Upper hemicontinuity] 
    \emph{A set-valued function $\Gamma: S \rightarrow T$ is upper hemicontinuous at $s \in S$ if for any open neighborhood $V$ of $\Gamma(s)$ there exists a neighborhood $U$ of $s$ such that for all $x\in U$, $\Gamma(x)$ is a subset of $V$.}
    \end{definition}

    \begin{definition}[Lower hemicontinuity] 
    \emph{A set-valued function $\Gamma: S \rightarrow T$ is lower hemicontinuous at $s \in S$ if for any open set $V$ intersecting $\Gamma(s)$ there exists a neighborhood $U$ of $s$ such that $\Gamma(x)$ intersects $V$ for all $x\in U$. }
    \end{definition}

    \begin{definition}[Continuity] 
    \emph{A set-valued function $\Gamma: S \rightarrow T$ is said to be continuous at $s \in S$ if it is both lower and upper hemicontinuous at $s$.}
    \end{definition}
    \vspace{0.5mm}

    We say that $\Gamma$ is upper(lower)-hemicontinuous in $S$ if it is upper(lower)-hemicontinuous at each point in $S$. If $\Gamma$ is both upper and lower hemicontinuous in $S$, then it is continuous in $S$. Notice that if $\Gamma$ is indeed a function, i.e., for each $s\in S$, $\Gamma(s)$ is a singleton, then all the definitions above coincide with the usual definition of continuity of a function at a point. In addition to the above, a set-valued function $\Gamma$ is said to be \emph{closed-valued} (\emph{compact-valued or convex-valued}) at $s\in S$ if the set $\Gamma(s)\subset T$ is closed (compact or convex, respectively).
    \begin{theorem}[Berge's]\label{th:Berge}
    {Let $\Phi: S \times T \rightarrow \Re$ be a continuous function, and let $\Gamma: S \rightarrow T$ be a compact and continuous, set-valued function (both, lower and upper hemicontinuous) such that for all $s \in S$, $\Gamma(s) \ne \emptyset$. Define 
    $$M(s):= \max\lrset{\Phi(s, t) | t\in \Gamma(s)} .$$
    Then, $M$ is continuous in $S$ and the mapping defined 
    $$\Phi_s := \lrset{t | t \in \Gamma(s), \Phi(s,t) = M(s) } $$ 
    is compact and upper hemicontinuous mapping of $S$ into $T$. }
    \end{theorem}

    \begin{remark}\label{th:partialberge}
    	\emph{If instead, we only have upper (lower) semicontinuity of $\Phi$ and upper (lower) hemicontinuity of $\Gamma$, then we get upper (lower) semicontinuity of $M$ (see, \citet[Chapter 6, Page 115]{berge1997topological}).} 
    \end{remark}
    
    We refer the reader to \citet[Chapter 6]{berge1997topological} and \citet[Chapter 9]{SundaramOpt1996} for alternative characterizations of upper and lower hemicontinuity and proofs of the Theorem~\ref{th:Berge} and Remark~\ref{th:partialberge}.

    \subsection{Empirical Likelihood}\label{ch:EL} 
    In this appendix, we describe a non-parametric method of inference that allows the use of likelihoods without having to assume that the data-generating distribution belongs to a parametric family of distributions. The majority of the content in this appendix is borrowed from \cite{owen2001empirical} and is presented here for completeness.

	\paragraph{Notation. } For a probability measure, $\eta$, \emph{CDF} is the function $F:\Re\rightarrow [0,1]$, where for $x\in\Re$,
    \[F_\eta(x) := \mathbb{P}_{X\sim \eta}\lrp{X \le x}.\] 
    Here, $X\sim\eta$ denotes that $X$ is a random variable which is distributed as $\eta$. 
    Let $F_\eta(x^-) := \mathbb{P}_{X\sim \eta}\lrp{X < x}$, so that 
    $$\mathbb{P}_{X \sim \eta}(X = x) = F_\eta(x) - F_\eta(x^-).$$
    \begin{definition}
    \emph{Suppose we are given observations $X_1, \dots, X_n$. The empirical CDF (ECDF) of these is the function 
    \[ F_n(x) :=  \frac{1}{n}\sum\limits_{i=1}^n \mathbb{1}_{X_i \le x}, \quad \forall x\in\Re, \]
    where, $\mathbb{1}_{X_i \le x}$ is $1$, if $X_i \le x$, otherwise $0$.}
    \end{definition}

    \begin{definition}
    \emph{Suppose we are given observations $X_1, \dots, X_n$, that are assumed to be independent and with a common distribution. The non-parametric likelihood of observing samples from a distribution with CDF $F$, $L(F)$, is given by} 
    \[ L(F) := \prod\limits_{i=1}^n \lrp{ F(X_i) - F(X^-_i) }. \]
    \end{definition}

    \begin{remark}\label{rem:discreteEL}\emph{Observe that $L(F)$ is always $0$ for CDFs of continuous distributions. In particular, to have a non-trivial $L(F)$, the distribution corresponding to $F$ must put non-zero mass at least on  each of the observation points.} 
    \end{remark}

    \begin{theorem}[Informal]\label{th:ELikelihood}
    Suppose we are given observations $X_1, \dots, X_n$ that are assumed to be independent and with a common distribution. ECDF $F_n$ maximizes the empirical likelihood, i.e., for any $F\ne F_n$, $L(F_n) > L(F)$. 
    \end{theorem}
    We refer the reader to \citet[Theorem 2.1]{owen2001empirical} for a proof. 

    As in the parametric inference, we may base hypothesis tests on empirical likelihood ratio: if $L(F)$ is much smaller than $L(\tilde{F})$, we reject the hypothesis that the samples are generated from some distribution with CDF $F$. We refer the reader to \cite{owen2001empirical} for details of this approach and to Section~\ref{sec:ELT.mean} for an application of this method in the MAB setting, where the algorithm uses the empirical likelihood ratio method to test if arm $i$ has the maximum mean, for each arm $i$.

\section{Proof of Lemma~\ref{lem:lem201}}\label{sec:klrightdenselemma}
Let $\mathcal P(\Re)$ denote the collection of all probability measures with support in $\Re$, and set $\mathcal H = \mathcal P(\Re)$. For $\eta\in\mathcal P(\Re)$, let $m(\eta)$ denote its mean. For probability measure $\eta$, let $F_\eta$ denote the associated CDF function. Consider a large $y$ whose value will be fixed later, and take $\gamma \in (0,1)$.
Construct  another distribution function  $\kappa$ as follows: Set
\[
F_\kappa(x) = (1-\gamma) F_\eta(x), \quad  \forall x \leq y, \] 
and
\[
1-F_{\kappa}(x) = \beta (1-F_{\eta})(x), \quad \forall x > y.
\]
Note that
\[
\beta = 1+ \gamma \frac{F_\eta(y)}{1-F_{\eta}(y)} > 1.
\]
Then,
\begin{equation*}
0 \leq \KL(\eta,\kappa) = \int_{x \in \Re} \log \left (\frac{dF_\eta(x)}{d F_\kappa(x)} \right ) dF_\eta(x)
\leq   - F_\eta(y) \log (1-\gamma) .
\end{equation*}

By selecting
$\gamma = 1-\exp(-a)$, we get
\[
\KL(\eta,\kappa) \leq a.
\]

Moreover, for $y$ such that
$\eta(y^+) = \eta(y^-)$,
\begin{align*}
m(\kappa) &= (1-\gamma) \int_{-\infty}^{y} x d F_\eta(x)
+\left (1+  \frac{\gamma F_\eta(y)}{1-F_{\eta}(y)} \right ) \int_{y}^{\infty} x d F_\eta(x)\\ &\geq \exp(-a)m(\eta) + (1- \exp(-a)) F_\eta(y) y.
\end{align*}
Since, RHS increases to infinity as
$y \rightarrow \infty$, one can select $y$ sufficiently large so that
$m(\kappa) \geq b$.

\section{Proofs from Section~\ref{sec:simpLB}}\label{app:lb}
In this section, we give proofs of the results related to simplification of the lower bound max-min problem and the properties of the solutions of this max-min problem. 

\paragraph{Notation. }Recall that we denote the given bandit instance by $\mu \in \mathcal M$. For simplicity of notation, we assume that the unique best arm in $\mu$ is arm $1$, i.e., $a^*(\mu) = 1$. Furthermore, we denote the collection of all the bandit instances in which arm $1$ sub-optimal by $\operatorname{Alt}(\mu)$, i.e., 
\[ \operatorname{Alt}(\mu) = \bigcup\limits_{j\ne 1} \lrset{\nu\in\mathcal M: ~ m(\nu_1) < m(\nu_j)}. \]
Moreover, recall that for $\eta\in\mathcal P(\Re)$ and $x\in \Re$, 
\[ \KLinfU(\eta,x) = \min\limits_{\substack{ \kappa\in\mathcal L:\\ m(\kappa) \ge x  }} \quad \text{ and }\quad \KLinfL(\eta,x) = \min\limits_{\substack{\kappa\in\mathcal L: \\ m(\kappa) \le x}} .\]
Furthermore, for $\eta\in x$, $m(\eta) \in M$, where $M =\left[-B^\frac{1}{1+\epsilon}, B^\frac{1}{1+\epsilon}\right]. $

\subsection{Proof of Lemma~\ref{lem:MeanBounds}}\label{app:lb:MeanBounds}
The bounds on mean follow from an easy application of Jensen's inequality. To show the uniqueness of probability measures with these extreme means, consider the following optimization problem for $x\in\Re$: 
$$ \min\lrset{ \mathbb{E}_\eta(|X|^{1+\epsilon})~\text{ s.t. }~ \eta\in\mathcal P(\Re) ~\text{ and }~m(\eta) =x },$$ 
for which $\eta = \delta_x$ is the unique minimizer. This again follows from Jensen's inequality since $\abs{\cdot}^{1+\epsilon}$ is a convex function. Choosing $x\in\{ -B^{\frac{1}{1+\epsilon}}, B^{\frac{1}{1+\epsilon}} \}$ gives the desired result. 

\subsection{Proof of Lemma \ref{lem:maxminsimple.mean}}\label{sec:proofmaxminmean}
Recall that for $\mu\in \mathcal M$ with arm $1$ being the unique optimal arm, the max-min lower bound problem is given by 
\[ V(\mu) = \sup\limits_{t\in\Sigma_K} ~ \inf\limits_{\nu\in\operatorname{Alt}(\mu)}~\sum\limits_{a=1}^K t_a \KL(\mu_a, \nu_1) .\]

     Let 
     $
          \mathcal S_j = \lrset{\nu\in\mathcal M: ~ m(\nu_1) < m(\nu_j) }.
     $
     Then, for fixed $t\in\Sigma_K$, the inner minimization problem in $V(\mu)$ equals 
     \[ \min\limits_{j\ne 1} \inf\limits_{\nu\in\mathcal S_j} \lrset{\sum\limits_{a=1}^K t_a\KL(\mu_a,\nu_a) }. \]
     Clearly, choosing $\nu\in\mathcal S_j$ such that $\nu_i = \mu_1$ for $i \notin \lrset{1,j}$, makes the value of the function being minimized in the above expression smaller. Hence, the value of the infimum problem in $V(\mu)$ equals
     \[ \min\limits_{j\ne 1} \inf\limits_{\substack{ \nu_1\in\mathcal L, \nu_j \in \mathcal L: \\ m(\nu_1) \le m(\nu_j) }} \lrset{t_1\KL(\mu_1, \nu_1) + t_j \KL(\mu_j, \nu_j)}. \]
     Introducing $x\in\Re$ and $y\in\Re$ such that $x\le y$, and using the definitions of $\KLinfU$ and $\KLinfL$, the above equals
     \[ \min\limits_{j\ne 1} \inf\limits_{x\le y} \lrset{t_1\KLinfL(\mu_1, x) + t_j \KLinfU(\mu_j, y)}, \]
     giving the desired result. 

\subsection{Proof of Lemma~\ref{lem:propL}}\label{sec:proof_propL}
\emph{Convexity} of $\mathcal L $ follows from its definition. 

\vspace{2mm}
\noindent\emph{Uniform integrability: } Since each probability measure $\eta$ in \(\cal L  \) has a uniformly bounded \(p^{th}\) moment for a fixed \( p > 1 \), their Skorokhod transforms, $$\lrset{u \mapsto F_\eta^{-1}(u): \eta \in \mathcal L},$$ form a uniformly integrable collection \cite{williams1991probability}.

\vspace{2mm}
\noindent\emph{Compactness of $\cal L$:} It is sufficient to show that \( \cal L \) is closed and tight. Prohorov's Theorem then gives that it is a compact set in the topology of weak convergence \cite{billingsley2013convergence}. We first show that it is a closed set. Towards this, consider a sequence \( \eta_n \) of probability measures in \( \cal L \), converging weakly to \(\eta \in \mathcal{P}(\Re) \). By Skorohod's Representation Theorem (see \cite{billingsley2013convergence}), there exist random variables \(Y_n, Y \) defined on a common probability space, say \((\Omega, \mathcal{F}, q) \), such that \( Y_n \sim \eta_n \), \(Y \sim  \eta\), and \(Y_n \xrightarrow{a.s.} Y \). Then, by Fatou's Lemma, 
\[ \E{\eta}{\abs{X}^{1+\epsilon}} = \E{q}{\abs{Y}^{1+\epsilon}} = \E{q}{\liminf\limits_{n\rightarrow \infty}\abs{Y_n}^{1+\epsilon}} \leq \liminf\limits_{n\rightarrow \infty}\E{q}{\abs{Y_n}^{1+\epsilon}} \leq B. \]
Hence, \(\eta\) is in \(\cal L \) and the class is closed in the weak topology. To see that it is tight, consider  
$$ K_\epsilon := \left[ -\lrp{2B{\epsilon}\inv}^{\frac{1}{1+\epsilon}}, \lrp{2B{\epsilon}\inv}^{\frac{1}{1+\epsilon}} \right]. $$ 
For \(\eta \in \cal L\), \( \eta\lrp{K^c_\epsilon} \leq \epsilon \).

\vspace{2mm}
\noindent\emph{Convergence of mean:} Consider a sequence \(\eta_n \in  \cal L \) weakly converging to \(\eta \in \cal L \). Then, there exist random variables \(Y_n, Y \) defined on a common probability space \((\Omega, \mathcal{F},q) \) such that \(Y_n \sim \eta_n \), \(Y \sim \eta \), and \(Y_n\xrightarrow{a.s.} Y\) (Skorohod's Theorem, see \cite{billingsley2013convergence}). Since \(\eta_n,\eta \) are uniformly integrable, 
	$$m(\eta_n) = \E{q}{{Y_n}} \longrightarrow \E{q}{{Y}} = m(\eta)$$ 
	(see \cite[Theorem 13.7]{williams1991probability}).

\subsection{Proof of Lemma~\ref{lem:propklinf} and Supporting Results}\label{sec:prop.klinf.mean}
We now prove various properties of these projection functions, which use the dual formulations, the properties of the optimal solutions, and the properties of the class $\mathcal L$ from Lemma~\ref{lem:propL}. Recall from the definitions of these projection functionals that for $\eta \in \mathcal L$, 
$$ \forall x\le m(\eta), \quad \KLinfU(\eta,x) = 0,$$ 
and 
$$ \forall x\ge m(\eta), \quad \KLinfL(\eta,x) = 0.$$ 

Moreover, for a fixed $\eta$, $\KLinfU(\eta,\cdot)$ is a non-decreasing function and $\KLinfL(\eta,\cdot)$ is a non-increasing  function.

\subsubsection{Proof of Lemma~\ref{lem:propklinf}} 
\emph{Joint convexity:} Consider two probability measures, $\eta_1$ and $\eta_2$, and let $x_1\in \Re$ and $x_2\in\Re$. Let $\kappa_1 \in \mathcal L$ and $\kappa_2 \in \mathcal L$ be the optimal primal variables from Theorem~\ref{th:klinfDual_mean} for the $\KLinfU(\eta_1,x_1) $ and $\KLinfU(\eta_2, x_2)$, respectively. For $\lambda \in (0,1)$, let 
$$\eta_{12}:= \lambda \eta_1 + (1-\lambda)\eta_2,$$ 
$$x_{12} := \lambda x_1 + (1-\lambda) x_2,$$
and
$$\kappa_{12} := \lambda \kappa_1 + (1-\lambda) \kappa_2.$$ 
Clearly, $\kappa_{12} $ is feasible for the $\KLinfU(\eta_{12}, x_{12})$ problem. Hence, $\KLinfU(\eta_{12}, x_{12})$ is at most $\KL(\eta_{12}, \kappa_{12}) $. By convexity of $\KL$, 
\[ \KL(\eta_{12}, \kappa_{12}) \le \lambda \KL(\eta_1, \kappa_1) + (1-\lambda)\KL(\eta_2, \kappa_2),\numberthis\label{eq:KL12} \]
which equals the convex combination of $\KLinfU(\eta_1,x_1)$ and $\KLinfU(\eta_2, x_2)$, as desired.

\vspace{2mm}
\noindent\emph{Strict convexity:} In the above discussion, let $\eta\in\cal L$ and consider $\eta_1 = \eta_2 = \eta$. Furthermore, let $x_2 > x_1 > m(\eta)$. Since $\kappa_1$ and $\kappa_2$ differ on $\Supp(\eta)$ (see Lemma~\ref{lem:soldiffS_klinfu}), (\ref{eq:KL12}) holds with strict inequality. This follows from strict convexity of $\KL(\eta,\cdot)$ on $\Supp(\eta)$. 

Symmetric arguments prove these properties for $\KLinfL$. \BlackBox

\vspace{2mm}
\noindent\emph{Joint continuity:} Lemma~\ref{lem:lscklinf_mean} and Lemma~\ref{lem:USCGEQ_mean} below, together imply that $\KLinfU$ and $\KLinfL$ are jointly continuous on $\mathcal L \times M^o$.

\subsubsection{Supporting Results}
We first show that the $\KL$ projection functionals, $\KLinfU(\cdot,\cdot)$ and $\KLinfL(\cdot, \cdot)$ are jointly lower-semicontinuous in their arguments. 

\paragraph{Notation. }Recall that 
\[M  = \left[ -B^{\frac{1}{1+\epsilon}}, B^{\frac{1}{1+\epsilon}}\right].\] 
For $x\in\Re$, let 
\[ 	\mathcal{D}^L_x \triangleq \lrset{\eta \in \mathcal{L} : m(\eta) \leq x}\quad\text{ and }\quad \mathcal{D}^U_x \triangleq \lrset{\eta \in \mathcal{L} : m(\eta) \geq x}. \]
For $\eta\in \mathcal P(\Re)$, recall that 
	\[ \KLinfL(\eta,x) = \min\limits_{\kappa \in D^L_x}~ \KL(\eta,\kappa) \quad \text{ and } \quad \KLinfU(\eta,x) = \min\limits_{ \kappa \in D^U_x }~\KL(\eta,\kappa). \]
For $y\in\Re$, let $f(y) = \abs{y}^{1+\epsilon}$, and for $c \ge 0$, let $f^{-1}(c) := \max\{ y : ~ f(y) \le c \}$. Then, $M = [ -f^{-1}(B), f^{-1}(B) ].$ Let $M^o$ denote the interior of the set $M$. 

\begin{lemma}\label{lem:001}
	For  $x\in M$, the sets \( \mathcal{D}^L_x \) and \(\mathcal{D}^U_x \) are closed, compact, and  convex subsets of $\mathcal P(\Re)$ in the topology of weak convergence. 
\end{lemma}
\begin{proof}
	The convexity of these sets follows from their definitions. Next, since \(\cal L\) is compact (Lemma~\ref{lem:propL}), to show that $\mathcal D^L_x$ and $\mathcal D^U_x$ are compact subsets, it is sufficient to show that these are closed subsets of $\cal L$. To this end, consider a weakly converging sequence $\{\eta_n\}_{n\in\mathbb N} \in \mathcal D^U_x$ such that $\eta_n \xRightarrow{D} \eta$. Since $\cal L $ is compact, $\eta\in\cal L$. It now follows from Lemma~\ref{lem:propL} that $m(\eta) \ge x$, implying that $\eta\in\mathcal D^U_x$. 
\end{proof}

\begin{lemma}\label{lem:lscklinf_mean}
	For \(\eta\in \mathcal{P}(\Re)\) and \( x\in M \), the functionals \(\KLinfU(\cdot,\cdot) \) and \(\KLinfL(\cdot,\cdot) \) are jointly lower-semicontinuous at \((\eta,x) \).
\end{lemma}
\begin{proof}
	For \(\eta,\kappa \in \mathcal{P}(\Re) \), \(\KL(\eta,\kappa) \) is jointly lower-semicontinuous function in the topology of weak convergence (see \citet{posner1975random}) and hence, a jointly lower-semicontinuous function of \((\eta,\kappa,x)\). Since \(\mathcal D^U_x\) is a compact set for each \(x\) (Lemma \ref{lem:001}), it is sufficient to show that \(\mathcal  D^U_x\) is an upper-hemicontinuous correspondence (see Theorem~\ref{th:partialberge} and \citet[Theorem 1, Page 115]{berge1997topological}). 
	
	\vspace{2mm}
	\noindent\emph{Upper-hemicontinuity:} Fix $\tilde{x} \in M$. We will show upper-hemicontinuity at $\tilde{x}$. See, Section~\ref{sec:Berge} for a  definition of upper-hemicontinuity and \citet[Proposition 9.8]{SundaramOpt1996} for its sequential characterization, which we will use here. 

	 Consider a sequence \(x_n \in M\) such that \(x_n \rightarrow \tilde{x}\). Let \( \eta_n \in \mathcal  D^U_{x_n}  \), which exist since \( D^U_{x_n} \) are non-empty sets. Since \( \cal L \) is a tight, and hence relatively compact collection of probability measures (see Lemma~\ref{lem:propL}), and \( \eta_n \in \cal L \), \(\eta_n \) has a weakly convergent sub-sequence, say \(\eta_{n_i} \) converging to \(\eta \in \cal L \) (since \(\cal L\) is also closed). Furthermore, since $\eta_{n_i} \in \mathcal D^U_{x_{n_i}}$, 
	$ m(\eta_{n_i}) \geq x_{n_i}. $ From Lemma \ref{lem:propL}, \( m(\eta) = \lim_{n_i} m(\eta_{n_i}) \geq \tilde{x} \), which implies that \(\eta \in \mathcal  D^U_{\tilde{x}} \), proving upper-hemicontinuity of the set \( \mathcal D^U_{\tilde{x}}\) at \(\tilde{x}\).
	
	Similar arguments hold for \( \KLinfL(\cdot, \cdot) \). 
\end{proof}

To prove upper-semicontinuity, we use dual formulations from Theorem~\ref{th:klinfDual_mean}. We recall some notation here. 

\paragraph{Notation. } Recall that $\Re^+$ and $\Re^-$ denote the collection of non-negative and non-positive real numbers. For \(x\in M^o\), \(\bm{\lambda} \in \Re^2\), $\bm{\gamma}\in \Re^2$, and \(X \in \Re \),  recall that \[
g^U( X,\bm{\lambda}, x) = 1 - \lambda_1 (X-x)  - \lambda_2 (B-\abs{X}^{1+\epsilon}), \numberthis\label{eq:gu} \] 
and
\[{S}^U(x) = \lrset{\lambda_1 \geq 0, \lambda_2\geq 0: ~~ 1 + \lambda_1 x - \lambda_2 B -  \frac{\epsilon\lambda^{1+\frac{1}{\epsilon}}_1}{(1+\epsilon)^{1+\frac{1}{\epsilon}} \lambda^{\frac{1}{\epsilon}}_2 } \ge 0 }. \numberthis\label{eq:SUmean}  \]
Notice that $S^U$ is a convex set.

\begin{lemma}\label{lem:USCGEQ_mean}
	\(\KLinfU \), viewed as a function from \( \mathcal{L}\times [ -B^{\frac{1}{1+\epsilon}},B^{\frac{1}{1+\epsilon}} ) \), is a jointly upper-semicontinuous function.
\end{lemma}
\begin{proof}
	Let 
	\[ \bar{M} = \left[\left. -B^{\frac{1}{1+\epsilon}}, B^{\frac{1}{1+\epsilon}} \right)\right. = \left[\left.-f\inv(B),f\inv\lrp{B}\right)\right. . \]
	We prove in Theorem~\ref{th:klinfDual_mean}\ref{th:klinfUDual_mean} that for \(x\in \bar{M}\), 
	\[\KLinfU(\eta,x) =   \max\limits_{\bm{\lambda}\in S^U(x)} ~ \E{\eta}{\log g^U(X,\bm{\lambda},x)},\]
	where $g^U(\cdot, \cdot, \cdot) $ is defined in (\ref{eq:gu}). 
	
	\paragraph{Proof sketch:} Since for \(x \in  \bar{M}\), \(S^U(x)\) (defined in (\ref{eq:SUmean})) is a compact set (see Lemma~\ref{lem:compactdualspace}), and for all \(y\in\Re\) \(g^U(y,\cdot,\cdot)\) is a jointly continuous map. We will show that the set $S^U(\cdot)$ is an upper-hemicontinuous correspondence (defined in Section~\ref{sec:Berge}) and the objective function, $\E{\eta}{\log\lrp{g^U(X,\bm{\lambda},x)}}$, is jointly upper-semicontinuous in $(x,\eta,\bm{\lambda})$. Then Theorem~\ref{th:partialberge} implies that $\KLinf(\cdot,\cdot)$ is jointly upper-semicontinuous function of its arguments. 

	\paragraph{Upper-hemicontinuity of $S^U(\cdot)$: }
	Clearly, \((0,0) \in S^U(x)\) for all \(x\in \bar{M}\). Next, consider a sequence \( x_n \longrightarrow x_0 \in \bar{M}\) and a sequence \( \bm{\lambda}_n \in S^U(x_n) \). Since \(x_n \longrightarrow x_0 \), there exists a closed and bounded (compact) subset, K, of \( \Re \) containing \(x_0\), such that for some \(J \geq 1\), and all \(n \geq J\), \(x_n \in K\). Since \( \min_y ~ g^U(y,\cdot,\cdot) \) is a jointly continuous function, for \(n\geq J\), \(\bm{\lambda}_n\) also belongs to a compact subset of \(\Re^2\). Bolzano-Weierstrass theorem then gives a convergent sub-sequence \( \lrset{(x_{n_i},\bm{\lambda}_{n_i}} \) in \( \Re^3 \) with the limit \( \lrset{x_0,\bm{\lambda}} \). It is then sufficient to show that \( \bm{\lambda}\) lies in \( S^U(x_0) \), which follows since
	\[  g^U(y,\bm{\lambda}_n, x_n) \geq 0 ~~\Rightarrow ~~ g^U(y,\bm{\lambda}, x_0) \geq 0,\]
	proving that the correspondence \( S^U(\cdot) \) is upper-hemicontinuous (see Section~\ref{sec:Berge} and \citet[Proposition 9.8]{SundaramOpt1996} for definitions and equivalent characterizations of upper-hemicontinuity). 
	
	\paragraph{Upper-semicontinuity of the objective function:} 
	Let \( h(x,\eta,\bm{\lambda}) := \E{\eta}{\log\lrp{g^U(X,\bm{\lambda},x)}} \). Consider a sequence \((x_n, \eta_n, \bm{\lambda}_n) \in \bar{M}\times \mathcal{L}\times {S}^U(x_n)\) converging to \(x,\eta,\bm{\lambda} \in M\times \mathcal{L} \times{S}^U(x)\). Notice that the convergence is defined coordinate-wise, and \(\eta_n\) converges to \(\eta\) in weak topology. We want to show the following:
	\[ \limsup\limits_{n\rightarrow \infty}~~ h(x_n,\eta_n, \bm{\lambda}_n) \leq h(x,\eta,\bm{\lambda}).  \]
	
	By Skorohod's Theorem (see \cite{billingsley2013convergence}), there exist random variables \(Y_n,Y \) defined on a common probability space \((\Omega, \mathcal{F},q) \) such that \(Y_n\sim \eta_n \), \(Y\sim \eta \) and \(Y_n \xrightarrow{a.s.}{Y} \). Hence, \( \log\lrp{g^U(Y_n,\bm{\lambda_n}, x_n)} \xrightarrow{a.s.}\log\lrp{g^U(Y,\bm{\lambda},x)}.\) Observe that
	\[ h(x_n,\eta_n,\bm{\lambda}_n) =  \E{q}{\log\lrp{g^U(Y_n,\bm{\lambda}_n,x_n)}} \quad \text{ and }\quad h(x,\eta,\bm{\lambda}) = \E{q}{\log\lrp{g^U(Y,\bm{\lambda},x)}} .\]
	Let 
	\[ 0\leq Z_n = c_{1n} + c_{2n}\abs{Y_n} + c_{3n}\abs{Y_n}^{1+\epsilon},  \]
	where
	\[ c_{1n} = \lambda_{1n}\abs{x_n} +  \lambda_{3n}B, \quad c_{2n} = {\lambda_{1n}}, \quad c_{3n} = \lambda_{3n}, \]
	\(Z_n \xrightarrow{n\rightarrow\infty} Z\) and \( c_{in} \xrightarrow{n\rightarrow \infty}c_i < \infty \) for $i\in \lrset{1,2,3}$. With these notation, 
	\[\log\lrp{g^U(Y_n,\bm{\lambda}_n, v_n)} \le  \log(1+Z_n),\]
	 and there exist constants \( c_{0n} \xrightarrow{n\rightarrow\infty}c_0 \) such that \( \log(1+Z_n) \leq c_{0n} + Z^{\frac{1}{1+\epsilon}}_n \). Using the form of \(Z_n\) from above, there also exist constants \( c_{4n} \xrightarrow{n\rightarrow \infty} c_4 \) and \( c_{5n}\xrightarrow{n\rightarrow\infty} c_5 \) such that 
	\[ Z^{\frac{1}{1+\epsilon}}_n \leq c_{4n} + c_{5n}\abs{Y_n}. \]
	Thus, there exist constants \( c_{0n}, c_{4n}, c_{5n} \) converging to \(c_0,c_4, c_5\) such that 	
	\[ \log\lrp{g^U(Y_n,\bm{\lambda}_n,x_n)} \leq c_{0n}+c_{4n}+c_{5n}\abs{Y_n} \triangleq  f^U(Y_n,\bm{\lambda}_n,x_n) .\]
	and 
	\[f^U(Y_n,\bm{\lambda}_n,x_n) \xrightarrow{a.s.}{f^U(Y,\bm{\lambda},x)} \quad \text{ and }\quad \E{q}{f^U(Y_n,\bm{\gamma}_n, x_n)} \rightarrow \E{q}{f^U(Y,\bm{\gamma},x)}, \]
	since \( \eta_n,\eta \in \cal L \), whence \(Y_n,Y\) are uniformly integrable (see \cite{williams1991probability}) . Since, 
	\[f^U(Y_n,\bm{\lambda}_n,x_n) - \log\lrp{g^U(Y_n,\bm{\lambda}_n,x_n)}  \geq 0,\]
	by Fatou's Lemma,
	\begin{align*} 
	&\E{q}{\liminf\limits_{n\rightarrow \infty} \lrp{f^U(Y_n,\bm{\lambda}_n,x_n)-\log\lrp{g^U(Y_n,\bm{\lambda}_n,x_n)}}}\\ 
	& \qquad\qquad \leq \E{q}{f^U\lrp{Y,\bm{\lambda},x}} - \limsup\limits_{n\rightarrow \infty} \E{q}{\log\lrp{g^U(Y_n,\bm{\lambda}_n,x_n)} },
	\end{align*}
	which implies
	\begin{align*}
	h(x,\eta,\bm{\lambda}) = \E{q}{\limsup\limits_{n\rightarrow \infty} \log\lrp{g^U(Y_n,\bm{\lambda}_n,x_n)} } &\geq \limsup\limits_{n\rightarrow \infty} \E{q}{\log\lrp{g^U(Y_n,\bm{\lambda}_n,x_n)}}\\ &= \limsup\limits_{n\rightarrow \infty} h(x_n\eta_n, \bm{\lambda}_n).
	\end{align*}
\end{proof}
Arguments similar to those in the proof for Lemma~\ref{lem:USCGEQ_mean}, prove upper-semicontinuity for $\KLinfL$ on $\mathcal L\times (-B^{\frac{1}{1+\epsilon}}, B^{\frac{1}{1+\epsilon}}]$.

\subsection{Proof of Lemma~\ref{lem:positiveG}}\label{app:positiveG}
Since $x$ belongs to a compact set, from continuity (Lemma~\ref{lem:propklinf}\ref{lem:cont_klinf_mean}), $\KLinfL$ and $\KLinfU$ attain minimum in $[m(\mu_j),m(\mu_1)]$. Further, for $x$ in this range, and for $y$ and $z$ chosen as in the lemma statement, at least one of $\KLinfU(\mu_j,x)$ and $\KLinfL(\mu_1,x)$ is strictly positive, implying that $G_j(\mu,y,z) > 0$. \BlackBox

\subsection{Proof of Lemma \ref{lem:AltVmu.mean}}\label{sec:prooflemAltVmu.mean}
\paragraph{Notation. }Given $\mu \in \mathcal M$, recall that $a^*(\mu) = 1$. For \(j \in \lrset{2,\dots,K}\), consider functions 
$$g_j: \mathcal M \times \Re^3\rightarrow \Re \quad \text{ and } \quad G_j: \mathcal M\times\Re^2\rightarrow\Re. $$ Let $\Re^+$ denote the collection of non-negative reals. Then,  for $x\in\Re$, $y\in\Re^+$ and $z\in\Re^+$, recall that $g_j(\mu,y,z,x) =  y \KLinfL(\mu_1,x)+z \KLinfU(\mu_j,x)$, and 
$$G_j(\mu,y,z) = \inf_{x\in [m(\mu_j), m(\mu_1)]}g_j(y,z,x).$$ 

\paragraph{Proof for Alternate representation for $V(\mu)$: } For a fixed $t\in\Sigma_K$ and $\mu$, the fact that the infimum in the $j^{th}$ problem in (\ref{eq:Vmu.mean}) is attained at a common point follows since $\KLinfU$ and $\KLinfL$ are non-decreasing and non-increasing functions of the second arguments, respectively. The range constraint on the common point follows from the observation that for $x$ outside $[m(\mu_j), m(\mu_1)]$, one of the $\KL$ projection functionals is increasing and the other remains constant $(=0)$. Hence, the overall function value is high for $x$ outside $[m(\mu_j, \mu_1)]$. This gives (\ref{eq:altVmurep.mean}).

\paragraph{Proof for $V(\mu)> 0$ : }To see that $V(\mu)$ is strictly positive, consider $$\tilde{t} = \lrset{\frac{1}{K}, \dots, \frac{1}{K}}.$$ From Lemma~\ref{lem:positiveG}
\[ V(\mu) \ge \min\limits_{j\ne 1} G_j(\mu,\tilde{t}_1, \tilde{t}_j) > 0. \]

\paragraph{Proof for Uniqueness of the infimizers: }
For fixed $\mu$  and $t\in\Sigma_K$, for each $j\ne 1$, the function $g_j(\mu,t_1,t_j, \cdot)$ is strictly convex on $[m(\mu_j),m(\mu_1)]$ (since $\KLinfL(\mu_1,\cdot)$ and $\KLinfU(\mu_j,\cdot)$ are strictly convex on this set), implying uniqueness of the infimizers in $G_j(\mu, t_1, t_j)$.

\paragraph{Proof for Continuity of $x^*_j$ and $G_j$: }
To see continuity of $x^*_j(\cdot,\cdot,\cdot)$, observe that for each $j$, the function $g_j(\cdot,\cdot,\cdot,\cdot)$ is jointly continuous in its arguments. This follows from joint continuity of $\KLinfU$ and $\KLinfL$. Since $\cal L$ is a uniformly-integrable collection of probability measures, for $\nu_1\in\mathcal L^o $ and $\nu_j \in \mathcal L^o$ such that $m(\nu_1) \ge m(\nu_j)$, the set-valued map $\nu_1\times \nu_j \rightarrow [m(\nu_j), m(\nu_1)]$ can be verified to be a continuous correspondence (see Section~\ref{sec:Berge} for definition of continuity of set-valued maps). Hence, an application of Berge's Theorem to the optimization problem in $G_j(\mu, t_1, t_j)$ (see Section~\ref{sec:Berge}) along with the uniqueness of the minimizer implies that $x^*_j$ is a jointly-continuous function of its arguments. This also implies that $G_j(\cdot, \cdot, \cdot)$ is a jointly continuous function of its arguments.

\subsection{Proof of Theorem~\ref{PropOpt} and Supporting Results}\label{sec:proofPropOpt}
\paragraph{Notation. }For $x\in\Re$, $y\in\Re^+$ and $z\in\Re^+$, let
\begin{equation*}
g_j(\mu,y,z,x) :=  y \KLinfL(\mu_1,x)+z \KLinfU(\mu_j,x), \text{ and } G_j(\mu,y,z) := \inf_{x\in [m(\mu_j), m(\mu_1)]}g_j(y,z,x). 
\end{equation*}
Recall that $M = [-B^\frac{1}{1+\epsilon}, B^\frac{1}{1+\epsilon}]$. Furthermore, for $\nu\in\mathcal M$, $t^*:\mathcal M \rightarrow \Sigma_K$, $t^*(\nu)$ be the set of maximizers for $V(\nu)$. 

\subsubsection{Supporting Results}
We first show that if $y >0$ and $z > 0$, then the infimum in the expression for $G_j$ is attained in the interior. This will be helpful in proving monotonicity of $G_j(\mu,y,\cdot)$, which will later be used to show uniqueness of the maximizer in $V(\mu)$. 
\begin{lemma}\label{lem:minint.mean}
For $\mu\in \mathcal M^o$ with arm $1$ being the unique best, for $w \in \Sigma_K$ such that $w_1 > 0$  and for $j \ne 1$, $w_j > 0$, $x^*_j \in (m(\mu_j), m(\mu_1))$.
\end{lemma}
\begin{proof}
Lemma~\ref{lem:AltVmu.mean} shows that $x^*_j \in [m(\mu_j), m(\mu_1)]$. We now argue that if $w_1 > 0$ and $w_j > 0$, $x^*_j$ cannot be on the boundary of this set. 

Since $g_j(\nu,w_1,w_j,\cdot)$ is a convex function, it is sub-differentiable. For fixed $\mu$ and $w$, let the set of sub-gradients of $g_j$ at a point $x$ be denoted by $\partial g_j(x)$. Similarly, denote these for $\KLinfU(\mu_j, \cdot)$ and $\KLinfL(\mu_1, \cdot)$ by $\partial\KLinfU(x)$ and $\KLinfL(x)$, respectively. Since $x^*_j$ is a minimizer of $g_j$, $0\in\partial g_j(x^*_j)$.

Recall that $\KLinfU(\mu_j, \cdot)$ and $\KLinfL(\mu_1, \cdot)$ are strictly convex in $[m(\mu_j), m(\mu_1)]$ with $m(\mu_j)$ and $m(\mu_1)$ being the unique points of minimum, respectively. Hence, $0\in\partial \KLinfU(s)$ iff $s= m(\mu_j)$ and $0\in \partial\KLinfL(s)$ iff $s=m(\mu_1)$.

For $x\in [m(\mu_j), m(\mu_1)]$, $\partial g_j(x) = w_1\partial\KLinfL(x) + w_j\partial\KLinfU(x)$. Since $w_1 > 0$ and  $0\in\partial \KLinfU(m(\mu_j))$ and $0\notin \partial \KLinfL(m(\mu_j))$, $0\notin \partial g_j(m(\mu_j))$. Similarly, $0\notin\partial g_j(m(\mu_1))$. Hence, $x^*_j \in (m(\mu_1), m(\mu_j))$. 
\end{proof}

We now prove certain monotonicity properties that will be useful in proving uniqueness of the set of maximizers, $t^*$. For $j \ne 1$ and $z\in\Re^+$, define functions 
\[h_j(\mu,z)  = \inf\limits_{x \in M}~ \lrset{\KLinfL(\mu_1,x) + z \KLinfU(\mu_j, x)}.\]
Clearly, $h_j(\mu,\cdot) \in [0, \KLinfL(\mu_1, m(\mu_j))]$.

\begin{lemma}\label{lem:contandmonforh}
For $\mu \in \mathcal M^o$ with unique optimal arm, say arm $1$, for all $j\ne 1, h_j(\mu,\cdot)$ is continuous and strictly increasing on $\Re^+$.  
\end{lemma}

\begin{proof}
Continuity of $h_j(\mu,\cdot)$ follows from an application of Berge's Theorem to the optimization problem: 
$$\inf\limits_{x\in M} ~ g_j(\mu,1,\cdot,\cdot).$$ 
Since  $g_j(\mu,1,\cdot,\cdot)$ is  jointly continuous  and the set $M$ is a constant, hence continuous, conditions of Berge's Theorem are satisfied. See Section~\ref{sec:Berge} and references therein, for the statement of Berge's Theorem. 

Next, consider $z' >  z \ge 0$. Then, from Lemma~\ref{lem:minint.mean} there exists $x' \in (m(\mu_j), m(\mu_1))$ such that $h_j(\mu,z') = \KLinfL(\mu_1,x') + z'\KLinfU(\mu_1,x') $. Since $z' > z$, we have 
$$h_j(\mu,z') > \KLinfL(\mu_1,x') + z\KLinfU(\mu_1,x') \ge h_j(\mu, z), $$
proving the desired monotonicity.
\end{proof}

\subsubsection{Proof of Theorem~\ref{PropOpt}}
Suppose for $t\in t^*(\mu)$, $t_1 = 0$. Then, for all $j\ne 1$, the infimizers in $G_j(\mu,0,t_j)$ are $m(\mu_j)$ and $G_j(\mu,0,t_j)=0$ for all $j\ne 1$. Similarly, suppose $t_j = 0$ for some $j\ne 1$. Then, by the same argument as above, $G_j(\mu,t_1, t_j) = 0$. However, this contradicts the strict positivity of $V(\mu)$ (Lemma~\ref{lem:AltVmu.mean}), proving Theorem~\ref{PropOpt}\ref{PropOpt_gre0}. 

Next, since for $j\ne 1$, $G_j(\cdot,\cdot,\cdot)$ are jointly-continuous functions (Lemma~\ref{lem:AltVmu.mean}), the map $(\mu,t)\rightarrow \min_j G_j(\mu,t_1,t_j)$ is also jointly-continuous. Moreover, $\Sigma_K$ is a compact set that is continuous in $\mu$ (a constant correspondence). An application of Berge's theorem (see Section~\ref{sec:Berge}) to the outer maximization problem in (\ref{eq:altVmurep.mean}) implies that $V: \mathcal M^o \rightarrow \Re$ is a continuous function, and the set of maximizers, $t^*$, is upper-hemicontinuous (see Section~\ref{sec:Berge} for a definition). Upper-hemicontinuity together with the set being a singleton (to be proven), implies continuity of $t^*$ (see \citet[Chapter 9]{SundaramOpt1996}).

The proof for \ref{PropOpt_EqualVals} and uniqueness of the maximizer follows along the lines of \citet[Theorem 5]{garivier2016optimal} (however, carefully avoiding taking derivatives) where the authors proved similar result for the setting of SPEF. 
Since for $t\in t^*(\mu)$, $t_i > 0$ for all $i\in [K]$, we can rewrite $V(\mu)$ as 
\[ V(\mu) = \max\limits_{w\in\Sigma_K, w_1 > 0}~ w_1 ~ \min\limits_{j\ne 1} ~ h_j\lrp{\frac{w_j}{w_1}}. \]
For $t^*\in t^*(\mu)$,

\[ t^*\in \argmax\limits_{w\in\Sigma_K, w_1 >0 }~ w_1 ~ \min\limits_{j\ne 1} ~ h_j\lrp{\frac{w_j}{w_1}}.\]
Setting $y^*_i = \frac{t^*_j}{t^*_1}$ for $i\ne 1$ and using that $\sum\limits_{j\in [K]} t^*_j = 1$, we have 
\[ t^*_1 = \frac{1}{1+\sum\limits_{j\ne 1} y^*_j}, \quad \text{and} \quad t_j = \frac{y^*_j}{1+\sum\limits_{i\ne 1} y^*_i}. \numberthis \label{eq:yt}\]
Clearly, $y^* = \{y^*_j\}_{2}^K \in \Re^{K-1}$ and 
\[ y^* ~\in~ \argmax\limits_{\lrset{y_j}_2^K \in \Re^{K-1}} ~  \frac{\min \limits_{j\ne 1}~h_j(\mu,y_j)}{1+\sum\limits_{j=2}^K y_j}.  \] 
Equation~\eqref{eq:yt} implies that if the maximizers $y^*\in\Re^{K-1}$ in the above formulation are unique, then $t^*(\mu)$ is singleton. 
We now show that at optimal $y^*$, $h_j(\mu,y^*_j)$ is the same for all $j\ne 1$, proving \ref{PropOpt_EqualVals}. Moreover, since $h_j(\mu,\cdot)$ are monotonic, optimal $y^*$ will be unique (inverse of $h_j(\mu,\cdot)$ at $V(\mu)$) giving uniqueness of $t^*(\mu)$.

Let $$\mathcal B = \lrset{b\in\lrset{2,\dots, K}: ~ h_b(\mu,y^*_b) = \min\limits_{j\ne 1} h_j(\mu,y^*_j)) },$$ 
and let $\mathcal A = \lrset{2, \dots, K}\setminus\mathcal B$. Assume that $\mathcal A \ne \emptyset$. Then, for all $a\in\cal A$ and $b\in\cal B$, we have $h_a(\mu,y^*_a) > h_b(\mu,y^*_b). $ From continuity and monotonicity of $h_j(\mu,\cdot)$ from Lemma~\ref{lem:contandmonforh},  there exists $\epsilon_0 > 0$ such that 
\[ \forall a\in\mathcal A, ~ \forall b \in \mathcal B, ~~ h_a\lrp{\mu, y^*_a - \frac{\epsilon_0}{\abs{\mathcal A}}} > h_b\lrp{\mu, y^*_b + \frac{\epsilon_0}{\abs{\mathcal B}}}. \]
Setting $\bar{y}_a = y^*_a - \frac{\epsilon_0}{\abs{\mathcal A}}$ for $a\in \mathcal A$, and $\bar{y_b} = y^*_b + \frac{\epsilon_0}{\abs{\mathcal B}}$. Clearly, $\sum_j \bar{y}_j = \sum_j y^*_j$. Moreover, there exists $b\in\mathcal B$ such that 
\[ h_b\lrp{\mu, \bar{y}_b} = \min\limits_{i\in \mathcal B} h_i\lrp{\mu, \bar{y}_i} = \min\limits_{j\ne 1} h_j(\mu, \bar{y}_j). \]
Hence,
\[ \frac{\min\limits_{j\ne 1}~ h_j(\mu,\bar{y}_j) }{1+\sum\limits_{j\ne 1} \bar{y}_j} = \frac{h_b(\mu, \bar{y}_b)}{ 1+\sum\limits_{j\ne 1} y^*_j } > \frac{h_b(\mu, y^*_b)}{1+\sum\limits_{j\ne 1} y^*_j} = \frac{\min\limits_{j\ne 1} ~ h_j(\mu, y^*_j)}{ 1+ \sum\limits_{j\ne 1} y^*_j} ,\]
contradicting the optimality of $y^*$. Hence, $\mathcal A = \emptyset$, proving the desired results. \BlackBox

\section{Proofs from Section~\ref{sec:dual}}
In this section, we prove the alternative representations for the two $\KL$ projection functions $\KLinfU$ and $\KLinfL$ (Theorem \ref{th:klinfDual_mean}). Since these are symmetric quantities, with $m(\cdot) \ge x$ and $m(\cdot) \le x$ being linear constraints, they have similar properties. While we state the supporting results for both these functions, we only give proofs for these results and dual representation for $\KLinfU$. Proofs for $\KLinfL$ follow along the same lines.

For  $x\in M$ (recall definition from Section~\ref{sec:lb.mean}), define the sets  
\begin{equation*}
	\mathcal{D}^L_x \triangleq \lrset{\eta \in \mathcal{L} : m(\eta) \leq x}\quad\text{ and }\quad \mathcal{D}^U_x \triangleq \lrset{\eta \in \mathcal{L} : m(\eta) \geq x}.
\end{equation*} 
For $\eta\in\mathcal P(\Re)$, these are the feasible regions for $\KLinfL(\eta,x)$ and $\KLinfU(\eta,x)$, respectively. Recall from Lemma~\ref{lem:001} that these are closed, compact, and convex subsets of $\mathcal P(\Re)$ in the topology of weak convergence. For a set $A$, let $A^c$ denote its complement, and $|A|$ denote its size. Also, recall that we endow the space of probability measures with topology of weak convergence, or equivalently, with the L\'evy metric. For a sequence of probability measures $\eta_n$, we denote by $\eta_n \xRightarrow{D} \eta$ its convergence to $\eta$ in this topology.

\paragraph{Extending $\KL$ divergence.} Recall that \(\mathcal{P}(\Re)\) denotes the space of all probability measures on \(\Re\). Let \(\mathcal M^+(\Re)\) denote the collection of all finite, positive measures on \(\Re\). Extend the Kullback-Leibler Divergence to \(\mathcal M^+(\Re)\times \mathcal M^+(\Re)\), i.e., let \(\KL: \mathcal M^+(\Re)\times \mathcal M^+(\Re)\rightarrow \Re\) defined as:
\[\KL(\kappa_1,\kappa_2) \triangleq \int_{y\in\Re} \log\lrp{ \frac{d\kappa_1}{d\kappa_2}(y)} d\kappa_1(y). \]
Note that for \(\kappa_1 \in \mathcal{P}(\Re)\) and \(\kappa_2 \in \mathcal{P}(\Re)\), \(\KL(\kappa_1,\kappa_2) \) is the usual Kullback-Leibler Divergence between the probability measures. 
 
For $y\in\Re$, let $f(y) = \abs{y}^{1+\epsilon}$. Recall that for \(\eta \in\mathcal{P}(\Re)\), \(x \in \Re\), and \(B > f({x})\),  \(\KLinfU(\eta,x)\) is defined as the solution to the following optimization problem, \(\mathcal{O}_1\):
\begin{equation*}
\inf\limits_{\kappa \in \mathcal M^+(\Re)} \KL(\eta,\kappa)\quad  \text{ s.t. } \quad \int\limits_{y\in\Re}yd \kappa(y) \geq x,\quad  \int\limits_{y\in\Re}f({y})d\kappa(y) \leq B, \quad \int\limits_{y\in\Re}d \kappa(y) = 1.
\end{equation*}

\subsection{The Lagrangian Dual Problem}
We now present the Lagrangian dual problem for ${\mathcal O_1}$. Let \(\bm{\lambda}=\lrp{\lambda_1,\lambda_2,\lambda_{3}} \in \Re^3\). For \(\kappa\in \mathcal M^+(\Re)\), the Lagrangian, denoted by \(L(\kappa, \bm{\lambda})\), for the Problem \(\mathcal{O}_1\) is given by, 
\begin{align*}
L(\kappa, \bm{\lambda}) &= \int\limits_{\Re} \log \lrp{\frac{d \eta}{d\kappa}(y)} d\eta(y) + \lambda_{3} - \lambda_2\int\limits_{y \in \Re} d\kappa(y) + \lambda_1 x \\ 
& \qquad \qquad \qquad - \lambda_1\int\limits_{y \in \Re} y d\kappa(y)
+ \lambda_2 B + \lambda_2\int\limits_{y \in \Re} f({y}) d\kappa(y). \label{lagrangian} \numberthis
\end{align*}

The Lagrangian dual problem corresponding to the Problem \(\mathcal{O}_1\) is given by
\begin{equation}
\label{lagrangianDual}
\max\limits_{\substack{\lambda_{3}\in\Re,\lambda_1\geq 0, \lambda_{2}\geq 0}} \lrp{\inf\limits_{\kappa\in \mathcal M^+(\Re)} ~~~ L(\kappa,\bm{\lambda})}.
\end{equation}

Let \(\Supp(\kappa)\) denote the support of measure \(\kappa\),  
\begin{equation*}
h(y,\bm{\lambda}) \triangleq -\lambda_{3}-y\lambda_1+f({y})\lambda_{2}, \;\;\;\;\; \mathcal{Z}(\bm{\lambda}) = \lrset{y\in\Re: h(y,\bm{\lambda})=0},
\end{equation*}
and 
\[S_1 = \lrset{\bm{\lambda}\in\Re^3: \lambda_1 \geq 0, \; \lambda_2 \geq 0, \lambda_3 \in \Re, ~~\inf\limits_{y\in\Re}~ h(y,\bm{\lambda}) \geq 0  }. \]
Observe that for \(\bm{\lambda}\in S_1\), \(\mathcal{Z}(\bm{\lambda})\) is either singleton or empty set. This is easy to see since \(f(\cdot)\) is strictly convex and continuous function. In particular, if \(\mathcal{Z}(\bm{\lambda}) \) is non-empty, \(y_0\) that minimizes \(h(y,\bm{\lambda})\) is the unique element in \(\mathcal{Z}(\bm{\lambda})\).
\begin{lemma}\label{lem:probO2}
	The Lagrangian dual problem (\ref{lagrangianDual}) is simplified as below. 
	\begin{equation*}
	\max\limits_{\substack{\lambda_{3}\in\Re,\lambda_1\geq 0, \lambda_{2}\geq 0}} \lrp{\inf\limits_{\kappa\in \mathcal M^+(\Re)} L(\kappa,\bm{\lambda})} = \max\limits_{\bm{\lambda} \in S_1 } \lrp{\inf\limits_{\kappa\in \mathcal M^+(\Re)} L(\kappa,\bm{\lambda})} .
	\end{equation*}
\end{lemma}
We call the problem on right as \(\bm{\mathcal{O}_2}\).

\begin{proof}
	Let \(\bm{\lambda}\in \Re^3\setminus S_1\). Then, there exists \(y_0\in\Re \) such that \(h(y_0,\bm{\lambda}) < 0\). We show below that for such a \(\bm{\lambda},\) the inner infimum in the problem on left above, equals $-\infty$. Thus, to maximize this infimum, it is sufficient to consider \(\bm{\lambda }\in S_1. \) Towards this, for \(M > 0\), consider the measure \(\kappa_M\in \mathcal M^+(\Re) \) be such  that  \(\kappa_M(y_0) = M\) and 
	\[\frac{d\eta}{d \kappa_M}(y) = 1, \quad  \text{for } \quad y \in\Supp(\eta)\setminus \lrset{y_0}.\]
	Then the Lagrangian in (\ref{lagrangian}) can be re-written as: 
	
	\begin{align*}
	L(\kappa_M,\bm{\lambda} ) &= \underbrace{\int\limits_{y\in\Re} \log\lrp{\frac{d \eta}{d \kappa_M}(y)} d \eta(y)}_{\triangleq A} + \underbrace{\int\limits_{y\in \Re}{h(y,\bm{\lambda})  } d\kappa_M(y)}_{\triangleq B} +    \lambda_3+\lambda_1 x - \lambda_2B.
	\end{align*}
	From above, it can be easily seen that \(L(\kappa_M, \bm{\lambda})\xrightarrow{M\rightarrow \infty}-\infty\), since \(A+B\rightarrow -\infty \), implying that the infimum is $-\infty$, giving the desired result.
\end{proof}

\begin{lemma}
\label{lem:MinimizerOfLagrangian_mean}
For \(\bm{\lambda}\in S_1 \), there is a unique \(\kappa^*\in \mathcal M^+(\Re)\) that minimizes \(L(\kappa, \bm{\lambda})\). It satisfies
\begin{equation}
\label{eq_OptParams}
\Supp(\kappa^*)\subset \lrset{\Supp(\eta)\cup \mathcal{Z}(\bm{\lambda})}.
\end{equation}

Furthermore, for \(y \in \Supp(\eta)\), \(h(y,\bm{\lambda}) >0 \), and 
\begin{equation}
\label{eq_OptDist}
\frac{d \kappa^*}{d \eta}(y) = \frac{1}{-\lambda_{3} - \lambda_1 y +\lambda_2 f({y})}.
\end{equation}
\end{lemma}

\begin{proof}
\label{proof_lem:MinimizerOfLagrangian_mean}
First, observe that for a fixed set of dual variables $\bm{\lambda}$, if there are multiple measures that minimize $L(\cdot, \bm{\lambda})$, then all of them agree on $\Supp(\eta)$. This follows from strict convexity of $\KL(\eta,\cdot)$, hence of $L(\cdot,\bm{\lambda})$, on support of $\eta$. 

Before establishing uniqueness of the optimizer, we first show that any measure, say \(\kappa^* \), satisfying (\ref{eq_OptParams}) and (\ref{eq_OptDist}) minimizes \(L(\kappa, \bm{\lambda})\). Let \(\kappa_1\) be any measure in \(\mathcal M^+(\Re)\) that is different from \(\kappa^*\). Since \(\mathcal M^+(\Re)\) is a convex set, for \(t\in [0,1]\), 
$$\kappa_{2,t}\triangleq (1-t)\kappa^*+t\kappa_1$$ 
belongs to 
\(\mathcal M^+(\Re).\) Since \(L(\kappa,\bm{\lambda})\) is convex in $\kappa$, to show that \(\kappa^*\) minimizes \(L(\kappa,\bm{\lambda})\), it suffices to show 
\[ \frac{\partial L\lrp{\kappa_{2,t},\bm{\lambda}}}{\partial t}\bigg\rvert_{t=0} \geq 0.\]
Substituting for \(\kappa_{2,t} \) in  (\ref{lagrangian}), 
\begin{align*}
L\lrp{\kappa_{2,t}, \bm{\lambda}} 	=\int\limits_{y\in\Supp(\eta)}\log\lrp{\frac{d\eta}{d\kappa_{2,t}}(y)} d\eta(y) &+\lrp{\lambda_{3}+\lambda_1 x - \lambda_2 B}+\int\limits_{\Re}h(y,\bm{\lambda})d\kappa_{2,t}(y).
\end{align*}	
Evaluating the derivative with respect to \(t\) at \(t=0\), 
\begin{align*}
\frac{\partial L\lrp{\kappa_{2,t}, \bm{\lambda}}}{\partial t}\bigg\rvert_{t=0}&= \int\limits_{y\in\Supp(\eta)}\frac{d\eta}{d\kappa^*}(y)(d\kappa^*-d\kappa_1)(y)+
\int\limits_{\Re}\lrp{\lambda_{3}+\lambda_1y -\lambda_2 f({y})}(d\kappa^*-d\kappa_1)(y).
\end{align*}
For \(y\in\Supp(\eta) \), \(\partial \eta /\partial \kappa^* =h(y)\). Substituting this in the above expression, we get: 
\begin{align*}
\frac{\partial L\lrp{\kappa_{2,t}, \bm{\lambda}}}{\partial t}\bigg\rvert_{t=0} ~~ 
&= ~~ \int\limits_{y\in\Supp(\eta)}h(y)(d\kappa^*-d\kappa_1)(y)- \int\limits_{\Re}h(y)(d\kappa^*-d\kappa_1)(y)\\
&= ~~ -\int\limits_{y\in \lrset{\Re\setminus\Supp(\eta)}}h(y)d\kappa^*(y) + \int\limits_{y\in \lrset{\Re\setminus\Supp(\eta)}}h(y)d\kappa_1(y)\\
&\geq ~~ 0.
\end{align*}
where, for the last inequality, we have used the fact that for \(y\in \lrset{\Supp(\kappa^*)\setminus\Supp(\eta)}\), \(h(y)=0\) and \(h(y)\geq 0 \), otherwise.

\paragraph{Uniqueness of the minimizer. } Fix $\bm{\lambda}\in S_1$. Let there be another optimizer, $\kappa_{\bm{\lambda}}$. Then, $\kappa_{\bm{\lambda}}$ and $\kappa^*$ only differ outside $\Supp(\eta)$. This is because of strict convexity of $\KL(\eta,\cdot)$ when restricted to $\Supp(\eta)$. Then, this implies that there exists a point $z\in\Supp(\kappa_{\bm{\lambda}})$ such that $z\not\in \Supp(\eta) \cup \mathcal Z(\bm{\lambda})$, i.e., $\kappa(z) > 0$ and $h(z,\bm{\lambda}) > 0$. Consider
\[ 
\kappa'(y) =
\begin{cases}
\kappa_{\bm{\lambda}}(y), \quad&\text{ for }y\in\Supp(\eta) \cup \mathcal Z(\bm{\lambda}),\\
0, \quad&\text{ otherwise}.
\end{cases} 
\]
Clearly, $L(\kappa',\bm{\lambda})$ is strictly smaller than $L(\kappa, \bm{\lambda})$, contradicting the optimality of $\kappa_{\bm{\lambda}}$.
\end{proof}

\subsection{Proof of Theorem~\ref{th:klinfDual_mean}}
\label{Proof_th:klinfDual_mean}

Let
\begin{equation}\label{S2}
S_2 = \lrset{(\lambda_1,\lambda_2): \lambda_1 \geq 0, \; \lambda_2 \geq 0, \quad \inf\limits_{y\in\Re}~ \lrset{1-(y-x)\lambda_1 - (B-f({y}))\lambda_2} \geq 0  }. 
\end{equation}
Furthermore, define
\begin{equation}\label{htilde}
\tilde{h}(y,(\lambda_1,\lambda_2)) \triangleq 1-(y-x)\lambda_1 - (B-f({y}))\lambda_2.
\end{equation}

To prove the alternative expression for \(\KLinfU\) given by this theorem, we first show that both the primal ( $\KLinfU$ problem, denoted as $\mathcal O_1$) and the dual problems ($\mathcal{O}_2$,  from Lemma \ref{lem:probO2}) are feasible. Further, we argue that strong duality holds for the Problem \(\mathcal{O}_1\) and show that the expression on the right in  (\ref{eq:KLinf}) is the corresponding optimal Lagrangian dual. 

\paragraph{Primal feasibility.} Let \(\delta_y\) denote a unit mass at point \(y\). Since \(f({x}) < B \), there exists \(\epsilon > 0\) such that \(f({x+\epsilon}) < B\). Consider \(\kappa_{0} = \delta_{x+\epsilon}\). Consider distribution \(\kappa'\) which is a convex combination of \(\eta \) and \(\kappa_{0}\), given by: \(\kappa'=p\kappa_{0}+(1-p)\eta\), for \(p\in [0,1] \) chosen to satisfy the following two conditions. 
\[p(x+\epsilon)+(1-p)m(\eta)\geq x \text{ and } pf({x+\epsilon})+(1-p)\mathbb{E}_{\eta}(f({X})) \leq B. \]
It is easy to check that such a \(p\) always exists. \(\kappa'\) thus obtained satisfies the constraints of \(\mathcal{O}_1\) and \(\KL(\eta,\kappa') < \infty\), since \(\Supp(\eta)\subset \Supp(\kappa')\). Hence, primal problem \(\mathcal{O}_1 \) is feasible. 

\paragraph{Dual feasibility. } Next, we claim that \(\bm{\lambda^{1}}=(0,0,-1) \) is a dual feasible solution. To this end, it is sufficient to show that \(
\min_{\kappa\in \mathcal M^+(\Re)}L(\kappa,(0,0,-1))  > -\infty 
\). Observe that for \(\kappa\in\mathcal{M}^{+}(\Re)\), \(\KL(\eta,\kappa)\) defined to extend the usual definition of Kullback-Leibler Divergence to include all measures in \(\mathcal M^+(\Re)\), can be negative with arbitrarily large magnitude. From (\ref{lagrangian}), 
\[L(\kappa, \bm{\lambda^1}) = \KL(\eta,\kappa) - 1 + \int\nolimits_{y\in\Re} d\kappa(y). \]

Let \(\tilde{\kappa} \) denote the minimizer of \(L(\kappa,\bm{\lambda^1}) \). First, observe that \( \Supp(\tilde{\kappa}) = \Supp(\eta) \). If there is a \(y\) in \(\Supp(\eta)\) but outside \(\Supp(\tilde{\kappa})\), then \(L(\tilde{\kappa},\bm{\lambda^1}) \) is \(\infty\). On the other hand, if there exists \(y \text{ in } \lrset{\Supp(\tilde{\kappa})\setminus\Supp(\eta)}, \) it only contributes to increase the integral in the above expression and thus increases \( L(\tilde{\kappa},\bm{\lambda^1}).\) Thus, \( \Supp(\tilde{\kappa}) = \Supp(\eta) \). Furthermore, from Lemma \ref{lem:MinimizerOfLagrangian_mean}, for \(y \text{ in } \Supp(\eta),\) for $\bm{\lambda^1}$, the optimal measure \(\tilde{\kappa} \) must satisfy
\[\frac{d \tilde{\kappa}}{d\eta}(y) = 1. \]
Thus, \(\tilde{\kappa} = \eta \) and \(\min\limits_{\kappa\in \mathcal M^+(\Re)} L(\kappa,\bm{\lambda^1}) = 0. \) This proves the feasibility of the dual problem \(\mathcal{O}_2\).

\paragraph{Strong duality.} Since both primal and dual problems are feasible, both have optimal solutions. Furthermore, \(\kappa_{0} = \delta_{x+\epsilon}\) defined earlier, satisfies all the inequality constraints of \((\mathcal{O}_1)\) strictly, hence lies in the interior of the feasible region (Slater's conditions are satisfied). Thus strong duality holds for the problem \((\mathcal{O}_1)\) and there exists optimal dual variable \(\bm{\lambda}^*=( \lambda^*_1, \lambda^*_{2}, \lambda^*_3) \) that attains maximum in the problem \(\mathcal{O}_2 \) (see \citet[Theorem 1, Page 224]{luenberger1969optimization}). 

Since the primal problem is minimization of a convex function (which is non-negative on the feasible set) over a closed, compact, and convex set (see Lemma \ref{lem:001} for properties of the feasible region), it attains its infimum within the set. 

\paragraph{Dual representation.} Strong duality implies
\begin{equation*} 
\KLinfU(\eta,x) = \max\limits_{\bm{\lambda} \in S_1} ~\inf\limits_{\kappa\in \mathcal M^+(\Re)} ~L(\kappa,\bm{\lambda} ).\end{equation*}

\paragraph{Simplification of the dual representation. } Let \(\kappa^* \text{ and } \bm{\lambda^*}\) denote the optimal primal and dual variables. Since strong duality holds, and the problem \((\mathcal{O}_1)\) is a convex optimization problem, KKT conditions are necessary and sufficient for \(\kappa^* \text{ and } \bm{\lambda^*}\) to be optimal variables (see \citet[ page 224]{boyd2004convex}). Hence \(\kappa^*, \lambda_{3}^*\in \Re, \lambda_1^*\geq 0, \) and  \(\lambda_2^*\geq 0 \) must satisfy the following conditions (KKT):
\begin{equation}
\label{eqDualFeasibility}
\kappa^* \in \mathcal M^+(\Re),\quad \int\limits_{\Re} yd\kappa^*(y) \geq x, \quad \int\limits_{\Re} f({y}) d\kappa^*(y) \leq B,\quad \int\limits_{y\in\Re}d\kappa^*(y)=1, 
\end{equation}
\begin{equation}
\label{eqCS_mean}
\lambda_{3}^*\lrp{1-\int\limits_{\Re} d\kappa^*(y)}=0, ~~ \lambda_1^*\lrp{x-\int\limits_{\Re} yd\kappa^*(y)} = 0, ~~ \lambda_2^*\lrp{\int\limits_{\Re} f({y})d\kappa^*(y) - B} =0,
\end{equation}
and 
\begin{equation}
\label{eq:DualFeasible}
(\lambda^*_1,\lambda^*_2,\lambda^*_3)\in S_1.
\end{equation}
Furthermore, \(\kappa^*\) minimizes \(L(\kappa, \bm{\lambda^*})\). From conditions (\ref{eqCS_mean}), and Lemma \ref{lem:MinimizerOfLagrangian_mean},

\[L(\kappa^*,\bm{\lambda^*}) = \mathbb{E}_{\eta}\lrp{\log\lrp{-\lambda^*_3-\lambda^*_1 X+\lambda^*_2f({X})}},\]
where \(X\) is the random variable distributed as \(\eta. \)

Adding the equations in (\ref{eqCS_mean}), and using the form of \(\kappa^*\) from Lemma \ref{lem:MinimizerOfLagrangian_mean}, we get 
\[\lambda_{3}^* = -1-\lambda_1^* x + \lambda_2^* B .\]
With this condition on \(\lambda_{3}^*\), the region \( S_1\) reduces to the region \(S_2\) defined earlier in (\ref{S2}).

Since we know that the optimal \(\bm{\lambda^*}\) in \(S_1 \) with the corresponding minimizer, \(\kappa^* \), satisfies the conditions in (\ref{eqCS_mean}) and that \(\lambda^*_3 \) has the specific form given above, the dual-optimal value remains  unaffected by adding these conditions as constraints in the dual-optimization problem. With these conditions, the dual reduces to 
\begin{equation*}
\max\limits_{(\lambda_1, \lambda_{2})\in S_2} \mathbb{E}_{\eta}\lrp{\log\lrp{1-(X-x)\lambda_1-(B-f({X})\lambda_2)}},
\end{equation*}
and by strong duality, this is also the value of \(\KLinfU(\eta,x) \).

\paragraph{Simplifying the dual region $S_2$. } We now show that the dual region, $S_2$, is the same as the region $S^U(x)$ in the theorem statement. Towards this, for a fixed $x$, $B$, $\lambda_1 > 0$ and $\lambda_2 > 0$, let us compute the minimizer for 
$$ \inf\limits_{y\in\Re} ~ \lrset{1-\lambda_1(y-x) - \lambda_2(B-\abs{y}^{1+\epsilon})}. $$
Since $\abs{\cdot}^{1+\epsilon}$ is a strictly convex function, we get that
 \[ y^*_{\bm{\lambda}} = \lrp{\frac{\lambda_1}{\lambda_2 (1+\epsilon)}}^{\frac{1}{\epsilon}} \numberthis \label{eq:extrapoint}\]
 is the minimizer. Substituting this back into the constraint in $S_2$, we get that the set $S_2$ is the same as the set $S^U(x)$ in the theorem statement. 

\begin{remark}\label{rem:extrapoint}
For an optimal dual $\bm{\lambda}$, $y^*_{\bm{\lambda}}$ in (\ref{eq:extrapoint}) is the unique point outside $\Supp(\eta)$ where the corresponding optimal $\kappa^*$ can put mass, if required. 
\end{remark}
 
\paragraph{Tightness of the moment constraint.} The moment constraint always holds as an equality for the primal-optimal solutions as long as $\eta$ is itself not feasible to $\KLinfU$ problem. To see this, observe that the constraint in $S^U(x)$ together with $\lambda_2 = 0$ implies that $\lambda_1 = 0$. These correspond to optimal $\lambda_1$ and $\lambda_2$ iff $\eta$ is itself feasible. If not, $\lambda_2 > 0$. Hence, by (\ref{eqCS_mean}), \(\mathbb{E}_{\kappa^*}(f({X})) = B\). 

\paragraph{Tightness of the mean constraint.}
We now show that for $x > m(\eta)$ and $\eta\in \mathcal L$, the mean constraint holds as an equality for the optimal primal solutions. 

Recall that if \(\eta\) does not have full support, \(\kappa^*\) may have support outside \(\Supp(\eta)\). Then, for some \(c\geq 0\), 
\begin{align*}
1-c &= \mathbb{E}_{\eta}\lrp{\lrp{1-(X-x)\lambda_1^*-(B-f({X}))\lambda_2^*}^{-1}}\\
&\ge \lrp{1-(m(\eta)-x)\lambda_1^*-(B-\mathbb{E}_{\eta}(f({X})))\lambda_2^*}^{-1}\numberthis \label{eq:Jensen}\\
&\geq \lrp{1-(m(\eta)-x)\lambda_1^*}^{-1}.\numberthis \label{eq:InClassL}
\end{align*}
In (\ref{eq:Jensen}) we use Jensen's inequality. It is strict for $\eta$ which are non-degenerate with respect to $(1-\lambda_1(\cdot - x) - \lambda_2 (B - \cdot^2) )$. In particular, it is strict for $\eta$ supported on at least $3$ points.  Furthermore, if \(\eta\) is supported on exactly $1$ point, i.e., it is a degenerate distribution, then \(c>0\) even if (\ref{eq:Jensen}) is not a strict inequality. Thus, for such input distributions, 
\[ \frac{1}{1-(m(\eta)-x)\lambda_1^*} < 1,\]
and hence, \(\lambda_1^* > 0\). Condition (\ref{eqCS_mean}) then implies \(m\lrp{\kappa^*}=x\). Only case that remains is $\eta$ that are supported on exactly $2$ points. 

Consider $\eta \in \mathcal L^o$, supported on exactly  $2$ points. Clearly, for such an input distribution, $\lambda^*_2 > 0$. Thus, (\ref{eq:InClassL}) holds with strict inequality. Now, only remaining distributions are $2$-point distributions, $\eta$, supported on $-B^{\frac{1}{1+\epsilon}}$ and $B^{\frac{1}{1+\epsilon}}$. For such distributions, computation shows that the optimal $\kappa$ will satisfy the mean constraint with equality.

\paragraph{Uniqueness of the solutions. } As of now, we only have a unique $\kappa^*_\lambda$ for each dual variable $\bm\lambda$. However in the dual problem, $\forall \eta\in\mathcal P(\Re)$, the objective function 
\[ \E{\eta}{\log(g^U(X,\bm{\lambda}, x))} \]
is a strictly concave function. Hence, there is a unique dual-maximizer solution, $\bm \lambda^*$. Uniqueness of the primal-optimal solution now follows.

\subsection{Proof of Lemma~\ref{lem:compactdualspace}}\label{app:compactDual}
    Let 
    \[ h(\lambda_1,\lambda_2) = 1 + \lambda_1 x - \lambda_2 B -  \frac{\epsilon\lambda^{1+\frac{1}{\epsilon}}_1}{(1+\epsilon)^{1+\frac{1}{\epsilon}} \lambda^{\frac{1}{\epsilon}}_2 }. \]
    Then the constraint in \( S^U(x) \) is \( h(\lambda_1,\lambda_2) \ge 0 \). Clearly, \( h \) is a jointly concave function of  \( \lambda_1 \) and \( \lambda_2\). Hence, $S^U$ is a convex set. Moreover, the given constraint implies 
    \[  \max_{\lambda_1 \ge 0} ~~ h(\lambda_1,\lambda_2) ~ \ge ~ 0.\] To this end, optimizing over $\lambda_1$ and  setting 
    \[ \lambda_1 = \abs{x}^\epsilon (1+\epsilon) \lambda_2,  \]
    we get 
    \[ 0 \le \lambda_2 \le \frac{1}{B-\abs{x}^{1+\epsilon}}. \]
    Similarly, optimizing over $\lambda_2$, we get that 
    \[ 0 \le \lambda_1 \le  \frac{1}{B^{\frac{1}{1+\epsilon}} -{x}}. \]
    Similar bounds for $\gamma_1$ and $\gamma_2$ can be obtained by working with the constraint in $S^L(x)$.

\subsection{Properties of Optimizers of \texorpdfstring{$\KLinfU$}{KLinfU} and \texorpdfstring{$\KLinfL$}{KLinfL}} \label{app:prop_primal_opt}
Theorem~\ref{th:klinfDual_mean} established that there is a unique measure $\kappa^*$ that is optimal for the $\KLinfU(\eta,x)$ problem and has at most $1$ point in its support outside the set  $\Supp(\eta)$. The lemma below shows that when $\eta\in\mathcal L$ and $x_2 > x_1 > m(\eta)$, then the optimal primal solutions of $\KLinfU(\eta,x_1)$ and $\KLinfU(\eta, x_2)$ cannot be the same on the common support set, $\Supp(\eta)$. This will be used to show the strict convexity of these projection functionals for an appropriate range of the second argument. 

\begin{lemma}\label{lem:soldiffS_klinfu}Let $\eta\in\mathcal L$ and $m(\eta) < x_1 < x_2 $. Then, the unique probability measures, $\kappa_1$ and $\kappa_2$, that are optimal for $\KLinfU(\eta,x_1)$ and $\KLinfU(\eta,x_2)$, respectively, differ on $\Supp(\eta)$. 
\end{lemma}
\begin{proof}
Consider the optimal $\kappa_1$ and $\kappa_2$ from Theorem~\ref{th:klinfDual_mean}. Clearly, $m(\kappa_1) = x_1$ and $m(\kappa_2) = x_2$. Suppose $\kappa_1(y) = \kappa_2(y)$, for all $y \in \Supp(\eta)$. Then, $\kappa_1$ and $\kappa_2$ must differ outside this set to respect the mean constraint. 

Let 
$$p:= \kappa_1(\Supp\lrp{\eta}) = \kappa_2(\Supp\lrp{\eta})$$ 
be the mass that these measures assign to $\Supp\lrp{\eta}$. Then, $1-p$ mass is at each of the extra points, say $z_1 \in \Re^+$ and $z_2\in \Re^+$, for $\kappa_1$ and $\kappa_2$, respectively (given by (\ref{eq:extrapoint})). But 

\[ B = \E{\kappa_1}{\abs{X}^{1+\epsilon}} = \E{\kappa_1}{\abs{X}^{1+\epsilon} \mathbb{1}\lrp{X\in\Supp(\eta)}} + (1-p) z^{1+\epsilon}_1. \]

The first equality above follows due to the tightness of the moment constraint (Theorem~\ref{th:klinfDual_mean}). Since $\kappa_1$ and $\kappa_2$ agree on $\Supp(\eta)$, the above equals 

\[ \E{\kappa_2}{\abs{X}^{1+\epsilon} \mathbb{1}\lrp{X\in\Supp(\eta)}} + (1-p) z^{1+\epsilon}_1.\]
But the moment constraint is tight for $\kappa_2$ as well. Combined with the above, this gives that $z_1 = z_2$. This implies that $\kappa_1$ and $\kappa_2$ are the same probability measures, contradicting the fact that $m(\kappa_1) = x_1 < x_2 = m(\kappa_2)$.  
\end{proof}

A similar result can be proven for $\KLinfL$, which we state without proof. 
\begin{lemma}\label{lem:soldiffS_klinfl}Let $\eta\in\mathcal L$ and $x_2 < x_1 < m(\eta)$. Then, the unique $\kappa_1$ and $\kappa_2$ that are optimal for $\KLinfL(\eta,x_1)$ and $\KLinfL(\eta,x_2)$, respectively, differ on $\Supp(\eta)$. 
\end{lemma}

\section{Proofs from Section~\ref{sec:trackstop.mean}}
\subsection{Proof of Lemma~\ref{lemma:MinNoOfSamples}}\label{app:lem:MinNoOfSamples}
The above is true for $l=1$ as 
$$N_a(m) \geq \frac{m}{K} - 1\geq {m}^{1/2}-1.$$ 
Now, suppose that at step $lm$ each arm has at least $(lm)^{1/2}-1$ samples, i.e., $N_a(lm) \geq (lm)^{1/2}-1$. Then, arm $a$ needs at most $((l+1)m)^{1/2} -  (lm)^{1/2}$ samples to ensure the condition at the end of $(l+1)^{th}$ batch. 

Since $l > 1$ and \(m\geq (K+1)^2 \), 
$$m^{1/2}((l+1)^{1/2} -  l^{1/2}) < \frac{m^{1/2}}{l^{1/2}} < \frac{m}{K}, $$
where the first inequality is trivially true. Now, since the maximum number of samples required is an integer, each arm requires at most \(\floor{\frac{m}{K}}\) samples, and the algorithm has sufficient samples to distribute. This guarantees that all arms reach the minimum threshold. 

\subsection{Proof of Lemma~\ref{Lem:ASconvergence}}\label{app:lem:ASconvergence}
In order to show that \(N_a(lm)/lm \rightarrow t^*_a(\mu) \) as \(l\rightarrow \infty \), for \(n\in\mathbb{N}\),  let $M_n$ denote the set of indices in $\{1,2, \ldots,n\}$ where \(\bf AL_1 \) flipped the coins to decide which arm to sample from. Then, for \(l\in\mathbb{N}\), from Lemma~\ref{lemma:MinNoOfSamples}, 
$lm - |M_{lm}| \leq K \lrp{\sqrt{lm}-1}$, so that 
\begin{equation*}
\frac{|M_{lm}|}{lm} \xrightarrow{a.s.} 1, \text{ as } l \rightarrow \infty. 
\end{equation*}
Further, let $I_a(i)=1$ if  arm $a$ was sampled under  ${\bf AL_1}$ at step
$i$. Then, by law of large numbers for Bernoulli random variables
\begin{equation*}
\frac{1}{|M_n|}\sum_{i \in M_n} \left (I_a(i) - t^*_a(\Pi(\hat{\mu}(i))) \right )
\xrightarrow{a.s.} 0, \text{ as }n\rightarrow\infty,
\end{equation*}
where we set $\hat{\mu}(i) = \hat{\mu}({lm})$
for $ i \in \{l m, l m+1, \ldots,  (l+1)m -1\}$ for each $l$.

Further, 
\begin{equation*}
\frac{1}{|M_n|}\sum_{i \in M_n} \left (t^*_a(\Pi(\hat{\mu}(i))) - t^*_a(\mu)\right )
\xrightarrow{a.s.} 0, \text{ as }n\rightarrow\infty,
\end{equation*}
since $\hat{\mu}(n) \rightarrow \mu$ and $\Pi(\hat{\mu}(n)) \rightarrow \mu$ as $n \rightarrow \infty$, and \(t^*\) is a continuous function (Theorem~\ref{PropOpt}). 

Furthermore, 
\begin{align*}
\frac{N_a(lm)}{lm} &= \frac{\sum_{i\in M_{lm}}I_a(i)}{lm} + \frac{\sum_{i\in [lm]\setminus M_{lm}} I_a(i)}{lm}\\ 
&= \frac{\sum_{i\in M_{lm}}I_a(i)}{\abs{M_{lm}}}\frac{\abs{M_{lm}}}{lm} + \frac{\sum_{i\in [lm]\setminus M_{lm}} I_a(i)}{lm-\abs{M_{lm}}}\frac{lm-\abs{M_{lm}}}{lm}.
\end{align*}

From above, 
\[ \frac{N_a(lm)}{lm} \xrightarrow{a.s.} t^*_a(\mu), \text{ as }l \rightarrow \infty.  \]

\subsection{Simplification of Stopping Rule in Section~\ref{sec:ELT.mean}}
\label{simpSR}

Let 
\[\bm{Y^a}:=\lrp{Y^a_i: 1\leq i \leq N_a(n)}\] 
denote the \(N_a(n)\) samples from arm \(a\). For \(\nu\in\lrp{\mathcal{P}(\Re)}^K\), let \(L_{\nu}(\bm{Y^1},\dots,\bm{Y^K})\) denote the likelihood of observing the given samples under \(\nu\). If at time \(n\), 
\[m(\hat{\mu}_j(n)) > \max_{i\ne j}m(\hat{\mu}_{i}(n)),\] 
then the set of alternative bandit instances, $\operatorname{Alt}(\hat{\mu}(n))$, is given by 
\[ \operatorname{Alt}(\hat{\mu}(n)) = \bigcup\limits_{i\ne j}\lrset{\nu\in\mathcal M: ~ m(\nu_i) > m(\nu_j) } .\]

In this setting, the log of generalized empirical likelihood ratio is given by 
\[Z_{j}(n) = \log\lrp{ \frac{L_{\hat{\mu}(n)}\lrp{\bm{Y^1},\dots,\bm{Y^K}} }{\max\limits_{\mu'\in\operatorname{Alt}(\hat{\mu}(n))} L_{\mu'}\lrp{\bm{Y^1},\dots,\bm{Y^K}} }}.\]

Since at each time the arm to pull at this time is a function of all the previous samples and the next sample is generated independently from this arm, each of the likelihoods in the above expression simplify as product of likelihood of each sample. Thus, 
\begin{align*} Z_{j}(n) &= \log\lrp {\frac{\prod\limits_{a=1}^K \prod\limits_{i=1}^{N_a(n)} \hat{\mu}_a(n)\lrp{Y^a_i}}{ \sup\limits_{\mu'\in \operatorname{Alt}(\hat{\mu}(n))} \prod\limits_{a=1}^K \prod\limits_{i=1}^{N_a(n)} \mu'_a\lrp{Y^a_i}}}\\ &= \inf\limits_{\mu'\in\operatorname{Alt}(\hat{\mu}(n))} \lrset{ \sum\limits_{a=1}^K \sum\limits_{i=1}^{N_a(n)} \lrp{\log\lrp{\hat{\mu}_a(n)\lrp{Y^a_i}} - \log\lrp{\mu'_a(Y^a_i)} }} \\
&= \inf\limits_{\mu'\in\operatorname{Alt}(\hat{\mu}(n))} \lrset{ \sum\limits_{a=1}^{K} N_a(n)\KL(\hat{\mu}_a(n),\mu'_a)} . \end{align*}

\subsection{Proof of Proposition \ref{prop:DeviationsMean}}\label{App_DeviationsMean.BAI}
\paragraph{Notation. } For $x\in\Re$, let $f(x) := \abs{x}^{1+\epsilon}$ and define \(f\inv(c) = \max\lrset{y : f(y) = c} = c^{\frac{1}{1+\epsilon}}\). Let 
\[M = \left[-B^\frac{1}{1+\epsilon}, B^\frac{1}{1+\epsilon} \right]. \]
For \(x\in M^o\), 
\[ {S}^U(x) = \lrset{\lambda_1 \geq 0, \lambda_2 \ge 0: ~  \min\limits_{y\in\Re}\lrset{ g^U(y, \bm{\lambda}, x)} \geq 0}, \] 
and 
\[ S^L(x) = \lrset{\gamma_1 \geq 0, \gamma_2 \ge 0: ~  \min\limits_{y\in\Re}\lrset{ g^L(y, \bm{\gamma}, x)} \geq 0}, \]
where 
\[ g^U(y, \bm{\lambda}, x) = 1-\lambda_1(y-x)-\lambda_2(B-\abs{y}^{1+\epsilon}),\]
and 
\[ g^L(y, \bm{\gamma}, x) = 1 + \gamma_1(y-x) - \lambda_2(B-\abs{y}^{1+\epsilon}). \]
Also, for $x\in M^o$, \( S^U(x) \) and $S^L(x)$ are convex and compact sets (Lemma~\ref{lem:compactdualspace}). Moreover, for $\eta\in\mathcal P(\Re)$ and $x\in M^o$,
\[ \KLinfU(\eta,x) = \max\limits_{\bm{\lambda} \in S^U(x)}~ \E{\eta}{g^U(X, \bm{\lambda}, x)}, \]
and
\[ \KLinfL(\eta,x) = \max\limits_{\bm{\gamma} \in S^L(x)}~ \E{\eta}{\log g^L(X, \bm{\gamma},x) }, \]

\paragraph{Dual formulations.} Observe from above that for each arm i, 
\[N_i(n)\KLinfU(\hat{\mu}_i(n),m(\mu_i)) = \max\limits_{\bm{\lambda}\in S^U(m(\mu_i)) } \sum\limits_{j=1}^{N_i(n)} \log\lrp{g^U(X^i_j,\bm{\lambda}, m(\mu_i))}, \numberthis \label{app:eq:KLinfU.BAIm} \]
and 
\[N_i(n)\KLinfL(\hat{\mu}_i(n),m(\mu_i)) = \max\limits_{\bm{\gamma}\in S^L(m(\mu_i)) } \sum\limits_{j=1}^{N_i(n)} \log\lrp{g^L(X^i_j,\bm{\gamma}, m(\mu_i))}, \numberthis \label{app:eq:KLinfL.BAIm} \]
where \( X_j^i : ~ j\in\lrset{1,\dots, N_i(n)}  \) are samples from \(\mu_i \). 

\paragraph{Constructing mixture martingales.} Let \(q_{1i}\) be a uniform prior on the set \(S^U(m(\mu_i)) \), and  \(q_{2i}\) be the uniform prior on the set \(S^L(m(\mu_i)) \). It is possible to put uniform prior on these sets as they are compact (Lemma~\ref{lem:compactdualspace}). For samples \( X^i_j : ~ j\in \lrset{ 1, \dots, N_i(n)}\), define 
\[ L_i(n)=\E{\bm{\gamma} \sim q_{2i}}{ {\prod\limits_{j=1}^{N_i(n)} g^L( X^i_j,\bm{\gamma},m(\mu_i) )} \left\lvert X^i_1,\dots, X^i_{N_i(n)} \right.},\]
and 
\[U_i(n)=\E{\bm{\lambda} \sim q_{1i}}{{\prod\limits_{j=1}^{N_i(n)} g^U( X^i_j,\bm{\lambda},m(\mu_i))}\left\lvert X^i_1,\dots, X^i_{N_i(n)} \right.}.\]

Then, using (\ref{app:eq:KLinfU.BAIm}) and Lemma~\ref{lem:exp-concave.BAIm} with $d=2$ and 
\[ g_t(\bm{\lambda}) = \log g^U(X_t,\bm{\lambda},m(\mu_i)), \]
we have the following almost-sure inequality
\[ N_i(n)\KLinfU(\hat{\mu}_i(n), m(\mu_i)) \leq \log U_i(n) + 2\log(N_i(n)+1)+1.  \]
Similarly, we also have the following almost-sure inequality
\[ N_i(n)\KLinfL(\hat{\mu}_i(n), m(\mu_i)) \leq \log L_i(n) + 2\log(N_i(n)+1)+1.  \]

Next, for each arm \(i\), let  
\[Y^L_i(n) := N_i(n)\KLinfL(\hat{\mu}_i(n), m(\mu_i)) - 2 \log{(N_i(n)+1)}-1, \]
and
\[ Y^U_i(n) := N_i(n)\KLinfU(\hat{\mu}_1(n), m(\mu_i)) - 2\log\lrp{N_i(n)+1}-1. \]
Then we have that 
\[ e^{Y^L_i(n)} \leq L_i(n)  \quad \text{ and } \quad e^{Y^U_i(n)} \leq U_i(n), ~~a.s.\]

It is easy to verify that for each arm i, \(L_i(n)\) and \(U_i(n) \) are non-negative, super-martingales satisfying 
\[\Exp{U_i(n)} \leq 1 \quad\text{ and }\quad \Exp{L_i(n)}\leq 1.\] 
Thus, \( U_i(n)L_1(n) \) is a non-negative  super-martingale with mean at most \(1\), and satisfies that the event 
\[ \lrset{ \exists n :  N_i(n)\KLinfU(\hat{\mu}_i(n), m(\mu_i)) + N_1(n)\KLinfL(\hat{\mu}_1(n), m(\mu_1)) - h(n)\geq x } \]
is contained in 
\[ \lrset{ \exists n :  L_1(n)U_i(n) \geq e^x  }. \] 
Using Ville's inequality (see \cite{ville1939etude}), we get the desired result. 

\subsection{Proof of the Sample Complexity in Theorem~\ref{bigTh2}.}
\label{App_SampleComplexity}
In this section, we formally prove that the algorithm \(\bf{AL_1}\) is asymptotically optimal, i.e., the ratio of expected number of samples needed by the algorithm to stop and \(\log\frac{1}{\delta} \) matches the lower bound, asymptotically as \(\delta\rightarrow 0 \).

\paragraph{Notation. }  Recall that we use the projection map $\Pi = (\tilde{\Pi}_1, \dots, \tilde{\Pi}_K)$ for projecting the empirical distributions on $\cal L$ before computing the optimal weights. Here, for $\eta\in\mathcal P(\Re)$,
\[ \tilde{\Pi}_i(\eta) \in \arg\!\min\limits_{\kappa \in \mathcal L} d_K(\eta,\kappa), \quad \text{ and } \quad d_K(\eta,\kappa) = \sup\limits_{x\in\Re}~ \abs{F_\eta(x) - F_\kappa(x)}, \]
where $F_\eta$ and $F_\kappa$ are the CDFs of measures $\eta$ and $\kappa$, respectively. Moreover, for simplicity of presentation, we assume that \(\mu\in\mathcal{M}\) is such that \(m(\mu_1) > \max_{j\ne 1}m(\mu_j) \), i.e., arm $1$ is the unique optimal arm in $\mu$. Also, recall that for $\eta\in\cal L$, $m(\eta)\in M$, where
\[ M = \left[ -B^{\frac{1}{1+\epsilon}}, B^\frac{1}{1+\epsilon} \right]. \]

Let $\epsilon' >0$ and $n\in\mathbb{N}$. Define $\mathcal I_{\epsilon'}:= B_\zeta(\mu_1)  \times \dots \times B_\zeta(\mu_K)$, where $B_\zeta(\mu_i) = \{ \kappa\in\mathcal P(\Re): d_K(\kappa,\mu_i) \le \zeta \}$, and $\zeta > 0$ is chosen to satisfy the following: 
\[ \mu' \in \mathcal I_{\epsilon'} \implies \forall t' \in t^*(\Pi(\mu')), \exists t \in t^*(\mu) \text{ s.t. } \|t'-t\|_\infty \le \epsilon'.\]

$\zeta \rightarrow 0 $ implies that probability measures in $B_\zeta(\mu_i)$ converge weakly to $\mu_i$, for all $i$. Also, for all $\kappa\in B_\zeta(\mu_i)$ $d_K(\kappa,\mu_i)\le \zeta$. This implies that $d_K(\kappa,\tilde{\Pi}_i(\kappa))\le \zeta$. Together, these imply that $ d_K(\tilde{\Pi}_i(\kappa), \mu_i) \le 2 \zeta $. This follows from the triangle inequality for $d_K$. Existence of $\zeta(\epsilon')$, denoted as $\zeta$, is guaranteed by continuity of the map $t^*$ (see Theorem~\ref{PropOpt}).

For $T \geq m, T \in \mathbb{N}$, set
\[l_2(T)\triangleq \floor{\frac{T}{m}}, l_1(T)  \triangleq \max\lrset{1,\floor{\frac{T^{3/4}}{m}}}, \text{ and } l_0(T)\triangleq \max\lrset{1,\floor{\frac{T^{1/4}}{m}}},\] 
and define 
$$\mathcal{G}_T(\epsilon) = \bigcap\limits_{l=l_0(T)}^{l_2(T)}\lrset{\hat{\mu}(lm)\in\mathcal{I}_{\epsilon}}\bigcap\limits_{l=l_1(T)}^{l_2(T)} \lrset{  \max\limits_{a\leq K}\abs{\frac{N_a(lm)}{lm}-t^*_a(\mu)}\leq 4\epsilon }.$$

Let $\mu'$ be a vector of \(K\), \(1\)-dimensional distributions from $\mathcal P(\Re)$, let $t'\in \Sigma_K$, and let $[K]:=\lrset{1,\dots, K}$. Define the following:

\begin{equation*}
g(\mu',t') := \max\limits_{a\in[K]}~\min\limits_{b \ne 1}~ \inf\limits_{x \in M} \lrp{t'_a \KLinfL(\mu'_1,x)+t'_b \KLinfU(\mu'_b,x)}.
\end{equation*}
Note that for $\mu\in\mathcal (P(\Re))^K$, from Lemma~\ref{lem:lscklinf_mean} and From Berge's Theorem (reproduced in Section~\ref{sec:Berge}), \(g(\mu,t)\) is a jointly lower-semicontinuous function of $(\mu, t)$. Let \(\|.\|_{\infty}\) be the maximum norm in \(\Re^K, \) and 
\begin{equation}
C^*_{\epsilon'}(\mu) \triangleq \inf\limits_{\substack{\mu'\in\mathcal{I}_{\epsilon'} \\ t': \inf_{t\in t^*(\mu)}\|t'-t^*(\mu)\|\leq 4\epsilon'}} g(\mu',t').
\end{equation}
Furthermore, let \[T_0(\delta) = \inf\lrset{T \in \mathbb{N} : l_1(T)\times m + \frac{\beta(T,\delta)}{C^*_{\epsilon}(\mu)} \leq T}.  \]
Since \(\tau_{\delta}\geq 0 \), \[\mathbb{E}_{\mu}(\tau_{\delta}) = \sum\limits_{T=0}^{\infty} \mathbb{P}_{\mu}(\tau_{\delta} \geq T) \leq T_0(\delta) + m + \sum\limits_{T= T_0(\delta)+m+1}^\infty \mathbb{P}_{\mu}(\tau_{\delta}\geq T) .\] 
From Lemma~\ref{boundOnExpStopTime} and Lemma~\ref{ProbOfCompOfGoodSet.mean} below, 
\begin{equation}\label{eqMdep}
\limsup\limits_{\delta\rightarrow 0}\frac{\mathbb{E}_{\mu}(\tau_{\delta})}{\log\lrp{1/\delta}} \leq \frac{(1+\tilde{e})}{C^*_{\epsilon}(\mu)} + \limsup\limits_{\delta\rightarrow 0} \frac{m}{\log\lrp{1/\delta}}.
\end{equation}

From lower-semicontinuity of \( g(\mu',t' ) \) in \((\mu',t')\), it follows that \(\liminf\limits_{\epsilon'\rightarrow 0 }C^*_{\epsilon'}(\mu) \ge V(\mu).\)  First letting \(\tilde{e}\rightarrow 0 \) and then letting \(\epsilon\rightarrow 0 \), we get

\begin{equation}
\label{eq:temp}
\limsup\limits_{\delta\rightarrow 0}\frac{\mathbb{E}_{\mu}(\tau_{\delta})}{\log\lrp{1/\delta}} \leq \frac{1}{V(\mu)} + \limsup\limits_{\delta\rightarrow 0} \frac{m}{\log\lrp{1/\delta}}.
\end{equation}

Since \(m=o(\log(1/\delta)) \), \(\limsup\limits_{\delta\rightarrow 0}\frac{m}{\log 1/\delta} = 0\). Using this in (\ref{eq:temp}), we get (\ref{SampleComplexityub}).


\begin{lemma}
     \label{boundOnExpStopTime}
     \begin{equation}
     \mathbb{E}_{\mu}(\tau_{\delta}) \leq {T_0(\delta) + m} + \sum\limits_{T=T_0(\delta)+m+1}^{\infty} \mathbb{P}_{\mu}(\mathcal{G}^c_{T}).
     \end{equation}
     Furthermore, for any \(\tilde{e}>0\),
     \begin{equation}
     \label{eq:T0del}
     \limsup\limits_{\delta\rightarrow 0}\frac{T_0(\delta)}{\log\lrp{1/\delta}}\leq \frac{1+\tilde{e}}{C^*_{\epsilon}(\mu)}.
     \end{equation}
\end{lemma}

\begin{lemma}
     \label{ProbOfCompOfGoodSet.mean}
     $$\limsup\limits_{\delta\rightarrow 0} \frac{\sum\limits_{T=m+1}^{\infty} \mathbb{P}_{\mu}(\mathcal{G}^c_{T}(\epsilon))}{\log\lrp{1/\delta}} = 0. $$
\end{lemma}

\subsubsection{Proof of Lemma \ref{boundOnExpStopTime}.}
On \(\mathcal{G}_{T}(\epsilon')\), for \(t\geq l_0(T)\times m\), the  stopping statistic is given by, 
$$Z(t) = \max\limits_{a\in[K]}~\min\limits_{b\ne 1}~ Z_{a,b}(t), $$ 
where
\begin{equation}
\label{eq:char}
Z_{a,b}(t) = t~ \inf\limits_{x \in M} \lrp{\frac{N_a(t)}{t} \KLinfL(\hat{\mu}_a(t),x)+\frac{N_b(t)}{t} \KLinfU(\hat{\mu}_b(t),x)}.
\end{equation}

In particular, on $\mathcal G_T(\epsilon')$, for \(T\geq m\) and \(l\geq l_1(T)\), 
\begin{equation}
\begin{aligned}
\label{eq:StoppingChar}
Z(lm) &= \max\limits_{a\in[K]}~\min\limits_{b\ne a} ~\inf\limits_{x \in M} \lrset{ N_a(lm) \KLinfL(\hat{\mu}_a(lm),x)+N_b(lm)\KLinfU(\hat{\mu}_b(lm),x)}\\
&= lm \times \max\limits_{a\in[K]}~\min\limits_{b \ne a}~ \inf\limits_{x \in M} ~ \lrp{\frac{N_a(lm)}{lm} \KLinfL(\hat{\mu}_a(lm),x)+ \frac{N_b(lm)}{lm} \KLinfU(\hat{\mu}_b(lm),x)}\\
&= lm \times g\lrp{\hat{\mu}(lm),\lrset{\frac{N_1(lm)}{lm},\dots, \frac{N_K(lm)}{lm}} }\\
&\geq lm\times C^*_{\epsilon'}(\mu).
\end{aligned}
\end{equation}

Furthermore, the stopping time is at most $m\times\inf\lrset{l\geq l_1(T) : Z(lm)\geq \beta(lm,\delta), l\in\mathbb{N}}$. On $\mathcal{G}_{T}(\epsilon)$, 
\begin{equation}
\label{eq:minTauT.mean}
\begin{aligned}
\min\{\tau_{\delta}, T\}
&\leq l_1(T)\times m  + m\sum\limits_{l = l_1(T)+1}^{l_2(T)} 
\mathbbm{1}\lrp{lm < \tau_{\delta}}\\
&\leq l_1(T)\times m + m\sum\limits_{l = l_1(T)+1}^{l_2(T)} \mathbbm{1}\lrp{Z\lrp{lm}< \beta\lrp{lm,\delta}}\\
&\leq l_1(T)\times m + m\sum\limits_{l = l_1(T)+1}^{l_2(T)} \mathbbm{1}\lrp{l<\frac{ \beta\lrp{lm,\delta}}{m C^*_{\epsilon'}(\mu)}}\\
&\leq l_1(T)\times m + \frac{\beta(T,\delta)}{C^*_{\epsilon'}(\mu)}.
\end{aligned}
\end{equation}

Recall, 
$$T_0(\delta) = \inf\lrset{t \in \mathbb{N} : l_1(t)\times m + \frac{\beta(t,\delta)}{C^*_{\epsilon'}(\mu)} \leq t}.$$ 

On \(\mathcal{G}_T\), for \(T\geq \max\lrset{m,T_0(\delta)}\), from (\ref{eq:minTauT.mean}) and definition of \(T_0(\delta)\),
\[\min\lrset{\tau_{\delta},T}\leq l_1(T)\times m + \frac{\beta(T,\delta)}{C^*_{\epsilon'}(\mu)} \leq T, \]
which gives that for such a \(T,\) \(\tau_\delta \leq T \). Thus, for $T\geq \max\lrset{m, T_0(\delta)}$, we have $\mathcal{G}_{T}(\epsilon')\subset \lrset{\tau_{\delta} \leq T}$ and hence, $\mathbb{P}_{\mu}\lrp{\tau_{\delta} > T}\leq \mathbb{P}_{\mu}(\mathcal{G}^{c}_{T})$. 
Since \(\tau_{\delta}\geq 0\), 
\begin{equation}
\label{eq:ExpStopTime.mean}
\begin{aligned}
\mathbb{E}_{\mu}(\tau_{\delta}) \leq T_0(\delta)+m + \sum\limits_{T=m+1}^{\infty}\mathbb{P}_{\mu}\lrp{\mathcal{G}^c_T(\epsilon')}.
\end{aligned}
\end{equation}
Now, to bound \(\frac{T_0(\delta)}{\log\lrp{1/\delta}}\) as \(\delta\rightarrow 0\), let $\tilde{e}>0$ and define 

\[C(\tilde{e})\triangleq\inf\lrset{T\in\mathbb{N}: T-l_1(T)\times m\geq \frac{T}{1+\tilde{e}}} \text{ and } T_2(\delta) \triangleq \inf\lrset{T\in\mathbb{N}: \frac{T}{1+\tilde{e}} \geq \frac{\beta(T,\delta)}{C^*_{\epsilon'}(\mu)}}.\]
Then, 
\begin{equation}
\label{t0delta}
T_0(\delta) \leq \inf\lrset{T\in\mathbb{N} : T-l_1(T)\times m \geq \frac{T}{1+\tilde{e}} \geq \frac{\beta(T,\delta)}{C^*_{\epsilon'}(\mu)} }\leq C(\tilde{e}) + T_2(\delta).
\end{equation}

Recall from (\ref{eq:beta.mean}) that for $n\in\mathbb{N}$, 
\[ \beta(n,\delta) = \log\lrp{\frac{K-1}{\delta}(1+n)^4}  + 2.\]
Now, from the definition of \(T_2(\delta)\) above, 
\begin{equation} \label{eqT2}
T_2\lrp{\delta} = \frac{1+\tilde{e}}{C^*_{\epsilon'}(\mu)} \log\lrp{\frac{K-1}{\delta} \lrp{1+\log\frac{K-1}{\delta}}^4}  + O\lrp{\log\log\frac{1}{\delta}}.
\end{equation}

Clearly, 
$$\limsup\limits_{\delta\rightarrow 0} \frac{T_{2}(\delta)}{\log \frac{1}{\delta}} = \frac{1+\tilde{e}}{C^*_{\epsilon'}(\mu)}  \text{ and } \limsup\limits_{\delta\rightarrow 0}\frac{C(\tilde{e})}{\log 1/\delta} = 0. $$
Taking limits in (\ref{t0delta}), 
$$ \limsup\limits_{\delta \longrightarrow 0}\frac{T_{0}(\delta)}{\log{\lrp{1/\delta}}} \leq \frac{(1+\tilde{e})}{C^*_{\epsilon'}(\mu)}. \qquad \BlackBox$$

\begin{remark}{Using (\ref{t0delta}), and (\ref{eqT2}) in (\ref{eq:ExpStopTime.mean}), for small $\delta$, 
\[
\mathbb{E}_{\mu}(\tau_{\delta})\leq \frac{1+\tilde{e}}{C^*_{\epsilon'}(\mu)} \log\lrp{\frac{K-1}{\delta} \lrp{1+\log\frac{K-1}{\delta}}^4} + m.
 \]
 Since \(\tilde{e} > 0\) is arbitrary, 
\[
\mathbb{E}_{\mu}(\tau_{\delta})\leq \frac{1}{C^*_{\epsilon'}(\mu)} \log\lrp{\frac{K-1}{\delta} \lrp{1+\log\frac{K-1}{\delta}}^4} + m .
 \]

 Now, letting \(\epsilon \) decrease to \(0\),
 \begin{equation}\label{EqForM*}
     \mathbb{E}_{\mu}(\tau_{\delta})\leq \frac{1}{V(\mu)} \log\lrp{\frac{K-1}{\delta} \lrp{1+\log\frac{K-1}{\delta}}^4} + m .
 \end{equation}
 We use (\ref{EqForM*}) in our numerical experiments to estimate the optimal batch sizes (see Section~\ref{sec:optbatch.mean}). }
 \end{remark}

\subsubsection{Proof of Lemma \ref{ProbOfCompOfGoodSet.mean}.}
Fix \(T\geq m+1\). Let 
\[
\mathcal{G}^1_T \triangleq \bigcap\limits_{l'=l_0(T)}^{l_2(T)}\lrset{\hat{\mu}(l'm)\in\mathcal{I}_{\epsilon}}. \]
 Using union bounds,  
\begin{align*}
\mathbb{P}_{\mu}(\mathcal{G}^{c}_{T}(\epsilon)) 
&\leq \sum\limits_{l=l_0(T)}^{l_2(T)} \mathbb{P}_{\mu}\lrp{\hat{\mu}(lm)\not\in \mathcal{I}_{\epsilon}} + \sum\limits_{l=l_1(T)}^{l_2(T)} \sum\limits_{i=1}^K \mathbb{P}\lrp{\abs{\frac{N_i(lm)}{lm} - t_i^*(\mu) } \geq 4\epsilon, \; \mathcal{G}^1_T  }.
\end{align*}

The first term above can be bounded as: 
\begin{equation*}
\mathbb{P}_{\mu}\lrp{\hat{\mu}(lm)\not\in \mathcal{I}_{\epsilon}} \leq \sum\limits_{i=1}^K \mathbb{P}_{\mu}\lrp{d_K(\hat{\mu}_i(lm),\mu_i)\geq \zeta}.\numberthis\label{eq:compTemp}
\end{equation*}

For $l \geq 1$, by Lemma~\ref{lemma:MinNoOfSamples}, $N_a(lm)\geq \sqrt{lm}-1$ for each arm $a$. Let \(\hat{\mu}_{(a,s)} \) denote the empirical distribution corresponding to \(s\) samples from arm \(a\). Using union bound, 
\begin{align*}
\mathbb{P}_{\mu}\lrp{d_K(\hat{\mu}_{i}(lm),\mu_i)\geq \zeta} &\le \mathbb{P}_{\mu}\lrp{d_K(\hat{\mu}_{i}(lm),\mu_i)\geq \zeta, \; N_i(lm) \geq \sqrt{lm}-1}\\ 
&\leq \sum\limits_{s=\sqrt{lm}-1}^{T}\mathbb{P}_{\mu}\lrp{d_K(\hat{\mu}_{(i,s)},\mu_i)\geq \zeta}.
\end{align*}
Using the DKW inequality \cite{massart1990tight}, the above is bounded by 
\[ \sum\limits_{s= \sqrt{lm} -1 }^T e^{-2s\zeta^2} \le e^{-\sqrt{lm}\zeta^2}\lrp{1-e^{-2\zeta^2}}\inv. \]
Thus,
\begin{align*}\numberthis\label{partialProbBound}
\mathbb{P}_{\mu}(\mathcal{G}^{c}_{T}(\epsilon)) 
&\leq TKe^{-T^\frac{1}{8}\zeta^2}\lrp{1-e^{-2\zeta^2}}\inv + \sum\limits_{l=l_1(T)}^{l_2(T)} \sum\limits_{i=1}^K \mathbb{P}\lrp{\abs{\frac{N_i(lm)}{lm} - t_i^*(\mu) } \geq 4\epsilon', \; \mathcal{G}^1_T  }.
\end{align*}

To bound the other summation above, for $l\in\lrset{l_1(T), \dots, l_2(T)}$,  let  $M_{lm}$ denote the set of times in $\lrset{1, \dots, lm}$ when the algorithm flipped coins to decide which arm to pull. For $i\in [K] $ and $j \in \mathbb{N}$, let $I_i(j)$ be  the indicator that $i^{th}$ arm was pulled on $j^{th}$ time step, and $\hat{\mu}(j)$ denote the empirical distribution vector at the beginning of the batch to which $j$ belongs.  For $l \in \lrset{l_1(T), \dots, l_2(T)}$, define
\[A_2 := \frac{1}{lm}\sum\limits_{j\in M_{lm}}\abs{t^*_i(\Pi\lrp{\hat{\mu}(j)}) - t^*_i(\mu) }, \;\; \text{ and }\;\; A_3 := \frac{1}{lm}\sum\limits_{j\not \in M_{lm}}\abs{I_{i}(j)-t^*_i(\mu)}. \]
Observe that
\begin{align*}
\mathbb{P}\lrp{\abs{\frac{N_i(lm)}{lm} - t^*_i(\mu) }\geq 4\epsilon', \; \mathcal{G}^1_T }\leq \mathbb{P}\lrp{\underbrace{\frac{1}{lm}\abs{\sum\limits_{j\in M_{lm}} \lrp{I_i(j)-t^*_{i}(\Pi\lrp{\hat{\mu}(j)})} }}_{:=A_1} + A_2 + A_3 \geq 4\epsilon', \;\mathcal{G}^1_T}.
\end{align*}
Since \(\abs{I_{i}(j) - t^*_{i}(\mu)} \leq 1\), and from Lemma~\ref{lemma:MinNoOfSamples}, the sampling algorithm ensures that the number of forced exploration steps is bounded, i.e., \(lm - \abs{M_{lm}}\leq K\sqrt{lm}\), the term \(A_3\) in the above expression can be bounded from above as, 
\[ A_3 \leq \frac{lm-\abs{M_{lm}}}{lm} \leq \frac{K\sqrt{lm}}{lm} = \frac{K}{\sqrt{lm}},\]
If batch size \(m\) is proportional to \( \log\frac{1}{\delta}\) (see (\ref{optBatch}) and the associated discussion for the choice of batch size in Section~\ref{sec:optbatch.mean}) and decreases with increasing \(\delta\), for values of \(\delta\) close to \(0\), \(A_3 \leq \epsilon'\) for all T. Next,  
\begin{align*}
A_2 & = \frac{1}{lm} \sum\limits_{\substack{j\in M_{lm} \\ j < l_0(T)m}} \abs{t^*_i(\Pi\lrp{\hat{\mu}(j)}) - t^*_i(\mu)} +  \frac{1}{lm} \sum\limits_{\substack{j\in M_{lm} \\ j \geq l_0(T)m}} \abs{t^*_i(\Pi\lrp{\hat{\mu}(j)} - t^*_i(\mu)}.
\end{align*}
Observe that if \(l_0(T) = 1\), then the first term above is \(0\) since for \(j < m\), the algorithm does not flip any coins to decide the allocation of samples, and hence \(\abs{M_{m}} = 0\). On the other hand, if \(l_0(T) = \floor{\frac{T^{1/4}}{m}}\), then \(l_1(T)=\floor{\frac{T^{3/4}}{m}} \) and the first term being at most \(\frac{l_0(T)m}{l_1(T)m}\), is bounded by $\frac{1}{T^{1/2}}$. However, since \(T\geq m+1 \) and \(m\propto \log\lrp{1/\delta} \), for \(\delta\) close to \(0\), \(1/m \leq \epsilon' \). Thus, the first term is less than \(\epsilon'\) for all \(T\geq m\). 

For the second term, for \( j \geq l_0(T)\times m \), \(\hat{\mu}(j) \) lies in \(\mathcal{I}_{\epsilon'} \), and hence this term is bounded by \(\epsilon' \). This gives that \(A_2 \leq 2\epsilon'. \) Thus, for \( T\geq m, \text{ and } l \geq l_1(T) \):

\begin{align*}
&\mathbb{P}\lrp{\abs{\frac{N_i(lm)}{lm} - t^*_i(\mu) }\geq 4\epsilon', \;  \mathcal{G}^1_T }
\\ &\leq \mathbb{P}\lrp{\frac{1}{lm}\abs{\sum\limits_{j\in M_{lm}} \lrp{I_i(j)-t^*_{i}(\Pi\lrp{\hat{\mu}(j)})}}+ 2\epsilon' + \epsilon'  \geq 4\epsilon', \; \mathcal{G}^1_T}\\
&\leq \mathbb{P}\lrp{\abs{\sum\limits_{j\in M_{lm}} \lrp{I_i(j)-t^*_{i}(\Pi\lrp{\hat{\mu}(j)})}}\geq lm{\epsilon'}}.
\end{align*}

Let 
\[S_n = \sum_{j\in M_n} \lrp{I_i(j)-t^*_{i}(\Pi\lrp{\hat{\mu}(j)})}.\] 
Clearly, \(S_n\) being sum of independent, zero-mean random variables, is a martingale. Further, \(\abs{S_{n+1}-S_{n}} \leq 1 \). Thus using Azuma-Hoeffding inequality, 

\begin{align*}
\mathbb{P}\lrp{\abs{\frac{N_i(lm)}{lm} - t^*_i(\mu) }\geq 4\epsilon', \;\mathcal{G}^1_T } \leq 2\exp\lrp{-\frac{l^2 m^2 \epsilon'^2 }{2\abs{M_{lm}}}}\leq 2\exp\lrp{-\frac{lm {\epsilon'}^2}{2}}. 
\end{align*}
Summing over \(l\) and \(i\), the above bounded from above by 
\begin{align*}
\sum\limits_{l=l_1(T)}^{l_2(T)} 2K\exp\lrp{-\frac{lm{\epsilon'}^2}{2}}\leq \frac{2KT}{m}\exp\lrp{-\frac{l_1(T) m \epsilon'^2 }{2}}.\label{partialProbBound2}\numberthis
\end{align*}
Combining (\ref{partialProbBound}) and (\ref{partialProbBound2}), we get the desired result.

\end{document}